\documentclass[mnsc]{informs3}

\OneAndAHalfSpacedXI
\TheoremsNumberedThrough
\EquationsNumberedThrough

\usepackage{amsmath,amssymb,bm}
\usepackage{fix-cm}
\usepackage{booktabs}
\usepackage{graphicx}
\usepackage{caption}
\usepackage{subcaption}
\usepackage{algorithm}
\usepackage{algorithmic}
\usepackage{natbib}
\usepackage{xcolor}
\usepackage{tikz}
\usetikzlibrary{decorations.pathreplacing,arrows.meta,positioning,shapes.geometric,calc,fit,backgrounds}
\usepackage{hyperref}
\hypersetup{colorlinks=true,pdfnewwindow=true,pdfstartview=FitH,%
urlcolor=blue!90!red!45!black,citecolor=blue!90!red!45!black,linkcolor=red!90!black}
\usepackage[nameinlink,noabbrev]{cleveref}
\crefname{assumption}{Assumption}{Assumptions}
\Crefname{assumption}{Assumption}{Assumptions}

\newcommand{\X}{\mathcal{X}}
\newcommand{\Y}{\mathcal{Y}}
\newcommand{\D}{\mathcal{D}}
\newcommand{\G}{\mathcal{G}}
\newcommand{\M}{\mathcal{M}}

\newcommand{\Prob}{\mathbb{P}}
\newcommand{\E}{\mathbb{E}}
\newcommand{\1}{\mathbf{1}}
\newcommand{\OPT}{\mathrm{OPT}}
\newcommand{\RR}{\mathrm{RR}}

\bibpunct[, ]{(}{)}{,}{a}{}{,}%
\def\bibfont{\fontsize{9}{11}\selectfont}%
\def\bibsep{\smallskipamount}%
\tikzset{
    arr/.style={-{Latex[length=2.5mm,width=1.8mm]}, thick},
    pararr/.style={-{Latex[length=2.5mm,width=1.8mm]}, thick, teal!70!black},
    io/.style={draw, rounded corners=2pt, fill=gray!10, minimum width=15mm,
        minimum height=8mm, align=center, font=\small},
    model/.style={draw, rounded corners=3pt, fill=violet!12, minimum width=24mm,
        minimum height=10mm, align=center, font=\small\bfseries},
    agent/.style={draw, rounded corners=2pt, fill=cyan!10, minimum width=18mm,
        minimum height=8mm, align=center, font=\small},
    groupbox/.style={draw, rounded corners=5pt, dashed, inner sep=5pt}
}

\MANUSCRIPTNO{MS-0000-0000.00}

\begin{document}

\RUNAUTHOR{Abdolmaleki, Jasin and Wang}
\RUNTITLE{Workflow Portfolios under Imperfect Selection and Compute Cost}

\TITLE{Designing Agentic AI Workflow Portfolios under \\ Imperfect Selection and Compute Cost}

\ARTICLEAUTHORS{%
\AUTHOR{Mojtaba Abdolmaleki}
\AFF{Ross School of Business, University of Michigan, United States, \EMAIL{\href{mailto:mojtabaa@umich.edu}{mojtabaa@umich.edu}}}
\AUTHOR{Stefanus Jasin}
\AFF{Ross School of Business, University of Michigan, United States, \EMAIL{\href{mailto:sjasin@umich.edu}{sjasin@umich.edu}}}
\AUTHOR{Boyu Wang}
\AFF{TrueFoundry, \EMAIL{\href{mailto:boyu.wang@truefoundry.com}{boyu.wang@truefoundry.com}}}
\AFF{School of Management, University of San Francisco, United States, \EMAIL{\href{mailto:bwang62@usfca.edu>}{bwang62@usfca.edu}}}
}

\ABSTRACT{Agentic AI systems often approach the same task through multiple workflows that differ in reasoning strategy, verification structure, and compute cost. A natural deployment policy is to use the workflow with the highest average performance, but this can be suboptimal because different workflows may succeed on different instances. We study a portfolio-and-selector paradigm in which a firm runs multiple workflow executions and selects the final answer after observing their outputs. Additional executions may uncover correct answers that the best standalone workflow misses, but they consume compute and introduce plausible distractors that complicate final selection. We formulate this as a workflow portfolio problem in which the firm jointly chooses run size and allocation across workflow types. We summarize selector quality through an odds-lift index and derive sharp bounds on the value of workflow variety. For finite workflow pools, we develop exact formulations, linear programming relaxations, randomized rounding procedures, and computable performance certificates. For large implicit workflow classes, we derive a finite-dimensional dual and an ellipsoid method using a pricing oracle to identify workflows with high weighted accuracy net of recurring compute cost. Under a weak condition, the method obtains a near-optimal solution to the relaxation with polynomially many oracle calls. We evaluate the framework on three datasets: ABCD, Schema-Guided Dialogue, and HotpotQA. Relative to the best standalone workflow, portfolio optimization improves held-out selector accuracy by 3.1, 7.5, and 0.9 percentage points, respectively. Dual-guided workflow generation adds 3.5 points on ABCD and 24.1 on HotpotQA, with no additional gain on Schema-Guided Dialogue.

}

{\color{blue}
\KEYWORDS{agentic AI; workflow portfolios; large language models; endogenous portfolio size; compute allocation; selector design; selector strength; odds lift; random utility; ellipsoid method; separation oracle; repeated execution}
}

\maketitle
\setcounter{page}{1}


\fontsize{11pt}{16pt}\selectfont

\section{Introduction}
\label{sec:introduction}

Artificial intelligence is moving beyond individual productivity tools and
into the operating core of organizations. Enterprise software vendors are
embedding task-specific AI capabilities into business applications, while
firms are beginning to delegate not only information retrieval and drafting,
but also parts of operational decision processes to AI systems
\citep{gartner2025agents,gartner2025cancellations}. Recent business surveys
report that many organizations are experimenting with or deploying agentic AI,
while also emphasizing that realizing value requires redesigning workflows,
establishing governance, and building reliable orchestration layers rather than
simply giving employees access to models
\citep{mckinsey2025stateai,mckinsey2026agenticfoundations,
mitsmrbcg2025agentic,hbr2025agentictrust}. This shift raises a natural
operations-management question: once AI workflows become part of recurring
business processes, how should firms decide which workflows to deploy and how
much compute to devote to them?

The setting we study arises in service operations, compliance, claims
processing, internal knowledge work, and decision support. A
customer-service system may classify a ticket, retrieve relevant policies,
identify the customer's need, draft a response, check the response for
compliance, and decide whether to escalate the case to a human. A
claims-processing system may summarize a claim, compare it with contract
language, identify relevant exceptions, and recommend whether to approve or
deny it. A compliance assistant may retrieve applicable rules, analyze a
transaction, flag inconsistencies, and recommend an appropriate action. In
each of these settings, the relevant operational unit is not a single model
call, but a workflow: a sequence or graph of prompts, model calls, retrieval
steps, tools, verifiers, and aggregation rules that produces a candidate
answer.

At first glance, a natural approach is to identify the workflow with the
highest average accuracy and deploy it alone. Average performance, however,
can conceal substantial heterogeneity across tasks: different workflows may
succeed on very different cases. Relying on a single workflow can therefore
leave important gaps that another workflow could fill. A firm may have many
workflows that it could develop and maintain. Some are cheap and fast, while
others are more expensive but more careful. Some rely on retrieval, some on
decomposition or repeated reasoning, and others on verification or critique.
A workflow that performs poorly on average may still be valuable if it solves
cases that stronger workflows miss, whereas a high-performing workflow may add
little if it succeeds mainly on cases that are already easy. The firm must
therefore decide which workflow types to maintain, how many times to execute
each type, and which candidate answer to select as the final output, while
accounting for the recurring compute cost of these decisions.

Much of the existing literature follows a plan, then execute
paradigm \citep{shen2023hugginggpt,liang2024taskmatrix,lu2023chameleon,zhang2024aflow,hu2024adas,zhang2025maas}. A large model, router, or planner first selects or designs a single
workflow, which is then executed to produce the final answer. This plan--then--execute paradigm is illustrated in
\Cref{fig:traditional-paradigm}. This approach is
natural when the system can reliably predict, before execution, which workflow
is best suited to the task. In many operational settings, however, the
firm may not know in advance which workflow will succeed. It may be
easier to recognize a correct answer after observing the candidate outputs and
their supporting evidence, such as execution traces, retrieved sources,
agreement patterns, confidence signals, or verifier results, than to identify
ex ante the single workflow that will generate it.

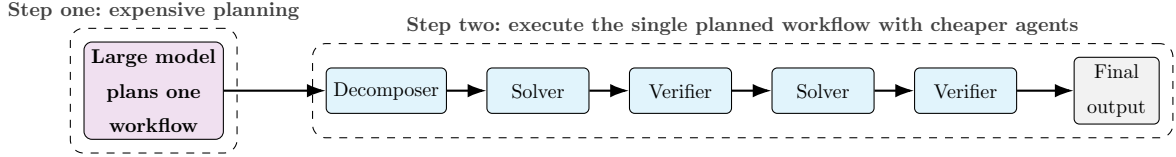
\begin{figure}[h]
\centering
\begin{tikzpicture}[
    scale=0.75,
    transform shape
]

\node[model] (planner) {Large model\\plans one\\workflow};

\node[agent, right=18mm of planner] (d) {Decomposer};
\node[agent, right=7mm of d] (s1) {Solver};
\node[agent, right=7mm of s1] (v1) {Verifier};
\node[agent, right=7mm of v1] (s2) {Solver};
\node[agent, right=7mm of s2] (v2) {Verifier};
\node[io, right=10mm of v2] (ans) {Final\\output};

\draw[arr] (planner) --
    node[above, font=\footnotesize] {} (d);
\draw[arr] (d) -- (s1);
\draw[arr] (s1) -- (v1);
\draw[arr] (v1) --
    node[above, font=\footnotesize] {} (s2);
\draw[arr] (s2) -- (v2);
\draw[arr] (v2) -- (ans);

\node[groupbox, fit=(planner),
      label={[font=\bfseries\small, text=black!75]
      above:Step one: expensive planning}] {};

\node[groupbox, fit=(d)(s1)(v1)(s2)(v2)(ans),
      label={[font=\bfseries\small, text=black!75]
      above:Step two: execute the single planned workflow with cheaper agents}] {};

\end{tikzpicture}
\caption{Traditional orchestration paradigm: a large model synthesizes one
workflow, which is then executed by cheaper specialized agents.}
\label{fig:traditional-paradigm}
\end{figure}

In this paper, we propose a portfolio and selector paradigm. The firm first builds
and maintains a set of workflow types. For each task class or deployment
setting, it decides how many execution slots to use and how to allocate those
slots across the available workflows. The same workflow may receive more than
one slot. The selected executions are then run in parallel, and a selector
chooses the final answer from the resulting candidates. This operating model creates a fundamental tradeoff. An additional execution
may produce a correct answer that would otherwise be unavailable, or it may
increase the representation of a workflow that performs well on many tasks.
At the same time, every execution consumes tokens, latency, tool calls, and
money. If its output is wrong, it also adds another plausible distractor for
the selector. More executions are therefore not necessarily better.

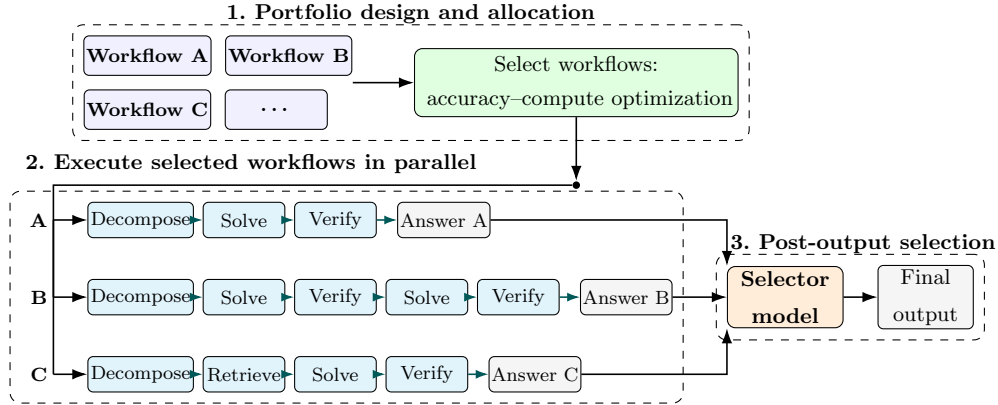
\begin{figure}[t]
\centering
\begin{tikzpicture}[
    scale=0.90,
    transform shape,
    arr/.style={
        -{Latex[length=2mm,width=1.4mm]},
        semithick
    },
    pararr/.style={
        -{Latex[length=1.7mm,width=1.2mm]},
        semithick,
        teal!70!black
    },
    stage/.style={
        draw,
        dashed,
        rounded corners=4pt,
        inner sep=4pt
    },
    portfolio/.style={
        draw,
        rounded corners=2pt,
        fill=blue!6,
        minimum width=15mm,
        minimum height=5.8mm,
        inner sep=1pt,
        align=center,
        font=\scriptsize\bfseries
    },
    selectbox/.style={
        draw,
        rounded corners=3pt,
        fill=green!12,
        minimum width=36mm,
        minimum height=10mm,
        inner sep=2pt,
        align=center,
        font=\footnotesize
    },
    agent/.style={
        draw,
        rounded corners=2pt,
        fill=cyan!10,
        minimum width=12mm,
        minimum height=5.2mm,
        inner xsep=1.2pt,
        inner ysep=1pt,
        align=center,
        font=\scriptsize
    },
    outnode/.style={
        draw,
        rounded corners=2pt,
        fill=gray!8,
        minimum width=11mm,
        minimum height=5.2mm,
        inner sep=1pt,
        align=center,
        font=\scriptsize
    },
    selector/.style={
        draw,
        rounded corners=3pt,
        fill=orange!15,
        minimum width=17mm,
        minimum height=9mm,
        inner sep=1.5pt,
        align=center,
        font=\footnotesize\bfseries
    },
    io/.style={
        draw,
        rounded corners=2pt,
        fill=gray!8,
        minimum width=14mm,
        minimum height=7mm,
        inner sep=1.5pt,
        align=center,
        font=\footnotesize
    }
]


\node[portfolio] (p1) at (-3.5, 0.35) {Workflow A};
\node[portfolio, right=2mm of p1] (p2) {Workflow B};
\node[portfolio, below=2mm of p1] (p3) {Workflow C};
\node[portfolio, right=2mm of p3] (pd) {\ldots};

\node[
    fit=(p1)(p2)(p3)(pd),
    inner sep=0pt
] (pbox) {};

\node[
    selectbox,
    right=9mm of pbox
] (select)
{Select workflows:\\[-0.5mm]
 { accuracy--compute optimization}};

\draw[arr] (pbox.east) -- (select.west);

\begin{scope}[on background layer]
\node[
    stage,
    fit=(p1)(p2)(p3)(pd)(select)
] (designbox) {};
\end{scope}

\node[
    font=\bfseries\footnotesize
] at ([yshift=1.5mm]designbox.north)
{1. Portfolio design and allocation};


\node[agent] (w2d) at (-3.6,-3.20) {Decompose};
\node[agent, right=1.3mm of w2d] (w2s1) {Solve};
\node[agent, right=1.3mm of w2s1] (w2v1) {Verify};
\node[agent, right=1.3mm of w2v1] (w2s2) {Solve};
\node[agent, right=1.3mm of w2s2] (w2v2) {Verify};
\node[outnode, right=3mm of w2v2] (y2) {Answer B};

\node[agent, above=6mm of w2d] (w1d) {Decompose};
\node[agent, right=1.3mm of w1d] (w1s) {Solve};
\node[agent, right=1.3mm of w1s] (w1v) {Verify};
\node[outnode, right=3mm of w1v] (y1) {Answer A};

\node[agent, below=6mm of w2d] (w3d) {Decompose};
\node[agent, right=1.3mm of w3d] (w3r) {Retrieve};
\node[agent, right=1.3mm of w3r] (w3s) {Solve};
\node[agent, right=1.3mm of w3s] (w3v) {Verify};
\node[outnode, right=3mm of w3v] (y3) {Answer C};

\node[
    font=\scriptsize\bfseries,
    left=4.5mm of w1d
] (lab1) {A};

\node[
    font=\scriptsize\bfseries,
    left=4.5mm of w2d
] (lab2) {B};

\node[
    font=\scriptsize\bfseries,
    left=4.5mm of w3d
] (lab3) {C};

\draw[pararr] (w1d) -- (w1s);
\draw[pararr] (w1s) -- (w1v);
\draw[pararr] (w1v) -- (y1);

\draw[pararr] (w2d) -- (w2s1);
\draw[pararr] (w2s1) -- (w2v1);
\draw[pararr] (w2v1) -- (w2s2);
\draw[pararr] (w2s2) -- (w2v2);
\draw[pararr] (w2v2) -- (y2);

\draw[pararr] (w3d) -- (w3r);
\draw[pararr] (w3r) -- (w3s);
\draw[pararr] (w3s) -- (w3v);
\draw[pararr] (w3v) -- (y3);

\begin{scope}[on background layer]
\node[
    stage,
    fit=(lab1)(w1d)(w1s)(w1v)(y1)
        (lab2)(w2d)(w2s1)(w2v1)(w2s2)(w2v2)(y2)
        (lab3)(w3d)(w3r)(w3s)(w3v)(y3)
] (execbox) {};
\end{scope}

\node[
    font=\bfseries\footnotesize,
    anchor=south west
] at ([xshift=1mm,yshift=1.2mm]execbox.north west)
{2. Execute selected workflows in parallel};


\node[
    circle,
    fill=black,
    inner sep=1.1pt
] (split) at ([yshift=-10mm]select.south) {};

\draw[arr]
    (select.south)
    --
    node[
        right,
        font=\scriptsize,
        align=left
    ] {}
    (split.north);

\coordinate (fan1) at ([xshift=-5mm]w1d.west);
\coordinate (fan2) at ([xshift=-5mm]w2d.west);
\coordinate (fan3) at ([xshift=-5mm]w3d.west);

\draw[semithick] (split.west) -| (fan2);
\draw[semithick] (fan1) -- (fan3);

\draw[arr] (fan1) -- (w1d.west);
\draw[arr] (fan2) -- (w2d.west);
\draw[arr] (fan3) -- (w3d.west);


\node[
    selector,
    right=8mm of y2
] (selector) {Selector\\model};

\node[
    io,
    right=5mm of selector
] (final) {Final\\output};

\draw[arr] (y1.east) -| (selector.north west);
\draw[arr] (y2.east) -- (selector.west);
\draw[arr] (y3.east) -| (selector.south west);
\draw[arr] (selector) -- (final);

\begin{scope}[on background layer]
\node[
    stage,
    fit=(selector)(final)
] (postbox) {};
\end{scope}

\node[
    font=\bfseries\footnotesize
] at ([yshift=1.5mm]postbox.north)
{\quad 3. Post-output selection};

\end{tikzpicture}

\caption{Portfolio-based orchestration: the decision maker selects workflow
types and allocates execution slots, runs the selected workflows in parallel,
and chooses the final answer after observing their realized outputs.}
\label{fig:portfolio-paradigm}
\end{figure}

The central question of this paper is:
\emph{How should a firm decide how many workflow executions to run
and how to allocate them across workflow types, when additional executions
can improve the quality of the candidate set but also increase selector
confusion and compute cost?}

\vspace{1mm}
\textbf{Related literature.} Our work brings together several related streams, which we discuss in detail
in \Cref{sec:related}. One stream studies how to design effective agentic
workflows and inference time architectures
\citep{zhang2024aflow,hu2024adas,zhuge2024gptswarm,
saadfalcon2024archon,zhang2025maas}. A second stream studies how to route queries
across models or workflows to balance quality and cost
\citep{chen2023frugalgpt,ong2024routellm,hu2024routerbench,
huang2025thriftllm}. A third stream generates multiple candidate answers and uses
ranking, fusion, self consistency, or multi-agent aggregation to choose among
them
\citep{wang2022selfconsistency,jiang2023llmblender,
wang2024mixture}. Together, these studies address important parts of the
deployment problem, namely workflow design, pre-execution allocation, and
post-output aggregation. Once these decisions are considered jointly, however, the problem
becomes one of choosing and allocating a portfolio of alternatives under
limited resources and imperfect final selection. This perspective relates to several established areas in operations research, including choice modeling, assortment optimization, coverage, submodular optimization, column generation, and optimization through separation oracles
\citep{luce1959choice,mcfadden1974conditional,
vanryzin1999assortment,talluri2004theory,bertsimas2019exact,dong2023pasta, wang2024randomized,nemhauser1978analysis,
barman2021concave,alaei2010maximizing,dantzig1960decomposition,
desaulniers2005column,groetschel1981ellipsoid,
groetschel1988geometric}. What remains missing is an integrated framework for designing workflow portfolios when both execution decisions and final answer selection affect
performance and cost.

The challenge is not limited to choosing among workflows that have already
been evaluated. In practice, the feasible workflow space is large and cannot
usually be enumerated in advance. We therefore need a way to model how the
firm accesses workflows outside its current evaluated pool. In this paper, we
assume that new workflows can be proposed by an oracle, which may take the
form of a teacher model, an automated workflow generator, a structured search
procedure, or a human engineering team. The firm guides the oracle toward
generating workflows that perform well on tasks that the current portfolio
handles poorly, while also accounting for the recurring cost of executing
those workflows. This raises two questions: how should the firm use the oracle
to search the implicit workflow space, and how can it certify the quality of
the resulting portfolio?

We address these questions through a selector aware workflow portfolio model.
The firm decides how many times to execute each workflow type, which jointly
determines the composition and total size of the candidate set. Each execution
produces a candidate answer, and a selector chooses the final output after
observing the resulting candidates. This portfolio-and-selector paradigm is illustrated in
\Cref{fig:portfolio-paradigm}. We summarize selector performance through
recovery curves that relate the number of correct candidates in the set to the
probability that the final selected answer is correct. The firm's objective
is to maximize deployed decision quality net of recurring compute cost.

\vspace{1mm}
\textbf{Our contributions.}
Our first contribution develops a general theory of selector strength. We
introduce an odds lift index that measures how much the selector increases the
odds of returning a correct answer relative to random selection. A finite bound
on this index places the selector's recovery curve below a common concave
envelope. This envelope gives a sharp limit on how much any workflow portfolio
can improve on the best single workflow, regardless of how many workflows are
available or how complementary they appear. It also yields simple screening
rules that identify when selector quality and execution cost make a
multi-workflow portfolio unattractive, so that the best decision is to use a
single workflow or not deploy the system.

Our second contribution develops an optimization framework for finite workflow
pools when the total number of workflow executions, which we call the
\emph{run size}, is itself a decision. The framework determines both the run
size and the allocation of execution slots across workflow types, recognizing
that an additional execution may improve accuracy on some tasks while
increasing selector confusion on others. For any fixed run size, the problem
becomes a concave coverage problem. This structure leads to exact integer
programming formulations, linear programming relaxations, randomized rounding
procedures, and computable optimality certificates. We then develop a sparse
grid method for choosing the run size. Rather than solving a separate
optimization problem for every possible size, the method evaluates only a
small collection of candidate values while retaining a provable approximation
guarantee. This provides a practical approach for jointly choosing how many
workflow executions to run and how to allocate them across workflow types.

Our third contribution develops an optimization method over the
\emph{implicit workflow class}, by which we mean the full set of feasible
workflows that the oracle can search over but the firm cannot enumerate in
advance. Because the same workflow type may occupy multiple execution slots,
the linear program for a fixed run size does not require an upper bound on the
number of times each type can be used. Although this linear program contains
one variable for every workflow type in the implicit class, its dual has only
finitely many task price variables, together with one constraint for each
workflow. A global pricing oracle can therefore identify a workflow whose
constraint is violated, or certify that no important violation remains. Using
the classical equivalence between optimization and separation, the ellipsoid
method computes a solution within $\varepsilon$ of the full linear programming
relaxation using a number of oracle calls that is polynomial in the problem
encoding and $\log(1/\varepsilon)$. Combining this optimization error with the
losses from the cardinality grid and randomized rounding yields a high
probability approximation guarantee relative to the entire implicit workflow
class, without requiring that class to be enumerated.

{\color{black}
Our fourth contribution is empirical. We first calibrate selector recovery
across the ABCD and Schema-Guided Dialogue service-operations tasks
\citep{chen2021abcd,rastogi2020sgd} and the HotpotQA question-answering
benchmark \citep{yang2018hotpotqa}. Selector strength is substantially above
random selection in the two service domains and remains positive, although
weaker, on HotpotQA. We then evaluate the complete stochastic workflow-
generation and portfolio-optimization pipeline in all three domains. On fresh
held-out tasks, actual selector accuracy rises from 43.625\% for the best
initial singleton to 50.250\% for the final ABCD portfolio, from 85.250\% to
92.750\% on Schema-Guided Dialogue, and from 30.250\% to 55.250\% on
HotpotQA. Dual-guided generation incorporates four workflows on ABCD, none on
Schema-Guided Dialogue, and one on HotpotQA. This variation is consistent with
the model's central implication: a generated workflow creates value only when
its incremental coverage and selection benefit justify its recurring compute
cost.}

The central message is that deploying agentic AI is neither a search for the
single best workflow nor a simple rule that more candidates are always better.
Workflow variety creates value only when the selector can identify useful
outputs reliably enough to justify the additional compute cost and selection
difficulty. Effective deployment therefore requires firms to manage workflow
generation, run size, the allocation of executions across workflow types,
selector strength, and compute expenditure as a joint decision.

\vspace{1mm}
\textbf{Organization of the paper.}
The rest of the paper is organized as follows. \Cref{sec:related}
reviews the related literature. \Cref{sec:model} formulates the
portfolio and selector problem and introduces the oracle interface for
searching the implicit workflow class. \Cref{sec:selector-curve}
studies how selector strength limits the value of workflow variety.
\Cref{sec:ip-lp-rounding} develops the finite pool optimization
framework, including exact integer programs, linear programming
relaxations, randomized rounding, and optimality certificates.
\Cref{sec:implicit-oracle} derives the dual formulation for the
implicit workflow class and uses a workflow pricing oracle together
with the ellipsoid method to obtain performance guarantees without
enumerating all feasible workflows.
\Cref{sec:stochastic-workflow-outcomes} extends the model and its
algorithmic guarantees to stochastic workflow execution.
\Cref{sec:numerical-experiments} presents the numerical experiments,
and \Cref{sec:conclusion} concludes the paper. Unless otherwise noted,
all proofs are provided in the appendix.

\section{Related Literature}
\label{sec:related}

Our work relates to several streams of literature. The first studies the
automated design of agentic workflows, compound AI systems, and
LLM generated algorithms or heuristics. The second examines model routing,
cascades, and methods that generate and aggregate multiple candidate outputs.
The third includes work on random utility, assortment choice, and selection
among competing alternatives. The fourth concerns coverage, submodular
optimization, and optimization over large implicitly represented decision
spaces. Finally, our paper contributes to the emerging operations management
literature on the deployment, allocation, and governance of AI systems. We
discuss these streams in turn and explain how they inform, but do not resolve,
the joint problem of workflow generation, execution allocation, and final
answer selection studied here.

\vspace{1mm}
\paragraph{Automated design of agentic workflows and compound AI systems.}
A growing computer science literature studies how to design and optimize
multistep language model systems. AFlow searches over workflows represented as
code, using execution feedback and tree search
\citep{zhang2024aflow}. Automated Design of Agentic Systems treats agent design
as a search problem in which a meta agent proposes new agentic systems
\citep{hu2024adas}. GPTSwarm represents interacting language agents as a graph
whose structure can be optimized \citep{zhuge2024gptswarm}. Archon searches
over inference architectures that combine repeated sampling, ranking, critique,
verification, fusion, and model choice
\citep{saadfalcon2024archon}. MaAS constructs an agentic supernet and selects
task dependent multiagent architectures to balance performance and resource
use \citep{zhang2025maas}. LLMSelector studies model assignment within a
compound AI system, asking which language model should be used in each module
of a fixed multicall architecture \citep{chen2025llmselector}. Collectively,
these studies show that workflow structure, module choice, and automated search
can have substantial effects on system performance. Our work treats all these methods as potential mechanisms for generating candidate
workflows, but studies a different operational decision. The primary objective
in automated workflow design is typically to discover a strong workflow,
architecture, or assignment of models to modules. We instead ask which workflow
types should be retained, how many execution slots should be allocated to each,
and how much compute should be spent when the final answer is chosen only after
the candidate outputs are observed. Under this perspective, a workflow that is
not the strongest on average may still be valuable if it solves cases missed by
the existing portfolio and the selector can recognize its contribution. The
same workflow may be harmful if it mainly produces costly distractors that make
final selection more difficult.

\vspace{1mm}
\paragraph{LLM based algorithm design and heuristic portfolio generation.}
A related literature studies the use of large language models to design
algorithms and heuristics. Recent surveys organize this work around the roles
of language models as optimizers, designers, predictors, and components of
algorithmic search \citep{liu2026survey}. AlphaEvolve combines language model
generated code modifications with evaluator feedback and evolutionary search
to improve algorithms \citep{novikov2025alphaevolve}. EoH-S moves beyond the
search for a single heuristic and instead constructs a small set of
complementary heuristics \citep{liu2026eohs}. The recurring process of
generating candidates, evaluating their performance, and reoptimizing the
system also has methodological parallels in self adjusting control and joint
learning and optimization
\citep{jasin2014reoptimization,chen2019selfadjusting,
chen2021joint,zhang2022cyclic,agrawal2014dynamic,
keskin2014dynamic,elmachtoub2022smart, faradonbeh2023online}. The above literature is closely related to our workflow generation layer. In our
setting, however, the generated objects are executable AI workflows drawn from
an implicit workflow class, and their value depends on how they are deployed
together. The objective therefore accounts not only for the performance of each
generated workflow, but also for its contribution to a portfolio under
imperfect final selection and its recurring execution cost.

\vspace{1mm}
\paragraph{LLM routing, model cascades, and efficient inference.}
Another stream studies how to allocate queries across models with different
cost and quality profiles. FrugalGPT develops adaptive strategies, including
model cascades, to reduce inference cost while preserving or improving
performance \citep{chen2023frugalgpt}. RouteLLM learns routers that choose
between stronger and weaker models using preference data
\citep{ong2024routellm}, while RouterBench provides a benchmark for evaluating
systems that route queries across multiple language models
\citep{hu2024routerbench}. ThriftLLM studies the cost conscious selection of
LLM ensembles for classification tasks \citep{huang2025thriftllm}. Recent
operations research also treats heterogeneous LLM use as a resource allocation
problem. \citet{dean2026multillm} optimize parallel query counts under
reliability constraints, \citet{li2026sequentialllm} study sequential model
choice and stopping with query and waiting costs, and
\citet{guo2026congestion} analyze static cascades under congestion and latency.
Together, these papers determine how model calls should be allocated,
sequenced, or stopped. Our paper studies a complementary problem. Rather than deciding which model to query next, we choose how many workflow executions to run, how to allocate
those executions across workflow types, and which answer to select from the
resulting candidates. Routing and cascade methods govern the process of
acquiring model outputs, while our framework governs the design and evaluation
of the candidate portfolio presented to the selector.

\vspace{1mm}
\paragraph{Output ensembling, ranking, self consistency, and inference time
aggregation.}
A related literature improves language model performance by generating
multiple candidate outputs and then selecting, combining, or refining them.
Self consistency samples several reasoning paths and returns the answer with
the greatest agreement across paths \citep{wang2022selfconsistency}.
LLM Blender ranks outputs from multiple language models and combines the
highest quality candidates through generative fusion
\citep{jiang2023llmblender}. Mixture of Agents uses a layered architecture in
which agents at later stages build on outputs produced by agents in earlier
stages \citep{wang2024mixture}. Archon treats ranking, fusion, critique,
verification, and related inference techniques as components of an
architecture that can be selected through search
\citep{saadfalcon2024archon}. Our paper abstracts from the details of the selector itself. The selector may use ranking, fusion, consistency, critique, verification, or some other
mechanism. We instead study the upstream portfolio decision: which workflows
should generate candidates, how many executions should be run, and when an
additional candidate is worth its compute cost and its effect on final
selection.

\vspace{1mm}
\paragraph{Choice modeling, coverage, and implicit optimization.}
Our selector model draws on random utility and assortment choice. Multinomial
logit and Plackett--Luce provide tractable models of selection and ranking
\citep{luce1959choice,mcfadden1974conditional,plackett1975analysis,
vanryzin1999assortment,talluri2004theory}. 
Related work estimates preference parameters from observed choices and response times \citep{echenique2025general}.
In our setting, the alternatives
are workflow outputs rather than products, and their value depends on whether
they are correct. A wrong output enlarges the selector's choice set without
adding correct attraction. We use Plackett--Luce as an estimable
specialization, while our main variety bound relies only on a general odds
lift envelope. This choice structure leads directly to the optimization problem studied in
the paper. Once the run size is fixed, portfolio composition affects selector
performance only through the number of correct outputs on each task, yielding
a concave coverage problem and allowing us to use established linear
programming and rounding tools
\citep{nemhauser1978analysis,barman2021concave,
ageev2004pipage,chekuri2010dependent}. When run size is endogenous, however,
adding an execution changes both the correct count and the total number of
candidates, so the global objective is neither monotone nor generally
submodular. We handle this additional decision through a sparse cardinality
grid. Finally, when the feasible workflow class cannot be enumerated, a global
workflow pricing oracle separates the workflow indexed dual constraints,
allowing the ellipsoid method to optimize over the implicit class
\citep{groetschel1981ellipsoid,groetschel1988geometric}.

\vspace{1mm}
\paragraph{Operations management and AI deployment.}
An emerging operations management literature studies how AI systems should be
adopted, allocated, priced, and governed. Recent work examines query allocation
across heterogeneous language models
\citep{dean2026multillm,li2026sequentialllm,guo2026congestion},
the timing of AI adoption and access
\citep{abdolmaleki2026holdback}, pricing and delay in agentic AI services
\citep{abdolmaleki2026delayed}, and the provenance of language model generated
content \citep{radvand2026trainingfree}. 
\citet{baek2026ai} study complementarity among humans, large language
models, and operations research algorithms in inventory control.
We contribute to this literature by
studying workflow portfolio design. Our focus is how a firm should
choose the number and allocation of workflow executions, generate new workflow
types, and select among their outputs when performance depends jointly on
workflow complementarity, selector strength, and recurring compute cost.

\section{Model}
\label{sec:model}

This section develops the operational model for selector-aware workflow
portfolio design. We first define the task environment, workflow executions,
selector recovery, and the firm's objective of choosing both the composition
and size of the deployed portfolio. We then distinguish optimization over a
finite evaluated workflow pool from search over the larger implicit workflow
class and formalize the cost-aware pricing interface used to generate new
workflows.

\subsection{Deployment setting}
\label{sec:deployment-objective}

We study an operational setting in which a firm handles a stream of similar
tasks. Although the tasks share a common objective, individual instances may
differ substantially in the information, reasoning, and verification required
to solve them. A workflow designed for one source of difficulty may therefore
perform poorly on cases that require a different approach. This heterogeneity
creates a role for workflow portfolios: by running several complementary
workflows and selecting among their outputs, the firm may obtain a correct
answer on cases that no single workflow handles reliably. As a running example, consider customer-service routing. A customer initiates
a service conversation, and the system observes the customer's initial
messages together with relevant account information, order status, and policy
context. Determining the appropriate route may require the system to infer the customer's objective, identify the current service state, retrieve the
applicable policy, and verify that the proposed action is feasible. The final
decision may be, for example, to initiate an order cancellation, address a
refund-status inquiry, modify a subscription, begin troubleshooting, or
escalate the case to a human agent. Different workflows may emphasize different parts of this reasoning process. Running multiple workflows can increase the likelihood that at least one produces the correct route, but it also incurs additional compute cost and creates more candidate outputs for the selector to distinguish. The model below formalizes this tradeoff. (Remark~\ref{rem:task-types} discusses the extension to settings with multiple task types.)

\vspace{1mm}
\textbf{Tasks.}
Let $X\in\X$ denote a task instance drawn from a population $\D$, and let
$Y^\star\in\Y$ denote the correct answer, target decision, or benchmark
solution. The output space $\Y$ may be finite, numerical, structured, or
textual. In the baseline model, a returned answer $y$ receives payoff $\1\{y=Y^\star\}.$
In the customer-service routing example, $X$ contains the observed conversation, account
information, order status, and relevant context, while the target $Y^\star$ is the
correct routing label, such as ``cancel order,'' ``refund status,''
``subscription change,'' or ``escalate to human.''

\vspace{1mm}
\textbf{Workflows.}
A workflow $g\in\G$ is an executable AI procedure that maps a task instance 
to a candidate answer and an observable trace. Workflows can take many forms: 
a single prompt, a chain of specialized agents, a retrieval-augmented routine, 
a verify-and-revise loop, or a graph of agents and tools. The feasible space $\G$
contains all workflows consistent with the firm's operational constraints, such as allowed models, tools, retrieval sources, context length, latency, and termination rules. The main analysis does not require a particular graph representation of workflows; it only uses each workflow's evaluated correctness
vector and running cost. Appendix~\ref{app:workflow-grammar} gives one possible graph-based representation of $\G$ for readers who want a concrete finite workflow class.


When workflow $g$ is run on task $X$, it produces
\begin{equation}
    O_g(X)=\bigl(A_g(X),T_g(X)\bigr),
    \label{eq:workflow-output}
\end{equation}
where $A_g(X)\in\Y$ is the candidate answer and $T_g(X)$ is the observable
execution trace, potentially including intermediate messages, retrieved
documents, tool outputs, confidence scores, verifier signals, token counts, and
latency. We suppress the dependence on $X$ when it is clear from context.
Because $\G$ is typically too large to enumerate, the firm must search over
an ``implicit'' workflow space rather than selecting from a fixed list; we return to
this search layer in \Cref{sec:outer-inner-view,sec:workflow-access}.
In the routing example, a workflow might first retrieve relevant policy text,
infer the customer's intent, and propose a routing label. A verifier then checks
the proposal against the conversation and retrieved policy. If the proposal
fails verification, a revision agent updates it using the verifier's feedback,
and the verify--revise loop continues until the proposal passes verification or
a prespecified iteration limit is reached.

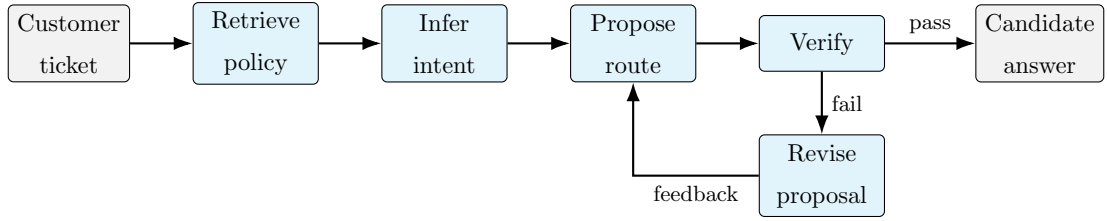
\begin{figure}[h]
\centering
\begin{tikzpicture}[
    node distance=7mm and 9mm,
    scale=0.92,
    transform shape
]
\node[io] (task) {Customer\\ticket};
\node[agent, right=of task] (ret) {Retrieve\\policy};
\node[agent, right=of ret] (intent) {Infer\\intent};
\node[agent, right=of intent] (route) {Propose\\route};
\node[agent, right=of route] (check) {Verify};
\node[agent, below=9mm of check] (revise) {Revise\\proposal};
\node[io, right=13mm of check] (ans) {Candidate\\answer};

\draw[arr] (task) -- (ret);
\draw[arr] (ret) -- (intent);
\draw[arr] (intent) -- (route);
\draw[arr] (route) -- (check);
\draw[arr] (check) -- node[above, font=\footnotesize] {pass} (ans);
\draw[arr] (check) -- node[right, font=\footnotesize] {fail} (revise);
\draw[arr] (revise.west) -| node[pos=0.25, below, font=\footnotesize]
    {feedback} (route.south);
\end{tikzpicture}
\caption{A customer-service routing workflow with a bounded
verify--revise loop.}
\label{fig:model-example-workflow}
\end{figure}


\vspace{1mm}
\textbf{Workflow accuracy.}
Each workflow has a recurring per-task running cost $c_g\ge0$. This cost may
represent dollars, tokens, latency-adjusted compute, tool-call charges, or an
additive resource index. The cost is recurring because it is paid every time the
workflow is run on an incoming task.
Let $(x_i,y_i^\star)_{i=1}^n$ be a labeled
development sample, i.e., the training or validation tasks used to evaluate
candidate workflows before deployment. 
The portfolio optimization problem below is defined on this training sample: workflow
correctness, and portfolio values are computed from
$(x_i,y_i^\star)_{i=1}^n$. 
For workflow $g$, define its task-level
correctness indicator by
\begin{equation}
    a_{ig}
    =
    \1\{A_g(x_i)=y_i^\star\},
    \qquad i=1,\ldots,n,\quad g\in\G.
    \label{eq:coverage-matrix}
\end{equation}
Thus, $a_{ig}=1$ if workflow $g$ returns the correct routing label for task
$i$, and $a_{ig}=0$ otherwise. Throughout the main analysis, we treat
$a_{ig}$ as deterministic, so each workflow--task pair has a fixed evaluated
correctness outcome on the development sample.
\Cref{sec:stochastic-workflow-outcomes} extends the model to stochastic
workflow execution, where $a_{ig}\in[0,1]$ denotes the probability that one
execution of workflow $g$ is correct on task $i$. 

\textbf{Portfolio and execution cost.}
We use $g\in\G$ to index workflow types. For each incoming task, the firm
chooses an execution-count vector
\[
    \bm m=(m_g)_{g\in\G}\in\mathbb Z_+^{\G}
\]
with finite support, where $m_g$ is the number of execution slots assigned to
workflow type $g$. It may appear natural to impose $m_g\le1$ for every
workflow type. Under an imperfect selector, however, repeated execution of a
strong workflow may raise the share of correct candidate outputs on the tasks
it already solves and thereby improve selector recovery; see
Appendix~\ref{app:workflow-multiplicity-example}. Under stochastic execution,
repeated runs can also yield different outcomes, making multiplicity even more
natural; \Cref{sec:stochastic-workflow-outcomes} extends most of our results to
that setting.

The total number of workflow executions is
\begin{equation}
    k(\bm m)
    =
    \sum_{g\in\G}m_g,
    \label{eq:portfolio-multiplicity}
\end{equation}
and must satisfy $k(\bm m)\le K_{\max}$, where $K_{\max}$ reflects
operational limits such as latency, context-window capacity, compute budget,
or risk policy.

Each execution occupies a separate slot, incurs cost $c_g$, and produces a
separate candidate for the selector. In the deterministic development-sample
model, all executions of workflow type $g$ share the evaluated correctness
vector $(a_{1g},\ldots,a_{ng})$. The zero vector $\bm 0$ represents the
outside option of not deploying the AI system and has value zero. For
$g\in\G$, let $\bm e_g$ denote the execution-count vector containing one
execution of workflow $g$ and zero executions of every other workflow.

For a feasible execution-count vector $\bm m$, let
\begin{equation}
    r_i(\bm m)
    =
    \sum_{g\in\G}a_{ig}m_g
    \label{eq:correct-count}
\end{equation}
denote the number of correct candidate outputs produced on task $i$, and let
\begin{equation}
    C(\bm m)
    =
    \sum_{g\in\G}c_gm_g
    \label{eq:portfolio-cost}
\end{equation}
denote the total recurring execution cost.

\vspace{1mm}

\vspace{1mm}
\textbf{Selector.}
Consider a feasible nonzero execution-count vector $\bm m$. Choose any
labeling $g_1,\ldots, $ $g_{k(\bm m)}$ of its execution slots satisfying
$
    |\{j:g_j=g\}|=m_g$ for all $g\in\G.
$
After all workflow executions have been completed, the firm observes
\begin{equation}
    O_{\bm m}
    =
    \bigl(O_{g_j}:j=1,\ldots,k(\bm m)\bigr).
\end{equation}
A post-output selector then chooses one of the candidate outputs:
\begin{equation}
    \operatorname{Sel}(X,O_{\bm m})
    \in
    \left\{
        A_{g_j}:j=1,\ldots,k(\bm m)
    \right\}.
    \label{eq:selector-rule}
\end{equation}
The selector may be an LLM judge, a verifier, a ranking model, a rule-based
policy checker, a human reviewer, or a combination of these mechanisms. In the
customer-service routing example, the selector observes the proposed routing
decisions and their execution traces and chooses the route submitted to the
service system.

We summarize selector performance through a family of \emph{recovery curves}
$\{\psi_k(r)\}_{k=1}^{K_{\max}}$, one for each candidate-set size $k$. Each
curve
\begin{equation}
    \psi_k:\{0,1,\ldots,k\}\to[0,1],
    \qquad
    \psi_k(0)=0,\quad \psi_k(k)=1,
    \label{eq:selector-curve-generic}
\end{equation}
maps the number of correct candidate outputs $r$ to the probability that the
selector returns a correct final output. The endpoint conditions are natural:
if no candidate is correct, the selector cannot recover one, and if all
candidates are correct, any choice succeeds.

The empirical selector-aware accuracy of a nonzero execution-count vector
$\bm m$ is
\begin{equation}
    J_\psi(\bm m)
    =
    \frac1n\sum_{i=1}^n
    \psi_{k(\bm m)}\bigl(r_i(\bm m)\bigr),
    \label{eq:endogenous-accuracy-generic}
\end{equation}
with $J_\psi(\bm 0)=0$ by convention.


\vspace{1mm}
\vspace{1mm}
\textbf{Deployment objective.}
The firm's goal is to choose an execution-count vector $\bm m$ that maximizes
selector-aware accuracy on the development sample net of recurring execution
cost. To place accuracy and cost on a common scale, let $\gamma\ge0$ denote
the cost price. If a correct decision is worth $v>0$ monetary units relative
to an incorrect decision, dividing monetary value by $v$ gives
$\gamma=1/v$. The net deployed value of $\bm m$ is
\begin{equation}
    \Pi_{\psi,\gamma}(\bm m)
    =
    J_\psi(\bm m)-\gamma C(\bm m),
    \qquad
    \Pi_{\psi,\gamma}(\bm 0)=0.
    \label{eq:net-deployed-value}
\end{equation}
The ideal deployment objective is
\begin{equation}
    \OPT_{\psi,\gamma}(\G)
    =
    \max_{\substack{
        \bm m\in\mathbb Z_+^{\G}\\
        \bm m\text{ has finite support}\\
        k(\bm m)\le K_{\max}
    }}
    \Pi_{\psi,\gamma}(\bm m).
    \label{eq:endogenous-optimum}
\end{equation}
This is the best net value achievable over the full workflow space $\G$, and 
serves as the benchmark for the algorithms developed in subsequent sections.

\vspace{1mm}
\begin{remark}[Multiple task types]
\label{rem:task-types}
The model is written for a fixed task type. In settings with several task
types, the same formulation can be applied separately within each type, so that
the firm chooses a type-specific workflow portfolio. A task type may be defined
by operational information available before execution, such as the claim category, product line, customer segment, language, channel,
risk tier, or required output format. Thus, task-type
information can be used to select the relevant portfolio before workflows are
run. Throughout the paper, we fix one task type and suppress the
type index. 
\end{remark}

\subsection{Restricted pools and the inner--outer view}
\label{sec:outer-inner-view}

Solving \eqref{eq:endogenous-optimum} directly is generally infeasible because
the workflow class $\G$ is large and implicit. We therefore distinguish two
objects. The first is a finite evaluated pool of workflows, over which the
portfolio problem can be solved directly. The second is the larger implicit
workflow class, over which the algorithm must search for additional useful
workflows.

For any finite evaluated pool $\M\subseteq\G$, define the restricted benchmark
\begin{equation}
    \OPT_{\psi,\gamma}(\M)
    =
    \max_{\substack{
        \bm m\in\mathbb Z_+^{\M}\\
        k(\bm m)\le K_{\max}
    }}
    \Pi_{\psi,\gamma}(\bm m).
    \label{eq:restricted-inner-optimum}
\end{equation}
Here, $\bm m$ is the execution-count vector restricted to workflow types in
$\M$. Every workflow type $g\in\M$ has already been evaluated on the
development sample, so its correctness vector
$(a_{1g},\ldots,a_{ng})$ and recurring execution cost $c_g$ are known. The restricted problem \eqref{eq:restricted-inner-optimum} is the
best net value achievable using only these evaluated workflow types. When
$\M=\G$, it coincides with the full deployment benchmark
\eqref{eq:endogenous-optimum}; otherwise it is only a finite-pool approximation.

The role of the outer layer is to decide where to search next in the implicit
class $\G$. The guiding principle is marginal value. Once a finite-pool linear program (LP)
relaxation is solved, the optimization problem assigns dual prices to development
tasks and to execution slots. A high task price means that an additional correct
candidate on that task would be valuable. A high execution-slot price means that
a new workflow must deliver enough marginal value to justify occupying one of
the limited run slots. These prices are passed to a workflow-pricing oracle,
which searches the implicit class for a workflow with high price-weighted
correctness net of recurring cost. Thus, the outer layer repeatedly asks a simple operational question:
\emph{is there a workflow, not yet available to the optimizer, whose expected
marginal contribution is large enough to matter?} 

\vspace{1mm}
\textbf{Workflow generator oracle.} We do not answer the above
question by searching $\G$ ourselves. Instead, we assume access to a workflow
generator, which we call the pricing oracle, and which may be a teacher model,
an automated workflow-search procedure, or a human engineering team. The
optimizer supplies the current prices, and the oracle proposes a workflow that
scores well against them. If the oracle returns such a workflow, the workflow
is evaluated on the development sample and becomes available to the
finite-pool optimizer. If the oracle can certify that no
workflow has sufficiently large price-adjusted value, then the current LP
relaxation has no important missing workflow at that query tolerance. We discuss this in more detail in \Cref{sec:workflow-access}. 

The overall information flow is summarized in \Cref{fig:inner-outer-loop}.

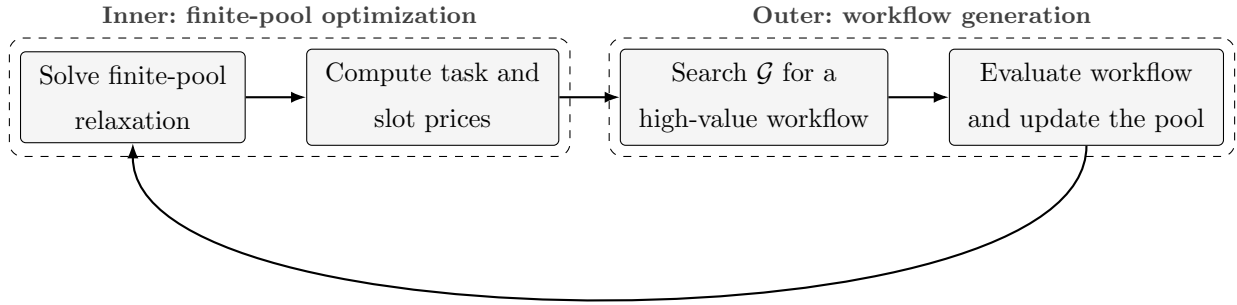
\begin{figure}[h]
\centering
\begin{tikzpicture}[
    node distance=7mm and 8mm,
    stepbox/.style={
        draw,
        rounded corners=2pt,
        fill=gray!8,
        align=center,
        inner xsep=7pt,
        inner ysep=5pt,
        font=\small
    },
    looparr/.style={-{Latex[length=2mm,width=1.5mm]}, thick}
]
\node[stepbox] (opt) {Solve finite-pool\\relaxation};
\node[stepbox, right=of opt] (price) {Compute task and\\slot prices};
\node[stepbox, right=of price] (gen) {Search $\G$ for a\\high-value workflow};
\node[stepbox, right=of gen] (eval) {Evaluate workflow\\and update the pool};

\begin{scope}[on background layer]
\node[
    groupbox,
    fit=(opt)(price),
    inner sep=4pt,
    label={[font=\bfseries\footnotesize, text=black!75]
        above:{Inner: finite-pool optimization}}
] (innerbox) {};

\node[
    groupbox,
    fit=(gen)(eval),
    inner sep=4pt,
    label={[font=\bfseries\footnotesize, text=black!75]
        above:{Outer: workflow generation}}
] (outerbox) {};
\end{scope}

\draw[looparr] (opt) -- (price);
\draw[looparr] (price) -- (gen);
\draw[looparr] (gen) -- (eval);
\draw[looparr]
    (eval.south)
    to[out=-90,in=-90,looseness=0.55]
    (opt.south);

\end{tikzpicture}
\caption{Inner--outer workflow-generation loop. The dashed regions distinguish
the finite-pool inner optimization from the outer workflow-generation and
pool-update steps.}
\label{fig:inner-outer-loop}
\end{figure}

The next subsection formalizes this workflow-pricing interface.
\Cref{sec:inner-lp-dual} derives the corresponding task and slot prices from
the LP relaxation, and \Cref{sec:implicit-oracle} uses these prices to optimize over the implicit workflow class through an ellipsoid-based method.

\subsection{Cost-aware generation over an implicit workflow class}
\label{sec:workflow-access}

To search beyond the currently evaluated workflow pool, the optimizer needs a
way to identify promising workflows in the implicit class $\G$. As noted above,
this search is carried out by the pricing oracle, and it is guided by
nonnegative \emph{task prices} $w_i$ supplied by the optimizer. These prices
come from the current dual solution: a larger $w_i$ means that an additional
correct output on development task $i$ would create greater marginal value in
the current LP relaxation. The prices therefore direct the oracle toward the
\emph{residual tasks}, meaning those the current workflow pool still handles
poorly and on which an additional correct output would create the most value.

Given task prices $w=(w_1,\ldots,w_n)$ and the cost price $\gamma$, define the
global workflow-pricing value
\begin{equation}
    P(w,\gamma)
    =
    \max_{g\in\G}
    \left\{
        \sum_{i=1}^n w_i a_{ig}
        -
        \gamma c_g
    \right\}.
    \label{eq:global-pricing-value}
\end{equation}
The term $\sum_i w_i a_{ig}$ rewards a workflow for being correct on
high-price tasks, while $\gamma c_g$ penalizes its recurring execution cost.
We
refer to the difference
\[
    s_g(w)
    =
    \sum_{i=1}^n w_i a_{ig}
    -
    \gamma c_g
\]
as the \emph{score} of workflow $g$ at prices $w$, so that
$P(w,\gamma)=\max_{g\in\G}s_g(w)$. Accordingly, the pricing problem searches
for a workflow with the greatest price-weighted correctness net of compute
cost.

Let $\theta$ denote the current dual price of occupying an execution slot (i.e., the \emph{slot price}).
The largest violation among the workflow-indexed dual constraints is
\begin{equation}
    \Phi(w,\theta,\gamma)
    =
    P(w,\gamma)-\theta
    =
    \max_{g\in\G}
    \left\{
        s_g(w)
        -
        \theta
    \right\}.
    \label{eq:cost-aware-generation-problem}
\end{equation}
Hence, $\Phi(w,\theta,\gamma)>0$ means that some workflow creates more price-weighted value than the slot price. The pricing interface asks the oracle to identify workflows that perform well on tasks with high current prices, while screening out candidates whose recurring execution costs outweigh their potential contribution. In the customer-service routing example, high task prices may concentrate on cancellation, refund, or escalation cases for which an additional correct workflow output would be especially valuable.

Because exact maximization over $\G$ may itself be difficult, we allow an
approximate stochastic pricing procedure. A $(\delta,\beta)$
\emph{weak global pricing oracle} returns a workflow
$\widetilde g\in\G$ such that, conditional on the query history, with
probability at least $1-\beta$,
\begin{equation}
    s_{\widetilde g}(w)
    \ge
    P(w,\gamma)-\delta.
    \label{eq:additive-weak-pricing-oracle}
\end{equation}
Given a proposed slot price $\theta$, exactly one of two things happens.
If the returned workflow has score strictly greater than
$\theta$, then the proposed price cannot be right: a workflow in $\G$ is worth
more than the slot it would occupy, even after its recurring execution cost is
charged, so the optimizer has found a better workflow than it currently has.
If, however, the returned workflow has score at most $\theta$, then, on
the event that the guarantee in \eqref{eq:additive-weak-pricing-oracle} holds,
no workflow in the entire class scores more than $\delta$ above the proposed
price:
\[
    P(w,\gamma)
    \le
    \theta+\delta.
\]
Thus a single call to the oracle either produces a workflow that beats the
current slot price or certifies, up to tolerance $\delta$, that no workflow in
$\G$ can do so. This is precisely the weak-separation information required by
the ellipsoid method, which \Cref{sec:implicit-oracle} uses to optimize over
the implicit class $\G$ without enumerating it.

Recall that the oracle may be implemented by a teacher model, an automated workflow-search
procedure, an explicit search over a prespecified workflow grammar, or a human
engineering team. The theoretical guarantee developed later requires the oracle
to satisfy the global approximation property in
\eqref{eq:additive-weak-pricing-oracle}. In particular, the mere failure of a
search procedure to find a workflow with score above $\theta$ does not certify
that no such workflow exists.



\section{Selector Strength and the Value of Workflow Variety}
\label{sec:selector-curve}

This section studies how selector quality limits the value of running multiple 
workflows. Even if the firm can generate many diverse candidate workflows, 
variety is only beneficial when the selector can reliably identify correct 
answers from the realized candidate set. The results in this section 
formalize this limitation: they bound how much any multi-workflow portfolio 
can outperform the best single workflow, as a function of selector strength 
and workflow cost. These bounds provide pre-optimization screens that can 
inform the firm whether expanding the portfolio or improving the selector is the more valuable investment.

The remainder of this section develops these bounds in three steps. 
Section~\ref{sec:odds-lift} defines the odds-lift index $\Lambda$, a single 
scalar that summarizes how much the selector improves the odds of returning a 
correct output relative to picking a candidate at random. Section~\ref{sec:general-variety-bound} uses $\Lambda$ to bound the net 
value of any workflow pool relative to the best available singleton 
workflow, producing a screening rule that can rule out portfolio expansion 
before solving any optimization problem. Section~\ref{sec:pl-ranking} 
specializes this bound to the Plackett--Luce choice model, in which the 
odds-lift index reduces to a single discrimination parameter that can be 
estimated directly from data.

\subsection{Recovery envelopes and selector odds lift}
\label{sec:odds-lift}

The recovery curves $\{\psi_k\}$ introduced in Section~\ref{sec:model} may 
vary with candidate-set size $k$ and need not follow any particular 
random-utility model. To obtain a single structural bound that applies across 
all portfolio sizes, we impose the following assumption. 

\vspace{1mm}
\begin{assumption}[Common fraction-based recovery envelope]
\label{ass:recovery-envelope}
There exists a function $H:[0,1]\to[0,1]$ with $H(0)=0$ and $H(1)=1$ such 
that
\begin{equation}
    \psi_k(r)
    \le
    H\left(\frac{r}{k}\right),
    \qquad
    k=1,\ldots,K_{\max},\quad r=0,\ldots,k.
    \label{eq:common-recovery-envelope}
\end{equation}
\end{assumption}

\vspace{1mm}
\Cref{ass:recovery-envelope} states that selector performance can be bounded by a function 
of the \emph{fraction} of correct candidates alone, rather than the count 
$r$ and size $k$ separately. This reflects the premise that a selector's 
task difficulty is governed primarily by how concentrated the correct answer 
is among the alternatives it is shown, a premise that holds across many 
selector implementations. It is worth noting that \Cref{ass:recovery-envelope} is mild: some envelope 
always exists. Specifically, given any finite collection of curves 
$\{\psi_k\}_{k=1}^{K_{\max}}$, the choice
\begin{equation}
    H(p) :=\max\{\psi_k(r):1\le k\le K_{\max},\ 0\le r\le k,\ r/k\le p\}
\end{equation}
is a valid envelope, since $\psi_k(0)=0$ and $\psi_k(k)=1$ for every $k$. 

We now introduce an index that measures how much the selector improves the 
odds of a correct final output relative to random selection.

\vspace{1mm}
\begin{definition}[Selector odds-lift index]
\label{def:odds-lift}
For a recovery envelope $H$, define
\begin{equation}
    \Lambda(H)
    =
    \sup_{p\in(0,1)}
    \frac{H(p)/(1-H(p))}{p/(1-p)},
    \label{eq:odds-lift-index}
\end{equation}
with the convention that $\Lambda(H)=+\infty$ if $H(p)=1$ for some 
$p\in(0,1)$.
\end{definition}
\vspace{1mm}

The index $\Lambda(H)$ admits a direct interpretation as an odds ratio. At 
any correct fraction $p$, $p/(1-p)$ is the odds of returning a correct 
output under uniform random selection, and $H(p)/(1-H(p))$ is the 
corresponding odds under the selector's recovery envelope. Their ratio 
measures the selector's odds improvement over random chance at that 
fraction, and $\Lambda(H)$ takes the largest such improvement, i.e., the 
supremum, over all fractions $p\in(0,1)$.

Three cases illustrate the range of $\Lambda(H)$. A selector no better than 
random chance has $H(p)=p$ at every fraction, giving $\Lambda(H)=1$. A 
selector that systematically improves on random chance, without ever 
achieving certainty in returning the correct output unless all candidates are correct, has 
$\Lambda(H)\in(1,\infty)$. A selector that recovers 
the correct output with certainty at some interior fraction $p\in(0,1)$, 
that is, $H(p)=1$ while $p<1$, has infinite odds-lift, $\Lambda(H)=+\infty$, 
by the convention in Definition~\ref{def:odds-lift}. This last case is 
excluded whenever we assume $\Lambda(H)$ is finite, as we do throughout the 
results that follow.

The next lemma converts the  bound $\Lambda(H)\le\Lambda$ into an explicit closed-form envelope 
for $H$.

\vspace{1mm}
\begin{lemma}[Odds-lift envelope]
\label{lem:odds-lift-envelope}
If $\Lambda(H)\le\Lambda<\infty$ for some $\Lambda\ge1$, then
\begin{equation}
    H(p)
    \le
    h_\Lambda(p)
    :=
    \frac{\Lambda p}{1+(\Lambda-1)p},
    \qquad p\in[0,1].
    \label{eq:odds-lift-envelope}
\end{equation}
\end{lemma}
\vspace{1mm}

The function $h_\Lambda$ is the tight recovery envelope implied by an odds-lift
bound $\Lambda$: no selector with odds lift at most $\Lambda$ can exceed
$h_\Lambda$ at any correct fraction. We call this a Luce-shaped envelope. (The
formal Plackett--Luce specialization, for which this same functional form
arises exactly, is introduced in \Cref{sec:pl-ranking}.) The next property, the concavity of $h_\Lambda$, is what makes the envelope 
useful for comparing portfolios of different sizes.

\vspace{1mm}
\begin{lemma}[Shape of the odds-lift envelope]
\label{lem:odds-envelope-concavity}
For every $\Lambda\ge1$, the function $h_\Lambda$ is increasing and concave on $[0,1]$.
\end{lemma}
\vspace{1mm}

The key here is that the realized recovery curves $\psi_k$ need not themselves be concave. 
Lemmas~\ref{lem:odds-lift-envelope} and \ref{lem:odds-envelope-concavity} 
show that a finite odds-lift bound nonetheless forces every such curve to 
lie below a common concave, Luce-shaped, envelope. 

\subsection{The value of workflow variety}
\label{sec:general-variety-bound}

We now bound $\OPT_{\psi,\gamma}(\M)$, the best net value attainable using
workflow types from the evaluated pool $\M$. For each workflow type $g\in\M$, define its \emph{standalone accuracy} as its
average correctness on the development sample:
\begin{equation}
    \bar a_g
    =
    \frac1n\sum_{i=1}^n a_{ig}.
    \label{eq:standalone-accuracy}
\end{equation}

Because repeated assignments are allowed, the extremal accuracy and cost
quantities for a size-$k$ portfolio can be defined directly from the workflow
types. For each $k\le K_{\max}$, define
\begin{equation}
    A_k
    =
    \max_{g\in\M}\bar a_g,
    \qquad
    C_k
    =
    k\min_{g\in\M}c_g.
    \label{eq:screening-sequences}
\end{equation}
Indeed, all $k$ execution slots may be assigned to the same workflow type.
Thus, $A_k$ is the largest average standalone accuracy that can be assigned
to $k$ execution slots, whereas $C_k$ is the smallest total execution cost of
any $k$ slots. Consequently, every execution-count vector
$\bm m\in\mathbb Z_+^{\M}$ satisfying $k(\bm m)=k$ obeys
\[
    \frac1k\sum_{g\in\M}\bar a_g m_g
    \le
    A_k
    \qquad\text{and}\qquad
    C(\bm m)
    \ge
    C_k.
\]
The two bounds are computed separately and need not be attained by the same
workflow type.

Combining these quantities with the concave odds-lift envelope
$h_\Lambda$ from \Cref{sec:odds-lift} yields an upper bound on the net value
of every size-$k$ portfolio and, consequently, on
$\OPT_{\psi,\gamma}(\M)$.

\vspace{1mm}
\begin{theorem}[Selector-limited value of workflow variety]
\label{thm:diversification-bound}
Suppose \Cref{ass:recovery-envelope} holds and
$\Lambda(H)\le\Lambda<\infty$ for some $\Lambda\ge1$. For a finite evaluated
workflow pool $\M$, let
    $V_1
    =
    \max\big\{0, $ $\max_{g\in\M}\{\bar a_g-\gamma c_g\}\big\}$
denote the best net value achievable by a singleton portfolio or the outside
option. Then
\begin{equation}
    V_1
    \le
    \OPT_{\psi,\gamma}(\M)
    \le
    U_{\Lambda,\gamma}(\M)
    :=
    \max\left\{
        0,
        \max_{1\le k\le K_{\max}}
        \left[h_\Lambda(A_k)-\gamma C_k\right]
    \right\}.
    \label{eq:diversification-upper-bound}
\end{equation}
In particular, when $\gamma=0$, letting $a^\star=\max_g\bar a_g$ yields
\begin{equation}
    a^\star
    \le
    \OPT_{\psi,0}(\M)
    \le
    h_\Lambda(a^\star).
    \label{eq:no-cost-diversification-bound}
\end{equation}
We call the upper bound $\OPT_{\psi,0}(\M)\le h_\Lambda(a^\star)$ in
\eqref{eq:no-cost-diversification-bound} the no-cost upper bound: it applies
when the recurring workflow running costs are ignored, i.e., when $\gamma=0$. This
no-cost upper bound is tight over the class of selectors with odds lift at most
$\Lambda$.
\end{theorem}

\vspace{1mm}
The lower bound $V_1\le\OPT_{\psi,\gamma}(\M)$ immediately holds because both the singleton portfolios and the outside option are feasible. Thus, the main content of the theorem is really
the upper bound, which gives a quick test for whether any genuinely
multi-workflow portfolio can be worth considering. Specifically, if
\begin{equation}
    V_1
    \ge
    \max_{2\le k\le K_{\max}}
    \left[h_\Lambda(A_k)-\gamma C_k\right],
    \label{eq:singleton-screen}
\end{equation}
then no multi-workflow portfolio in the evaluated pool can outperform the best
singleton or the outside option. (Note that the maximum in \eqref{eq:singleton-screen}
starts at $k=2$ because $k=0$ and $k=1$ are already accounted for by $V_1$.)

\begin{corollary}[Selector-strength and uniform-cost implications]
\label{cor:diversification-implications}
Let $a^\star=\max_g\bar a_g$ be as in \Cref{thm:diversification-bound}. The following hold:
\begin{enumerate}
    \item If 
    $\gamma=0$,
    the maximum gain over the best workflow satisfies
    \begin{equation}
        \OPT_{\psi,0}(\M)-a^\star
        \le
        h_\Lambda(a^\star)-a^\star
        =
        \frac{(\Lambda-1)a^\star(1-a^\star)}
        {1+(\Lambda-1)a^\star}
        \le
        \frac{\sqrt\Lambda-1}{\sqrt\Lambda+1}.
        \label{eq:selector-only-gain-cap}
    \end{equation}
    \item If every workflow has the same execution cost $c$, and
    \begin{equation}
        \gamma c
        \ge
        h_\Lambda(a^\star)-a^\star,
        \label{eq:uniform-cost-singleton-condition}
    \end{equation}
    then the best nonempty portfolio is a singleton. With the outside option,
    the global optimum is either that singleton or no deployment.
\end{enumerate}
\end{corollary}

The corollary highlights the limit of adding more workflows. When
$\Lambda=1$, the envelope is $h_1(p)=p$, so no combination of workflows can
outperform the best individual workflow. As $\Lambda$ grows, stronger selection
can support more value from complementary workflows, but this value is capped by
the selector-strength term in \eqref{eq:selector-only-gain-cap}. The finite pool
$\M$ can be arbitrarily large: in the no-cost bound, its size affects the cap
only through the best singleton accuracy $a^\star$, not directly through
$|\M|$. Thus a larger pool cannot by itself overcome a weak selector.

\subsection{Plackett--Luce specialization}
\label{sec:pl-ranking}

The general bound in \Cref{sec:general-variety-bound} does not require a random-utility model. We now specialize to
the Plackett--Luce model, a standard model for choice and ranking in which each
alternative has an attraction weight and is selected with probability
proportional to that weight. Beyond its foundations in choice theory, the model
has been used extensively in computer science for Bayesian ranking, label
ranking, rank aggregation, and online preference learning
\citep{luce1959choice,plackett1975analysis,guiver2009bayesian,
cheng2010label,hajek2014minimax,szorenyi2015online}.
The model is useful here because it gives an interpretable one-parameter
measure of selector strength, closed-form fixed-size increments, and a
tractable marginal-value formula for the generation loop.

Recall from Section~\ref{sec:model} that, for a nonzero execution-count vector
$\bm m$, the selector $\operatorname{Sel}(X,O_{\bm m})$ observes the task, the candidate
outputs, and their execution traces and returns one candidate as the final
output. To model how it makes this 
choice, condition on a task instance and let $Z_j=\1\{A_j=Y^\star\}$ 
indicate whether candidate $j$ is correct. We posit that the selector 
assigns each candidate $j$ a latent score
\begin{equation}
    S_j=\eta Z_j+\varepsilon_j,
    \label{eq:pl-score}
\end{equation}
and returns the candidate with the highest score. Here $\eta\ge0$ is the 
selector's \emph{discrimination advantage}: it measures how much more 
attractive a correct candidate is to the selector than an incorrect one, on 
average. When $\eta=0$, correct and incorrect candidates are equally 
attractive and the selector behaves like a coin flip; larger $\eta$ means 
the selector systematically favors correctness. The terms $\varepsilon_j$ 
are i.i.d.\ standard Gumbel shocks representing idiosyncratic factors, 
unrelated to correctness, that also influence the selector's choice. The 
selector never observes $Z_j$ directly; \eqref{eq:pl-score} is a 
statistical description of the relationship between true correctness and 
the selector's realized preference, not a claim about the selector's 
internal reasoning.

Let $\lambda=\exp(\eta)\ge1$. By the Gumbel-max identity, choosing the 
candidate with the highest score $S_j$ is equivalent to choosing candidate 
$j$ with probability proportional to $\exp(\eta Z_j)$: each correct 
candidate receives attraction weight $\lambda = \exp(\eta)$ and each incorrect candidate 
receives weight $1 = \exp(0)$, and the selector returns a given candidate with 
probability equal to its weight divided by the sum of all weights. Summing 
this probability over all correct candidates, if $r$ of $k$ candidates are 
correct, the probability that the selector returns a correct answer is
\begin{equation}
    \psi^{\mathrm{PL}}_{k,\lambda}(r)
    =
    \frac{\lambda r}{\lambda r+k-r}
    =
    \frac{\lambda r}{k+(\lambda-1)r},
    \qquad r=0,1,\ldots,k.
    \label{eq:pl-recovery-curve}
\end{equation}
This is the Plackett--Luce recovery curve. Define
\begin{equation}
    h_\lambda(p)
    =
    \frac{\lambda p}{1+(\lambda-1)p},
    \qquad p\in[0,1].
    \label{eq:h-lambda}
\end{equation}
Dividing numerator and denominator of \eqref{eq:pl-recovery-curve} by $k$ 
shows $\psi^{\mathrm{PL}}_{k,\lambda}(r)=h_\lambda(r/k)$: the Plackett--Luce 
recovery probability depends on $r$ and $k$ only through the correct 
fraction $p=r/k$, exactly the fraction-based structure assumed in 
Assumption~\ref{ass:recovery-envelope}. Moreover, $h_\lambda$ has the same 
functional form as the envelope $h_\Lambda$ from 
Lemma~\ref{lem:odds-lift-envelope}, and a direct calculation confirms
\begin{equation}
    \frac{h_\lambda(p)/(1-h_\lambda(p))}{p/(1-p)}
    =\lambda,
    \qquad p\in(0,1).
    \label{eq:pl-odds-lift}
\end{equation}
Thus the Plackett--Luce discrimination parameter $\lambda$ is exactly its 
corresponding selector odds-lift index: a Plackett--Luce selector with strength $\lambda$ 
satisfies $\Lambda(H)=\lambda$ with $H=h_\lambda$, so the general bounds of 
Section~\ref{sec:general-variety-bound} apply to it with equality at 
$\Lambda=\lambda$, not merely as an upper bound.

The continuous concavity of $h_\lambda$ follows immediately from
\Cref{lem:odds-envelope-concavity} by setting $\Lambda=\lambda$. The next lemma
records the additional result needed for the inner
optimization problem in \Cref{sec:ip-lp-rounding}.

\vspace{1mm}
\begin{lemma}[Fixed-size Plackett--Luce increments]
\label{prop:pl-concavity}
For every $k\ge1$ and $\lambda\ge1$, $\psi^{\mathrm{PL}}_{k,\lambda}$ is
nondecreasing and discrete concave. Its increments are
\begin{equation}
    d^{\mathrm{PL}}_{\ell,k}
    =
    \psi^{\mathrm{PL}}_{k,\lambda}(\ell)
    -
    \psi^{\mathrm{PL}}_{k,\lambda}(\ell-1)
    =
    \frac{\lambda k}
    {\left(k+(\lambda-1)\ell\right)
     \left(k+(\lambda-1)(\ell-1)\right)},
    \quad \ell=1,\ldots,k.
    \label{eq:pl-discrete-increments}
\end{equation}
\end{lemma}

\vspace{1mm}
The estimation of $\lambda$ from validation data is
described in \Cref{sec:numerical-abcd}. Once $\lambda$ is computed,
the general portfolio objective from 
Section~\ref{sec:model} has a closed-form 
expression. For every nonzero execution-count vector $\bm m$, define
\begin{equation}
    J_\lambda(\bm m)
    =
    \frac1n\sum_{i=1}^n
    h_\lambda\left(
        \frac{r_i(\bm m)}{k(\bm m)}
    \right),
    \qquad
    \Pi_{\lambda,\gamma}(\bm m)
    =
    J_\lambda(\bm m)-\gamma C(\bm m),
    \label{eq:pl-endogenous-objective}
\end{equation}
with both functions set to zero at $\bm m=\bm 0$. Here
$J_\lambda(\bm m)$ plays the role of $J_\psi(\bm m)$ from
Section~\ref{sec:model}, but with the recovery curve
$\psi_{k(\bm m)}$ replaced by its Plackett--Luce form $h_\lambda$ evaluated
at the correct fraction $r_i(\bm m)/k(\bm m)$. This is the objective used 
in the optimization algorithms of Sections~\ref{sec:ip-lp-rounding} and 
\ref{sec:implicit-oracle}.

To study the structure of $J_\lambda$, we represent each execution slot as a distinct labeled copy of its workflow type. Thus, an execution-count vector $\bm m$ can be viewed as an ordinary subset of the ground set $\G\times\{1,\ldots,K_{\max}\}$ containing $m_g$ labeled copies of each workflow type $g$, with the second coordinate serving only to distinguish repeated executions. A natural question is whether $J_\lambda$, viewed as a set function under this representation, has the familiar properties of a coverage objective commonly found in submodular optimization.

The next proposition shows that neither property holds in general. The
endogenous-size objective can decrease when an execution is added and can
exhibit both increasing and decreasing marginal returns. Standard
monotone-submodular optimization tools therefore do not apply directly,
which motivates our development in
\Cref{sec:ip-lp-rounding}.

\begin{proposition}[More workflows need not be better]
\label{prop:endogenous-structure}
For every $\lambda\ge1$, $J_\lambda$ is generally nonmonotone with respect
to adding executions: there exist a feasible $\bm m$ and workflow type $g$
such that selector-aware accuracy decreases $J_\lambda(\bm m+\bm e_g)
    <
    J_\lambda(\bm m).$
Moreover, $J_\lambda$ is neither
submodular nor supermodular. Subtracting the modular workflow-cost term
$\gamma C(\bm m)$ preserves these failures, so
$\Pi_{\lambda,\gamma}$ has the same general structure.
\end{proposition}

\vspace{1mm}

The reason is a selection externality: increasing one component $m_g$ by one
changes not only the number of correct candidates $r_i(\bm m)$ but also the
total number of candidates $k(\bm m)$ faced by the selector. If the added
execution is incorrect on task $i$, then $r_i(\bm m)$ is unchanged while
$k(\bm m)$ increases, this lowers the correct fraction
$r_i(\bm m)/k(\bm m)$, and since $h_\lambda$ is 
increasing, strictly lowers recovery probability on that task; this is why 
$J_\lambda$ can decrease when a workflow is added, ruling out monotonicity. 
Conversely, because $h_\lambda$ is concave, the marginal gain from adding a
correct workflow depends on the composition of the existing portfolio. A
correct workflow can be more valuable after an incorrect workflow has entered
the portfolio and diluted the correct fraction, producing increasing marginal
returns and violating submodularity.
Supermodularity would require the opposite inequality: the marginal gain from
adding a workflow must weakly increase as the base portfolio becomes larger. This property also fails\footnote{For example, let $w_1$ and $w_2$ be incorrect
workflows and let $c_1$ be correct. Adding $c_1$ to $\{w_1\}$ produces the
marginal gain $h_\lambda(1/2)$, whereas adding it to the larger set
$\{w_1,w_2\}$ produces only $h_\lambda(1/3)$. Since $h_\lambda$ is increasing,
$h_\lambda(1/2)>h_\lambda(1/3)$, which violates increasing marginal returns.
Thus some configurations exhibit increasing marginal returns and others
exhibit decreasing marginal returns, so $J_\lambda$ is neither submodular nor
supermodular.}. In Appendix \ref{app:proofs}, we construct explicit examples exhibiting both failures.

\section{The Inner Optimization Problem}
\label{sec:ip-lp-rounding}

This section studies the portfolio problem for a given nonempty evaluated
workflow-type pool $\M\subseteq\G$, allowing repeated execution of the same
workflow type. The resulting finite-pool formulations serve two purposes. They
produce exact and approximate deployment plans for the current workflow pool,
and they reveal the dual-price structure used to search over the implicit
workflow class in \Cref{sec:implicit-oracle}.

The restricted problem decomposes exactly by run size:
\begin{equation}
    \OPT_{\psi,\gamma}(\M)
    =
    \max\left\{
        0,
        \max_{1\le k\le K_{\max}}\OPT_k(\M)
    \right\},
    \qquad
    \OPT_k(\M)
    =
    \max_{\substack{
        \bm m\in\mathbb Z_+^{\M}\\
        k(\bm m)=k
    }}
    \Pi_{\psi,\gamma}(\bm m).
    \label{eq:size-decomposition}
\end{equation}
This decomposition isolates the source of difficulty created by endogenous run
size. As shown in \Cref{prop:endogenous-structure}, adding an execution changes
both the number of correct candidates and the total number of alternatives
faced by the selector, so the global objective need not be monotone or
submodular. 
Within a fixed-size problem, however, every feasible execution-count vector
satisfies $k(\bm m)=k$. The selector therefore faces the same number of
candidates across all portfolios in that slice, and portfolio composition
affects recovery only through the correct counts $r_i(\bm m)$. Under \Cref{ass:concave-selector} below, each fixed-size accuracy objective is
a \emph{concave-coverage} function in the sense of
\citet{barman2021concave}. We formulate the fixed-size problem exactly as an
integer program, derive an LP relaxation and its dual prices, and develop a
randomized-rounding procedure with an explicit performance certificate
\citep{ageev2004pipage,chekuri2010dependent,calinescu2011maximizing}.

The remainder of this section proceeds in five steps. \Cref{sec:geometric-cardinality-grid} 
shows that a sparse, geometrically spaced cardinality grid certifies a 
near-optimal run size without solving the fixed-size problem for every $k$. 
\Cref{sec:fixed-size-concave-recovery} then fixes a run size and imposes a 
concavity assumption on each $\psi_k$, introducing a curvature parameter 
that will later control the tightness of the rounding guarantee. 
\Cref{sec:inner-ip} formulates an exact integer program for each fixed size 
and shows its optimal value coincides with $\OPT_k(\M)$. 
\Cref{sec:inner-lp-dual} relaxes this integer program to an LP whose dual 
prices identify the residual tasks most worth targeting with new 
workflows. \Cref{sec:inner-rounding} rounds the LP solution into a feasible 
portfolio and bounds how far its value can fall short of the true 
fixed-size optimum, yielding a certificate that holds across the entire 
cardinality grid.

\subsection{Sparse geometric grids for the optimal run size}
\label{sec:geometric-cardinality-grid}

\Cref{prop:endogenous-structure} rules out simply deploying the largest
feasible portfolio because more workflows need not increase net value. Thus, finding
the best run size requires comparing across $k=1,\ldots,K_{\max}$, not
assuming $k=K_{\max}$ is optimal. Doing this by solving $\OPT_k(\M)$ for every
candidate $k$ can be expensive, especially since the inner problem must be
resolved at every generation round as the pool grows. This subsection shows
that when a single concave function generates every recovery curve, checking
only a sparse, geometrically spaced set of sizes is enough to certify a
near-optimal size, without solving the fixed-size problem at every $k$.
Formally, this is the case when the fraction-based envelope of
\Cref{ass:recovery-envelope} holds with equality rather than merely as an
upper bound. Recall that \Cref{sec:pl-ranking} showed that Plackett--Luce recovery has
exactly this form.

\vspace{1mm}
\begin{theorem}[Geometric cardinality approximation]
\label{thm:geometric-cardinality-approx}
Suppose the recovery curves are generated by a common increasing concave
function $h:[0,1]\to[0,1]$ with $h(0)=0$, so that
$\psi_k(r)=h(r/k)$,
    $r=0,\ldots,k$.
Let
    $\OPT_k^+(\M)=\max\{0,\OPT_k(\M)\}$.
Then, for every $1\le b\le k\le K_{\max}$,
\begin{equation}
    \OPT_b^+(\M)
    \ge
    \frac{b}{k}\OPT_k^+(\M).
    \label{eq:cross-cardinality-stability}
\end{equation}
Consequently, if a cardinality grid
$\mathcal K\subseteq\{1,\ldots,K_{\max}\}$ has coverage ratio $\varrho\ge1$,
meaning that for every $k\le K_{\max}$ there exists $b\in\mathcal K$ with
$b\le k\le \varrho b$, then
\begin{equation}
    \max_{b\in\mathcal K}\OPT_b^+(\M)
    \ge
    \frac1\varrho
    \OPT_{\psi,\gamma}(\M).
    \label{eq:geometric-grid-guarantee}
\end{equation}
In particular, the dyadic grid gives a $1/2$-approximation, and a grid with
ratio $\varrho=1/(1-\eta_{\mathrm{grid}})$ gives a $(1-\eta_{\mathrm{grid}})$-approximation using
$O(\eta_{\mathrm{grid}}^{-1}\log K_{\max})$ fixed-size solves.
\end{theorem}

\vspace{1mm}
Inequality \eqref{eq:cross-cardinality-stability} says that a portfolio of
size $b\le k$ can always secure at least a $b/k$ fraction of the value
achievable at size $k$: shrinking the run size costs at most proportionally,
never more, because concavity of $h$ rules out returns to scale that vanish
faster than linearly. This is what makes a coarse grid safe: whichever true
optimal size $k^\star$ exhaustive search would find, some grid point $b$
within a factor $\varrho$ of $k^\star$ still retains at least a $1/\varrho$ share
of the optimal value, which is exactly \eqref{eq:geometric-grid-guarantee}.

Two concrete grids turn this general guarantee into the specific ratios. The \emph{dyadic grid}
\[
\mathcal K_{\mathrm{dyad}}=\{2^j:j=0,1,\ldots,\lceil\log_2K_{\max}\rceil\}
\cap\{1,\ldots,K_{\max}\}
\]
has coverage ratio $\varrho=2$ and $O(\log K_{\max})$
points, giving the $1/2$-approximation. More generally, for any $\varrho>1$,
the \emph{ratio-$\varrho$ grid} 
\[
\mathcal K_\varrho=\{\lceil\varrho^j\rceil:
j=0,1,\ldots,\lceil\log_\varrho K_{\max}\rceil\}\cap\{1,\ldots,K_{\max}\}
\]
has
coverage ratio $\varrho$ and $O(\log K_{\max})$ points. Taking
$\varrho=1/(1-\eta_{\mathrm{grid}})$ gives the $(1-\eta_{\mathrm{grid}})$-approximation using
$O(\eta_{\mathrm{grid}}^{-1}\log K_{\max})$ fixed-size solves. The full grid $\mathcal K=\{1,\ldots,K_{\max}\}$ corresponds to the
case $\varrho=1$, where every size is its own anchor and the guarantee in
\eqref{eq:geometric-grid-guarantee} becomes exact. Small pools can afford this exhaustive search. Large generation loops
resolve the inner problem at every round, so a sparse grid, dyadic or
otherwise, cuts the number of LP solves and oracle calls.

\subsection{Concave recovery and selector curvature at fixed size}
\label{sec:fixed-size-concave-recovery}

We now fix $k$ and build towards solving 
$\OPT_k(\M)$: exactly, through the integer program in \Cref{sec:inner-ip}, and at
scale, through the rounding certificate in \Cref{sec:inner-rounding}. This
subsection lays the groundwork with the following condition on the fixed-size recovery curve $\psi_k$.

\vspace{1mm}
\begin{assumption}[Concave recovery at each fixed size]
\label{ass:concave-selector}
For every $k=1,\ldots,K_{\max}$, $\psi_k$ is nondecreasing and discrete concave:
\begin{equation}
    d_{1,k}\ge d_{2,k}\ge\cdots\ge d_{k,k}\ge0,
    \qquad
    d_{\ell,k}=\psi_k(\ell)-\psi_k(\ell-1).
    \label{eq:concave-increments}
\end{equation}
When $d_{1,k}>0$, define
\begin{equation}
    c_{\psi,k}=1-\frac{d_{k,k}}{d_{1,k}}.
    \label{eq:selector-curvature}
\end{equation}
\end{assumption}

\vspace{1mm}
The curvature $c_{\psi,k}\in[0,1]$ measures how much recovery flattens as
correct candidates accumulate: it is zero when marginal recovery gains are
constant and approaches one when the last correct candidate contributes much
less than the first.
The requirements used below are nested. The exact
integer-program result in \Cref{prop:ip-exact} requires only that $\psi_k$ be
nondecreasing, equivalently, that $d_{\ell,k}\ge0$ for every $\ell$. This is weaker than \Cref{ass:concave-selector}, which additionally requires
the diminishing-increment inequalities
$d_{1,k}\ge\cdots\ge d_{k,k}$. The stronger condition is used only to identify
the fixed-size objective as concave coverage and to establish the LP-rounding
certificate.
Diminishing increments are natural when the first correct candidate provides
most of the selector's recoverable signal and additional correct candidates
are partly redundant. For example, once the candidate set already contains a
clearly correct answer, a second or third correct answer may provide less
additional help to the selector than the first. Plackett--Luce recovery satisfies
the condition exactly by \Cref{prop:pl-concavity}.

For the rounding analysis in \Cref{sec:inner-rounding}, it is useful to
extend the fixed-size curve beyond the deployed range $r=0,\ldots,k$:
\begin{equation}
    \widetilde\psi_k(r)
    =
    d_{k,k}r
    +
    \sum_{t=1}^{k-1}\left(d_{t,k}-d_{t+1,k}\right)\min\{r,t\},
    \qquad r\ge0.
    \label{eq:fixed-size-analytical-continuation}
\end{equation}
This continuation expresses the fixed-size accuracy function as a modular
linear term plus a nonnegative combination of threshold coverage functions
$\min\{r,t\}$, the form used in the rounding proof, and it agrees with
$\psi_k$ at the deployed integer points $r=0,\ldots,k$, so it changes
nothing about the value of any actual size-$k$ portfolio.

\subsection{Exact integer programs}
\label{sec:inner-ip}

For a fixed run size $k$, the selector-aware contribution of task $i$ under an
execution-count vector $\bm m$ is
$\psi_k(r_i(\bm m))$, which is generally nonlinear in the execution counts.
We linearize this term using binary tier indicators.

Let $z_g\in\mathbb Z_+$ denote the number of execution slots assigned to
workflow type $g$. For each task $i$ and tier
$\ell=1,\ldots,k$, let $y_{i\ell}\in\{0,1\}$ indicate whether the execution
plan produces at least $\ell$ correct candidates on task $i$. Because
\[
    \psi_k(r)
    =
    \sum_{\ell=1}^{r} d_{\ell,k},
    \qquad
    d_{\ell,k}
    =
    \psi_k(\ell)-\psi_k(\ell-1),
\]
an integral allocation $z$ induces the execution-count vector with
$m_g=z_g$. Weighting $y_{i\ell}$ by $d_{\ell,k}$ therefore reproduces
$\psi_k(r_i(\bm m))$ exactly. The formulation below enforces that the active tiers
form a prefix and that their total number cannot exceed the number of correct
candidate executions produced on the task. For each $k$, solve
\begin{align}
    I_k(\M)=
    \max_{y,z}\quad
    &\frac1n\sum_{i=1}^n\sum_{\ell=1}^{k}d_{\ell,k}y_{i\ell}
    -\gamma\sum_{g\in\M}c_gz_g
    \label{eq:cost-aware-ip}\\
    \text{s.t.}\quad
    &\sum_{\ell=1}^{k}y_{i\ell}
      \le
      \sum_{g\in\M}a_{ig}z_g,
      &&i=1,\ldots,n,\nonumber\\
    &y_{i\ell}\le y_{i,\ell-1},
      &&i=1,\ldots,n,\quad \ell=2,\ldots,k,\nonumber\\
    &\sum_{g\in\M}z_g=k,\nonumber\\
    &z_g\in\mathbb Z_+,
      && g\in\M,\nonumber\\
    &y_{i\ell}\in\{0,1\},
      &&i=1,\ldots,n,\quad \ell=1,\ldots,k.\nonumber
\end{align}

Solving \eqref{eq:cost-aware-ip} is not an approximation. The next result
shows its optimal value coincides exactly with the fixed-size optimum
$\OPT_k(\M)$, and it does so under only a nondecreasing recovery curve,
without invoking the concavity assumed in \Cref{sec:fixed-size-concave-recovery}.

\vspace{1mm}
\begin{proposition}[Exactness of the fixed-size inner problem]
\label{prop:ip-exact}
If $\M$ is nonempty and $\psi_k$ is nondecreasing, then
$I_k(\M)=\OPT_k(\M)$. Consequently,
\begin{equation}
    \OPT_{\psi,\gamma}(\M)
    =
    \max\left\{
        0,
        \max_{1\le k\le K_{\max}} I_k(\M)
    \right\}.
    \label{eq:exact-global-enumeration}
\end{equation}
\end{proposition}

\vspace{1mm}
Equation \eqref{eq:exact-global-enumeration} is the size decomposition
\eqref{eq:size-decomposition} with each $\OPT_k(\M)$ replaced by its exact integer program
value $I_k(\M)$, so endogenous run size does not require a monolithic
nonlinear formulation: exact enumeration solves \eqref{eq:cost-aware-ip} for
every $k$, while \Cref{thm:geometric-cardinality-approx} permits a sparse grid
instead when a controlled approximation is sufficient.

\subsection{LP relaxation, dual prices, and reduced costs}
\label{sec:inner-lp-dual}

The exact integer program in \Cref{sec:inner-ip} solves the inner problem for the current
pool, but it says nothing about which workflows outside the pool are worth
generating next. Extracting that information requires relaxing integrality
and reading off dual prices.\footnote{\citet{chen2024noisy} likewise use shadow prices to coordinate primal
resource allocations, although their focus is joint differential privacy
rather than search over an implicit workflow class.}

Relaxing integrality in \eqref{eq:cost-aware-ip}
gives
\begin{align}
    L_k(\M)=
    \max_{y,z}\quad
    &\frac1n\sum_{i=1}^n\sum_{\ell=1}^{k}d_{\ell,k}y_{i\ell}
    -\gamma\sum_{g\in\M}c_gz_g
    \label{eq:cost-aware-lp}\\
    \text{s.t.}\quad
    &\sum_{\ell=1}^{k}y_{i\ell}
      \le
      \sum_{g\in\M}a_{ig}z_g,
      &&i=1,\ldots,n,\nonumber\\
    &y_{i\ell}\le y_{i,\ell-1},
      &&i=1,\ldots,n,\quad \ell=2,\ldots,k,\nonumber\\
    &\sum_{g\in\M}z_g=k,\nonumber\\
    & z_g\ge0,
      && g\in\M, \nonumber\\
    &0\le y_{i\ell}\le1,
      &&i=1,\ldots,n,\quad \ell=1,\ldots,k.\nonumber
\end{align}
For fixed $z$, let $x_i(z)=\sum_g a_{ig}z_g$, now possibly fractional. Under
\Cref{ass:concave-selector}, the marginal weights $d_{\ell,k}$ are already
sorted in decreasing order, so filling the $y$ slots in index order up to
$x_i(z)$ is optimal; the accuracy component then equals
$\widetilde\psi_k(x_i(z))$, the same analytical continuation defined in
\eqref{eq:fixed-size-analytical-continuation}, now evaluated at a possibly
fractional argument rather than an integer one.

Let $\mu_i\ge0$ denote the dual variable associated with the task-$i$
coverage constraint, let $\nu_{i\ell}\ge0$ correspond to the prefix constraint
for tier $\ell$, let $\sigma_{i\ell}\ge0$ correspond to the upper bound
$y_{i\ell}\le1$, and let $\theta$ be the unrestricted dual variable associated
with the fixed-cardinality constraint $\sum_g z_g=k$. Because repeated
execution is allowed, the primal variables $z_g$ have no upper bounds; hence,
the dual contains no workflow-specific variables corresponding to constraints
of the form $z_g\le1$. Using the boundary convention
$\nu_{i1}=\nu_{i,k+1}=0$, a dual formulation of
\eqref{eq:cost-aware-lp} is
\begin{align}
    D_k(\M)=
    \min_{\mu,\nu,\theta,\sigma}\quad
    &k\theta+
      \sum_{i=1}^n\sum_{\ell=1}^k\sigma_{i\ell}
    \label{eq:cost-aware-dual}\\
    \text{s.t.}\quad
    &\mu_i+\nu_{i\ell}-\nu_{i,\ell+1}+\sigma_{i\ell}
      \ge\frac{d_{\ell,k}}n,
      &&i=1,\ldots,n,\quad \ell=1,\ldots,k,\nonumber\\
    &\sum_{i=1}^n a_{ig}\mu_i-\gamma c_g
      \le\theta,
      &&g\in\M,\nonumber\\
    &\mu_i,\sigma_{i\ell}\ge0,
      &&i=1,\ldots,n,\quad \ell=1,\ldots,k,\nonumber\\
    &\nu_{i\ell}\ge0,
      &&i=1,\ldots,n,\quad \ell=2,\ldots,k,\nonumber\\
    &\theta\ \text{free}.\nonumber
\end{align}

For a workflow type $g\in\G\setminus\M$, define its reduced-cost score by
\begin{equation}
    \Gamma_k(g;\mu,\theta)
    =
    \sum_{i=1}^n\mu_i a_{ig}
    -
    \gamma c_g
    -
    \theta.
    \label{eq:cost-aware-reduced-cost}
\end{equation}
The first term is the task-price-weighted value of the workflow's correct
outputs, $\gamma c_g$ is its recurring execution cost, and $\theta$ is the
dual price of one execution slot. Thus,
$\Gamma_k(g;\mu,\theta)>0$ means that workflow $g$ violates its full-dual
constraint
$
    \sum_{i=1}^n\mu_i a_{ig}-\gamma c_g\le\theta.
$
To see why all workflow-indexed constraints can be checked through one pricing
problem, recall the score function
$
    s_g(\mu)
    =
    \sum_{i=1}^n\mu_i a_{ig}
    -
    \gamma c_g.
$
Because repeated-execution multiplicities are uncapped, we have
\begin{equation}
    \max_{\substack{z_g\ge0,\ g\in\G\\
                    \sum_{g\in\G}z_g=k}}
    \sum_{g\in\G}s_g(\mu)z_g
    =
    k\max_{g\in\G}s_g(\mu).
    \label{eq:single-workflow-support}
\end{equation}
Hence it is enough to solve the single global pricing problem
\[
    \max_{g\in\G}s_g(\mu).
\]
If its value exceeds $\theta$, the maximizing workflow identifies a missing
workflow that could improve the current LP solution and should therefore be
added to the evaluated pool. If its value is at most $\theta$, no workflow in
the implicit class has enough price-weighted accuracy, net of execution cost,
to improve the current relaxation at these dual prices.

\subsection{Cost-preserving randomized rounding and certificates}
\label{sec:inner-rounding}

The integer program in \Cref{sec:inner-ip}, when exactly solved, already returns a deployable
portfolio, with no rounding required. The obstacle is scale, not exactness:
this is a concave-coverage problem, and coverage problems are known to be NP-hard in
general \citep{nemhauser1978analysis,feige1998threshold}, so solving the integer program
to optimality by branch-and-bound can become expensive when $\M$ or $n$ is
large. This subsection trades exactness for scalability: it solves
the efficient LP relaxation \eqref{eq:cost-aware-lp}, rounds the fractional
solution into a feasible portfolio, and certifies how close the result
comes to the true fixed-size optimum $\OPT_k(\M)$.

For any feasible fractional execution vector $z\ge0$ satisfying
$\sum_{g\in\M}z_g=k$, let
$
    x_i(z)
    =
    \sum_{g\in\M}a_{ig}z_g,
$
be the fractional number of correct executions assigned to task $i$. We separate
the relaxed objective into three components:
\begin{equation}
    Q_k(z)
    =
    \frac1n\sum_{i=1}^n
    \widetilde\psi_k\bigl(x_i(z)\bigr),
    \qquad
    R_k(z)
    =
    \frac1n\sum_{i=1}^n x_i(z),
    \qquad
    C_k(z)
    =
    \sum_{g\in\M}c_gz_g.
    \label{eq:lp-components}
\end{equation}
The corresponding relaxed net value
is
$
    F_k(z)
    =
    Q_k(z)-\gamma C_k(z).
$

Let $z^{k\star}$ be an optimal solution of the size-$k$ LP, and write
$
    Q_k^\star=Q_k(z^{k\star}),
$
$
    R_k^\star=R_k(z^{k\star}),
$
and
$
    C_k^\star=C_k(z^{k\star}).
$
Then
$
    L_k(\M)
    =
    F_k(z^{k\star})
    =
    Q_k^\star-\gamma C_k^\star.
$
To obtain an integral execution plan, consider any feasible fractional
execution vector $z$ satisfying $\sum_{g\in\M}z_g=k$, and define
\[
    q_g(z)
    =
    \frac{z_g}{k},
    \qquad g\in\M.
\]
Because $\sum_g z_g=k$, the vector $q(z)$ is a probability distribution over
workflow types. Draw $ G_1,\ldots,G_k
    \stackrel{\mathrm{iid}}{\sim}
    q(z),$
and define the random execution-count vector
\[
    M_{k,g}^{\RR}(z)
    =
    \sum_{j=1}^k\1\{G_j=g\},
    \qquad
    \bm M_k^{\RR}(z)
    =
    \bigl(M_{k,g}^{\RR}(z):g\in\M\bigr).
\]
Thus, each of the $k$ execution slots independently receives a workflow type
according to the fractional allocation, and
\[
    k\bigl(\bm M_k^{\RR}(z)\bigr)=k
    \quad\text{almost surely},\qquad
    \E\!\left[M_{k,g}^{\RR}(z)\right]=z_g.
\]
Hence the rounding procedure preserves every workflow multiplicity, and
therefore total execution cost, in expectation. Moreover, for task $i$, the
random number of correct candidates satisfies
\[
    R_i^{\RR}(z)
    :=
    r_i\bigl(\bm M_k^{\RR}(z)\bigr)
    \sim
    \mathrm{Binomial}\!\left(
        k,
        \frac{x_i(z)}{k}
    \right).
\]
When $z=z^{k\star}$, abbreviate
\[
    \bm M_k^{\RR}
    :=
    \bm M_k^{\RR}(z^{k\star}),
    \qquad
    M_{k,g}^{\RR}
    :=
    M_{k,g}^{\RR}(z^{k\star}),
    \qquad
    R_i^{\RR}
    :=
    R_i^{\RR}(z^{k\star}).
\]

The next theorem bounds this rounding scheme for any feasible fractional
vector $z$, not only an exact LP optimizer.

\vspace{1mm}
\begin{theorem}[Fixed-size net-value rounding certificate]
\label{thm:cost-rounding}
Suppose \Cref{ass:concave-selector} holds for run size $k$ and
$d_{1,k}>0$. Let $z$ be any feasible fractional execution vector for the
repeated-execution LP, and let $\bm M_k^{\RR}(z)$ be the size-$k$ execution
multiset obtained by the with-replacement rounding procedure described above.
Then
\begin{align}
    \E\left[
        \Pi_{\psi,\gamma}\bigl(\bm M_k^{\RR}(z)\bigr)
    \right]
    &\ge
    \left(1-\frac1e\right)Q_k(z)
    +
    \frac{d_{k,k}}{e}R_k(z)
    -
    \gamma C_k(z),
    \label{eq:net-rounding-lower-bound}\\
    F_k(z)
    -
    \E\left[
        \Pi_{\psi,\gamma}\bigl(\bm M_k^{\RR}(z)\bigr)
    \right]
    &\le
    \frac1e
    \left\{
        Q_k(z)-d_{k,k}R_k(z)
    \right\}
    \le
    \frac{c_{\psi,k}}{e}Q_k(z).
    \label{eq:net-rounding-fractional-gap}
\end{align}
In particular, if $z=z^{k\star}$ is an optimal solution of the size-$k$ LP,
then
\begin{equation}
    \OPT_k(\M)
    -
    \E\left[
        \Pi_{\psi,\gamma}\bigl(\bm M_k^{\RR}\bigr)
    \right]
    \le
    \frac1e
    \left(
        Q_k^\star-d_{k,k}R_k^\star
    \right)
    \le
    \frac{c_{\psi,k}}{e}Q_k^\star.
    \label{eq:net-rounding-additive-gap}
\end{equation}
\end{theorem}


\vspace{1mm}
The final bound has a uniform interpretation. Because
$0\le y_{i\ell}^{k\star}\le1$ and $\sum_{\ell=1}^k d_{\ell,k}
    =
    \psi_k(k)-\psi_k(0)
    =
    1,$
the LP accuracy component satisfies $0\le Q_k^\star\le1.$
Consequently,
\begin{equation*}
    0
    \le
    \OPT_k(\M)
    -
    \E\left[\Pi_{\psi,\gamma}(\bm M_k^{\RR})\right]
    \le
    \frac{c_{\psi,k}}{e}Q_k^\star
    \le
    \frac{c_{\psi,k}}{e}
    \le
    \frac1e.
\end{equation*}
Thus $c_{\psi,k}/e$ is a directly interpretable worst-case additive loss in
accuracy-equivalent units, while the bound involving $Q_k^\star$ can be
strictly sharper for the realized LP solution (see also \Cref{remark:uniform bound} at the end of this section).

The recurring execution-cost term does not incur any additional rounding loss.
Because with-replacement rounding preserves each workflow multiplicity in
expectation,
\[
    \E\left[
    C\bigl(\bm M_k^{\RR}\bigr)
\right]
=
\sum_{g\in\M}c_g
\E\left[M_{k,g}^{\RR}\right]
=
\sum_{g\in\M}c_gz_g^{k\star}
=
C_k^\star.
\]
Thus, all loss in the certificate arises from rounding the nonlinear
selector-aware accuracy term, not from execution cost. The with-replacement
procedure is the natural rounding scheme for the repeated-execution model,
because it permits the same workflow type to be assigned to multiple execution
slots while preserving expected multiplicities and execution cost.

Specializing \eqref{eq:selector-curvature} to Plackett--Luce recovery gives a
closed form for the curvature:
\begin{equation}
    c_{k,\lambda}^{\mathrm{PL}}
    =
    1-
    \frac{k+\lambda-1}
    {\lambda\{k+(\lambda-1)(k-1)\}}.
    \label{eq:pl-selector-curvature}
\end{equation}
This expression depends only on the run size $k$ and selector strength
$\lambda$. It satisfies $c_{1,\lambda}^{\mathrm{PL}}=0,
    c_{k,1}^{\mathrm{PL}}=0.$
For $k\ge2$ and $\lambda>1$, the curvature is nondecreasing in both $k$ and
$\lambda$. Moreover, for fixed $\lambda$, $\lim_{k\to\infty}c_{k,\lambda}^{\mathrm{PL}}
    =
    1-\frac1{\lambda^2}.$
Thus the rounding certificate is exact for a singleton portfolio and for a
random selector. As the portfolio grows or the selector becomes more
discriminating, the recovery curve becomes more curved: the first correct
candidate accounts for a larger share of the total recovery gain, and the
worst-case additive rounding certificate becomes looser. This does not mean
that a stronger selector reduces portfolio value; it means only that a linear
relaxation may approximate the more strongly curved recovery objective less
tightly.

For a cardinality grid $\mathcal K$, define the grid LP benchmark
\begin{equation}
    U_{\mathrm{LP}}(\mathcal K;\M)
    :=
    \max\left\{0,\max_{k\in\mathcal K}L_k(\M)\right\}.
    \label{eq:finite-pool-grid-lp-upper}
\end{equation}
Under Plackett--Luce recovery, $c_{k,\lambda}^{\mathrm{PL}}$ is
nondecreasing in $k$. Let $\bar k=\max\mathcal K$ and define the
worst-case rounding loss on the grid by
\[
    \varepsilon_{\mathrm{rnd}}(\mathcal K,\lambda)
    :=
    \frac{c_{\bar k,\lambda}^{\mathrm{PL}}}{e}
    =
    \frac1e
    \max_{k\in\mathcal K}c_{k,\lambda}^{\mathrm{PL}}.
\]

\vspace{1mm}
\begin{corollary}[Grid-level Plackett--Luce rounding guarantee]
\label{cor:pl-multiplicative-rounding}
Suppose the selector follows the Plackett--Luce recovery curve with
strength $\lambda\ge1$.
For each $k\in\mathcal K$, let
$\bm M_k^{\RR}$ be obtained by applying the with-replacement rounding procedure to
an optimal solution of the size-$k$ LP in
\eqref{eq:cost-aware-lp}. 
Assume the firm retains a feasible baseline execution-count vector $\bm m^0\in\mathbb Z_+^{\M}$, such as a validated incumbent deployment, satisfying $\Pi_{\lambda,\gamma}(\bm m^0) \ge \underline V > 0. $
After rounding, choose the best realized portfolio among the baseline and the
rounded grid candidates:
$
    \widehat{\bm M}
    \in
    \operatorname*{arg\,max}_{
        S\in
        \{\bm m^0\}
        \cup
        \{\bm M_k^{\RR}:k\in\mathcal K\}
    }
    \Pi_{\lambda,\gamma}(S).
$
Then, we have
\begin{equation}
    \E\!\big[
        \Pi_{\lambda,\gamma}(\widehat{\bm M})
    \big]
    \ge
    \max\left\{
        \underline V,\,
        U_{\mathrm{LP}}(\mathcal K;\M)
        -
        \varepsilon_{\mathrm{rnd}}(\mathcal K,\lambda)
    \right\}
    \ge
    \frac{
        \underline V
    }{
        \underline V
        +
        \varepsilon_{\mathrm{rnd}}(\mathcal K,\lambda)
    }
    U_{\mathrm{LP}}(\mathcal K;\M).
    \label{eq:grid-multiplicative-rounding}
\end{equation}
Moreover, if the size-$k$ integer programs are solved exactly instead, the same guarantee
holds without the expectation.
\end{corollary}

\vspace{1mm}
Note that because $L_k(\M)$ is an LP relaxation of the size-$k$ problem, we have $U_{\mathrm{LP}}(\mathcal K;\M) \ge \max\big \{ 0,\max_{k\in\mathcal K} \OPT_k(\M) \big\}.$ Thus,
\begin{eqnarray*}
\E\!\big[
        \Pi_{\lambda,\gamma}(\widehat{\bm M})
    \big] \, \ge \, \frac{
        \underline V
    }{
        \underline V
        +
        \varepsilon_{\mathrm{rnd}}(\mathcal K,\lambda)
    } \, \max_{k\in\mathcal K} \OPT_k(\M). 
\end{eqnarray*}

The first inequality in (\ref{eq:grid-multiplicative-rounding}) combines two safeguards. The retained baseline guarantees
value at least $\underline V$, while the rounded grid solution achieves the LP
benchmark up to the uniform additive rounding loss
$\varepsilon_{\mathrm{rnd}}(\mathcal K,\lambda)$. The second inequality converts
these two additive guarantees into a multiplicative bound relative to the grid
LP benchmark.

The resulting factor depends only on the baseline value and the largest
selector curvature among the grid sizes. Since
$
    0
    \le
    \varepsilon_{\mathrm{rnd}}(\mathcal K,\lambda)
    \le
    1/e,
$
the factor is always strictly positive whenever $\underline V>0$. Moreover,
the rounding loss vanishes when the fixed-size recovery curve is linear, in
which case the LP solution is preserved in expectation by the rounding
procedure.

\vspace{1mm}
\begin{remark}[On the uniform bound $1/e$] \label{remark:uniform bound}
    The fact that the uniform bound $1/e$ in \Cref{thm:cost-rounding} does not vanish as the selector becomes
perfect is structural. Under Plackett--Luce recovery, as
$\lambda\to\infty$, $\psi^{\mathrm{PL}}_{k,\lambda}(r)
    \to
    \1\{r\ge1\}.$
The fixed-size accuracy problem then reduces to maximum coverage: the firm
must choose $k$ workflows to maximize the fraction of tasks covered by at
least one correct workflow. Thus, even a perfect selector removes selection
error but not the combinatorial difficulty of constructing the portfolio.
The resulting $1-1/e$ rounding factor, or equivalently the uniform additive
loss of at most $1/e$ under normalized accuracy, is consistent with the
classical maximum-coverage barrier. This gap concerns the LP-rounding method;
it disappears if the fixed-size integer program is solved exactly.
\end{remark}

\section{Ellipsoid-Based Optimization Over an Implicit Workflow Class}
\label{sec:implicit-oracle}

The finite-pool formulations in \Cref{sec:ip-lp-rounding} introduce one primal
variable for each evaluated workflow type. When the feasible workflow class
$\G$ is large and represented only implicitly, the corresponding
repeated-execution LP may contain too many variables to enumerate. The dual,
however, has only finitely many price variables. Its only implicit component is
a family of constraints indexed by workflow types. This is exactly the setting
in which the ellipsoid method can optimize through separation rather than
explicit enumeration
\citep{groetschel1981ellipsoid,groetschel1988geometric}.

Fix a run size $k$. An equivalent projected dual of the full LP relaxation is
\begin{align}
    E_k(\G)
    =
    \min_{\mu,\xi,\theta}\quad
    &\sum_{i=1}^n \xi_i
    +
    k\theta
    \label{eq:ellipsoid-dual}\\
    \text{s.t.}\quad
    &\xi_i+r\mu_i
    \ge
    \frac{\psi_k(r)}{n},
    &&i=1,\ldots,n,\quad r=0,\ldots,k,
    \nonumber\\
    &0
    \le
    \mu_i
    \le
    \frac{d_{1,k}}{n},
    \qquad
    0
    \le
    \xi_i
    \le
    \frac1n,
    &&i=1,\ldots,n,
    \nonumber\\
    &-\gamma c_{\max}
    \le
    \theta
    \le
    d_{1,k},
    \nonumber\\
    &\sum_{i=1}^n \mu_i a_{ig}
    -
    \gamma c_g
    \le
    \theta,
    &&g\in\G.
    \nonumber
\end{align}
Appendix~\ref{app:ellipsoid-pricing-details} derives
\eqref{eq:ellipsoid-dual} and proves that
$
    E_k(\G)=L_k(\G).
$

\subsection{Single-workflow separation}
\label{sec:ellipsoid-separation}

\Cref{sec:workflow-access} introduced the global workflow-pricing value
\[
    P(w,\gamma)
    =
    \max_{g\in\G}
    \left\{
        \sum_{i=1}^n w_i a_{ig}
        -
        \gamma c_g
    \right\}
\]
and explained how it compares with the price $\theta$ of one execution slot.
The dual in \eqref{eq:ellipsoid-dual} now provides the formal origin of these
quantities: the task prices $w_i$ are precisely the dual variables $\mu_i$,
and the workflow-indexed constraints require
$
    P(\mu,\gamma)
    \le
    \theta.
$
Thus, the pricing interface from \Cref{sec:workflow-access} is exactly the separation oracle needed to optimize the implicit dual. We do not solve $P(\mu,\gamma)$ by explicitly enumerating or optimizing over $\G$. Instead, at each candidate dual solution, we pass the task prices and execution-cost penalty to the weak global pricing oracle, which searches for a high-scoring workflow according to \eqref{eq:additive-weak-pricing-oracle}.


In particular, if the oracle returns a workflow $\widetilde g$ satisfying
$
    \sum_{i=1}^n \mu_i a_{i\widetilde g}
    -
    \gamma c_{\widetilde g}
    >
    \theta,
$
then the constraint indexed by $\widetilde g$ in
\eqref{eq:ellipsoid-dual} is violated. That constraint supplies a separating
hyperplane, which the ellipsoid method uses to exclude the current candidate
dual point and continue its search. Conversely, suppose the oracle is $\delta$-accurate in the sense of
\eqref{eq:additive-weak-pricing-oracle} and returns a workflow whose score is
at most $\theta$. The guarantee established in
\Cref{sec:workflow-access} then implies
$
    P(\mu,\gamma)
    \le
    \theta+\delta.
$
Hence all workflow-indexed constraints are satisfied after increasing the slot
price from $\theta$ to $\theta+\delta$. Because $\theta$ has coefficient $k$
in the dual objective, this adjustment increases the objective by at most
$k\delta$.

The pricing oracle therefore implements weak separation for the implicit dual:
it either produces a violated workflow constraint or certifies feasibility of
the entire workflow-indexed constraint family up to an objective error of
$k\delta$.

\subsection{Batched stochastic pricing}
\label{sec:batched-stochastic-pricing}

The pricing oracle may itself be stochastic. We assume that there exists a
constant $p_{\mathrm{orc}}>0$ such that every primitive oracle call, conditional
on the full query history, satisfies the additive weak-pricing guarantee in
\eqref{eq:additive-weak-pricing-oracle} with probability at least
$p_{\mathrm{orc}}$. Thus, when queried at task prices $\mu$ with tolerance
$\delta$, a primitive call returns a workflow $\widetilde g$ satisfying
\[
    \sum_{i=1}^n \mu_i a_{i\widetilde g}
    -
    \gamma c_{\widetilde g}
    \ge
    P(\mu,\gamma)-\delta
\]
with conditional probability at least $p_{\mathrm{orc}}$.

A single unsuccessful pricing call cannot safely be used to certify that no
violated workflow constraint exists. More seriously, an invalid separator can
undermine the correctness of the entire ellipsoid routine. We therefore
amplify the primitive oracle at each separation query. Specifically, the
algorithm makes $m$ pricing calls at the same dual prices and retains the
returned workflow with the highest score. The batch is unsuccessful only if none of its $m$ primitive calls satisfies
the weak-pricing guarantee. Therefore, conditional on the history before the
batch,
\begin{equation}
    \Prob\{\text{batch fails}\mid\text{history}\}
    \le
    (1-p_{\mathrm{orc}})^m
    \le
    e^{-p_{\mathrm{orc}}m}.
    \label{eq:batch-failure-tail}
\end{equation}
Thus, batching converts a primitive stochastic pricing procedure with constant
success probability into a weak-separation oracle whose failure probability
decays exponentially in the batch size.

\subsection{Ellipsoid algorithm and finite-call guarantee}
\label{sec:dual-guided-generation}

We now combine the projected dual, the batched pricing oracle, and the
fixed-size rounding procedure into an end-to-end algorithm. The method solves
the implicit LP separately for each run size on the cardinality grid, recovers
a fractional execution allocation supported on workflows discovered during
separation, and rounds that allocation into a deployable portfolio. We write
$\varepsilon_{\mathrm{ell}}>0$ for the prescribed ellipsoid-method tolerance,
where the subscript ``$\mathrm{ell}$'' denotes the ellipsoid method.

\begin{algorithm}[t]
\caption{Ellipsoid-Based Dual-Guided Workflow Optimization}
\label{alg:cost-aware-generation}
\begin{algorithmic}[1]
\small
\STATE \textbf{Input}: cardinality grid $\mathcal K$; recovery curves
$\{\psi_k\}$; cost price $\gamma$; implicit-LP tolerance
$\varepsilon_{\mathrm{ell}}>0$; batch size $m$; primitive stochastic pricing
oracle; feasible baseline execution-count vector $\bm m^0$.
\FOR{$k\in\mathcal K$}
    \STATE Set the pricing tolerance $\tau_k
        \leftarrow
        \frac{\varepsilon_{\mathrm{ell}}}{2k}.$
    \STATE Initialize
    $\G_k^{\mathrm{rec}}\leftarrow\varnothing$
    and initialize the ellipsoid routine for
    \eqref{eq:ellipsoid-dual} with optimization tolerance
    $\varepsilon_{\mathrm{ell}}/2$; set $t\leftarrow0$.
    \WHILE{the ellipsoid routine has not met its stopping criterion}
        \STATE Let $(\mu^t,\xi^t,\theta^t)$ be the current ellipsoid iterate.
        \IF{an explicit constraint in \eqref{eq:ellipsoid-dual} is violated}
            \STATE Supply one such violated explicit constraint to the
            ellipsoid routine as a separating hyperplane.
        \ELSE
            \STATE Call the primitive pricing oracle $m$ times at
            $(\mu^t,\gamma,\tau_k)$, obtaining
            $\widetilde g^{t,1},\ldots,\widetilde g^{t,m}$, and set $g^t
                \in
                \arg\max_{j\in[m]}
                \left\{
                    \sum_{i=1}^n
                    \mu_i^t a_{i\widetilde g^{t,j}}
                    -
                    \gamma c_{\widetilde g^{t,j}}
                \right\}.$
            \STATE Record the returned workflow:
            $
                \G_k^{\mathrm{rec}}
                \leftarrow
                \G_k^{\mathrm{rec}}\cup\{g^t\}.
            $
            \IF{$
                \sum_{i=1}^n\mu_i^t a_{ig^t}
                -
                \gamma c_{g^t}
                >
                \theta^t
            $}
                \STATE Supply the violated workflow constraint
                $\sum_{i=1}^n\mu_i a_{ig^t}
                    -
                    \gamma c_{g^t}
                    \le
                    \theta$
                to the ellipsoid routine as a separating hyperplane.
            \ELSE
                \STATE Supply the weak-separation certificate
                $
                    P(\mu^t,\gamma)
                    \le
                    \theta^t+\tau_k
                $
                to the ellipsoid routine.
            \ENDIF
        \ENDIF
        \STATE Let the ellipsoid routine perform its update and stopping test;
        set $t\leftarrow t+1$.
    \ENDWHILE
    \STATE Perform primal recovery by solving the restricted fixed-size LP
    relaxation \eqref{eq:cost-aware-lp} with
    $\M=\G_k^{\mathrm{rec}}$. Let $(y^k,z^k)$ be the resulting recovered
    primal solution, extended by $z_g^k=0$ for
    $g\notin\G_k^{\mathrm{rec}}$, such that $z_g^k\ge0,
        \sum_{g\in\G_k^{\mathrm{rec}}}z_g^k=k,
        F_k(z^k)
        \ge
        L_k(\G)-\varepsilon_{\mathrm{ell}}.$
    \STATE Set $q_g^k=z_g^k/k$, draw
    $
        G_1^k,\ldots,G_k^k
        \stackrel{\mathrm{iid}}{\sim}q^k,
    $
    and define $M_{k,g}^{\RR}
        =
        \sum_{j=1}^k\1\{G_j^k=g\},
        \bm M_k^{\RR}
        =
        \bigl(
            M_{k,g}^{\RR}:
            g\in\G_k^{\mathrm{rec}}
        \bigr).$
\ENDFOR
\STATE Return the execution-count vector $\widehat{\bm M}$ with the largest
realized net value among the outside option $\bm 0$, the baseline $\bm m^0$,
and the rounded vectors
$\{\bm M_k^{\RR}:k\in\mathcal K\}$.
\normalsize
\end{algorithmic}
\end{algorithm}


For each $k\in\mathcal K$, let
$N_{\mathrm{ell},k}(\varepsilon_{\mathrm{ell}})$ be a deterministic upper
bound on the number of iterations of the while-loop in
Algorithm~\ref{alg:cost-aware-generation} before the ellipsoid routine reaches
optimization tolerance $\varepsilon_{\mathrm{ell}}/2$. Thus, in the run
corresponding to size $k$, the iteration index takes the values $t=0,\ldots,T_k-1, \text{for some}
    T_k
    \le
    N_{\mathrm{ell},k}(\varepsilon_{\mathrm{ell}}).$
Classical ellipsoid-method complexity bounds imply\footnote{Here $B$ denotes
an upper bound on the binary encoding length of the rational problem data; see
\citet{groetschel1981ellipsoid,groetschel1988geometric}.}
$N_{\mathrm{ell}}(\varepsilon_{\mathrm{ell}})
    :=
    \sum_{k\in\mathcal K}
    N_{\mathrm{ell},k}(\varepsilon_{\mathrm{ell}})
    =
    \sum_{k\in\mathcal K}
    \operatorname{poly}\!\left(
        n,
        k,
        B,
        \log\frac{1}{\varepsilon_{\mathrm{ell}}}
    \right).$
Because each iteration invokes at most one batch of $m$ primitive pricing
calls, Algorithm~\ref{alg:cost-aware-generation} makes at most $T
    :=
    m\,N_{\mathrm{ell}}(\varepsilon_{\mathrm{ell}})$
primitive pricing calls.


\begin{theorem}[High-probability implicit-class guarantee]
\label{thm:pricing-certificate}
Suppose $\G$ is a finite, implicitly represented workflow-type class whose
rational data have encoding length at most $B$, and suppose
$c_g\le c_{\max}$ for every $g\in\G$. Suppose the selector follows the
Plackett--Luce recovery curve with strength $\lambda\ge1$, and let the
cardinality grid $\mathcal K$ have coverage ratio $\varrho\ge1$. Retain a
feasible baseline portfolio $\bm m^0$ satisfying
$
    \Pi_{\lambda,\gamma}(\bm m^0)
    \ge
    \underline V
    >
    0.
$
Suppose further that every primitive pricing call satisfies
\eqref{eq:additive-weak-pricing-oracle}, conditional on the query history,
with probability at least $p_{\mathrm{orc}}>0$. Run Algorithm~\ref{alg:cost-aware-generation} with batch size $m$. Then, with
probability at least
\begin{equation}
    1
    -
    N_{\mathrm{ell}}(\varepsilon_{\mathrm{ell}})
    e^{-p_{\mathrm{orc}}m}
    =
    1
    -
    \exp\left\{
        -\frac{
            p_{\mathrm{orc}}T
        }{
            N_{\mathrm{ell}}(\varepsilon_{\mathrm{ell}})
        }
        +
        \log N_{\mathrm{ell}}(\varepsilon_{\mathrm{ell}})
    \right\},
    \label{eq:ellipsoid-oracle-success-probability}
\end{equation}
all batched separation calls satisfy their weak-pricing guarantees, and the
returned portfolio obeys
\begin{equation}
    \E_{\mathrm{rnd}}\!\left[
        \Pi_{\lambda,\gamma}(\widehat{\bm M})
    \right]
    \ge
    \max\left\{
        \underline V,\,
        \frac{1}{\varrho}
        \OPT_{\lambda,\gamma}(\G)
        -
        \varepsilon_{\mathrm{ell}}
        -
        \varepsilon_{\mathrm{rnd}}(\mathcal K,\lambda)
    \right\}.
    \label{eq:ellipsoid-additive-final}
\end{equation}
Consequently,
\begin{equation}
    \E_{\mathrm{rnd}}\!\big[
        \Pi_{\lambda,\gamma}(\widehat{\bm M})
    \big]
    \ge
    \frac{1}{\varrho}
    \frac{
        \underline V
    }{
        \underline V
        +
        \varepsilon_{\mathrm{ell}}
        +
        \varepsilon_{\mathrm{rnd}}(\mathcal K,\lambda)
    }
    \OPT_{\lambda,\gamma}(\G).
    \label{eq:ellipsoid-multiplicative-final}
\end{equation}
Here, the expectation $\E_{\mathrm{rnd}}$ is taken over the final randomized-rounding step,
conditional on the successful batched-separation event.
\end{theorem}

To make the probability of any failed batch at most $\beta$, it is sufficient
to choose
\begin{equation}
    m
    \ge
    \frac{1}{p_{\mathrm{orc}}}
    \left[
        \log N_{\mathrm{ell}}(\varepsilon_{\mathrm{ell}})
        +
        \log\frac{1}{\beta}
    \right].
    \label{eq:ellipsoid-batch-size}
\end{equation}
The resulting primitive-call complexity is
\begin{equation}
    T
    =
    O\!\left(
        \frac{
            N_{\mathrm{ell}}(\varepsilon_{\mathrm{ell}})
        }{
            p_{\mathrm{orc}}
        }
        \log
        \frac{
            N_{\mathrm{ell}}(\varepsilon_{\mathrm{ell}})
        }{
            \beta
        }
    \right).
    \label{eq:ellipsoid-total-call-complexity}
\end{equation}

The guarantee separates three distinct sources of loss. The factor
$1/\varrho$ comes from replacing exhaustive search over all run sizes by the
cardinality grid. The term $\varepsilon_{\mathrm{ell}}$ is the error from
solving the implicit LP only approximately, and
$\varepsilon_{\mathrm{rnd}}(\mathcal K,\lambda)$ is the loss from converting
the recovered fractional allocations into integral execution portfolios. The
stochastic pricing oracle affects the confidence level, but not the value bound
conditional on successful separation.

Most importantly, the benchmark in
\eqref{eq:ellipsoid-additive-final}--\eqref{eq:ellipsoid-multiplicative-final}
is the best portfolio over the full implicit workflow class $\G$, not merely
over the workflows explicitly generated during the algorithm. The method
therefore does not require enumerating $\G$ or discovering every workflow with
positive reduced cost; it requires only enough global pricing calls to optimize
the implicit dual to the prescribed tolerance.

\section{Stochastic Workflow Outcomes}
\label{sec:stochastic-workflow-outcomes}

Sections~\ref{sec:model} through~\ref{sec:implicit-oracle} treat $a_{ig}$ as
deterministic: workflow $g$ either solves development task $i$ or it does not,
and repeated executions reproduce the same outcome. In practice, an execution
of the same workflow on the same task need not return the same answer twice.
In this section, we reinterpret $a_{ig}\in[0,1]$ as the probability that a
single execution of workflow $g$ is correct on task $i$, with the deterministic
model recovered when every $a_{ig}$ is zero or one.

The deployment decision is unchanged. It remains the execution-count vector
$\bm m=(m_g)_{g\in\G}\in\mathbb Z_+^{\G}$, with finite support. What changes is that the
number of correct candidate answers on a task is now random rather than
determined by $\bm m$.

The stochastic model requires the following independent-sampling assumption.

\vspace{1mm}
\begin{assumption}[IID workflow execution]
\label{ass:stochastic-execution-law}
For every development task $i$ and workflow type $g$, there is a parameter
$a_{ig}\in[0,1]$. Each execution of workflow $g$ on task $i$ has a correctness
indicator distributed as $\mathrm{Bernoulli}(a_{ig})$. Correctness indicators
are mutually independent across tasks, workflow types, execution copies, and
sampling stages. Thus, for each fixed pair $(i,g)$, repeated executions are
i.i.d. Bernoulli draws with mean $a_{ig}$.
\end{assumption}

\vspace{1mm}
For a feasible $\bm m$, let $Z_{igq}\sim\mathrm{Bernoulli}(a_{ig})$ denote
the correctness indicator of copy $q$ of workflow $g$ on task $i$. Under
\Cref{ass:stochastic-execution-law}, the random number of correct outputs on
task $i$ is
\begin{equation}
    R_i(\bm m)
    =
    \sum_{g\in\G}\sum_{q=1}^{m_g}Z_{igq}
    \in\{0,\ldots,k(\bm m)\},
    \qquad
    \E[R_i(\bm m)]
    =
    \sum_{g\in\G}m_ga_{ig}.
    \label{eq:stoch-correct-count-main}
\end{equation}
The stochastic net value is
\begin{equation}
    \Pi_{\psi,\gamma}^{\mathrm{stoch}}(\bm m)
    =
    \E\left[
        \frac1n\sum_{i=1}^n
        \psi_{k(\bm m)}\bigl(R_i(\bm m)\bigr)
    \right]
    -\gamma C(\bm m),
    \qquad
    \Pi_{\psi,\gamma}^{\mathrm{stoch}}(\bm 0)=0.
    \label{eq:stoch-true-value-main}
\end{equation}
Recovery is evaluated at the realized correct count and only then averaged. In
general,
$\E[\psi_k(R_i(\bm m))]\ne\psi_k(\E[R_i(\bm m)])$, so the deterministic
objective is not recovered by substituting expected correctness for realized
correctness. For a finite workflow pool $\M\subseteq\G$, define
\begin{equation*}
    \OPT_{\psi,\gamma}^{\mathrm{stoch}}(\M)
    =
    \max_{\substack{\bm m\in\mathbb Z_+^{\M}\\
                    k(\bm m)\le K_{\max}}}
    \Pi_{\psi,\gamma}^{\mathrm{stoch}}(\bm m).
\end{equation*}

\vspace{1mm}
\paragraph{Selector-strength bound.}
The deterministic selector-strength bound of
\Cref{thm:diversification-bound} extends to stochastic execution. Only the
reading of $\bar a_g$ changes: under \eqref{eq:standalone-accuracy} it was the
fraction of development tasks workflow $g$ solves, whereas it is now the
average probability that one execution of $g$ is correct. Recall that
$h_\Lambda(p)=\Lambda p/\{1+(\Lambda-1)p\}$ from
\Cref{lem:odds-lift-envelope}.

\vspace{1mm}
\begin{proposition}[Selector odds lift under stochastic execution]
\label{prop:stoch-odds-lift-main}
Suppose \Cref{ass:recovery-envelope} holds and
$\Lambda(H)\le\Lambda<\infty$. Every feasible nonzero $\bm m$ satisfies
\begin{equation}
    \frac1n\sum_{i=1}^n
    \E\!\left[
        \psi_{k(\bm m)}\bigl(R_i(\bm m)\bigr)
    \right]
    \le
    h_\Lambda\!\left(
        \frac{1}{k(\bm m)}
        \sum_{g\in\G}m_g\bar a_g
    \right).
    \label{eq:stoch-odds-lift-main}
\end{equation}
For a finite pool $\M$, let
$a^\star=\max_{g\in\M}\bar a_g$,
$c_{\min}=\min_{g\in\M}c_g$, and $V_1^{\mathrm{stoch}}
    =
    \max\{
        0,
        \max_{g\in\M}(\bar a_g-\gamma c_g)
    \}.$
Then
\begin{equation}
    V_1^{\mathrm{stoch}}
    \le
    \OPT_{\psi,\gamma}^{\mathrm{stoch}}(\M)
    \le
    \max\left\{
        0,
        \max_{1\le k\le K_{\max}}
        \left[h_\Lambda(a^\star)-\gamma k c_{\min}\right]
    \right\}.
    \label{eq:stoch-variety-bound-main}
\end{equation}
If $\gamma=0$, then
\begin{equation}
    a^\star
    \le
    \OPT_{\psi,0}^{\mathrm{stoch}}(\M)
    \le
    h_\Lambda(a^\star),
    \qquad
    \OPT_{\psi,0}^{\mathrm{stoch}}(\M)-a^\star
    \le
    \frac{\sqrt\Lambda-1}{\sqrt\Lambda+1}.
    \label{eq:stoch-no-cost-cap-main}
\end{equation}
\end{proposition}

\vspace{1mm}
The conclusion is the same as in the deterministic model. Under the iid law,
$\E[R_i(\bm m)]=\sum_gm_ga_{ig}$, and Jensen's inequality reduces the bound
to the same average success probabilities $\bar a_g$. Thus stochastic
execution adds no new term to the selector-strength cap; the additional
statistical issue is that the probabilities $a_{ig}$ must now be estimated.

\vspace{1mm}
\paragraph{Implicit workflow generation.}
The probabilities $a_{ig}$ are unknown. Whenever a workflow $g$ is first made
available to the optimizer, we execute it independently $L_1$ times on every
development task and define
\begin{equation}
    \widehat a_{ig}^{(1)}
    =
    \frac1{L_1}\sum_{\ell=1}^{L_1}Z_{ig}^{(1,\ell)},
    \qquad
    Z_{ig}^{(1,\ell)}
    \stackrel{\mathrm{iid}}{\sim}
    \mathrm{Bernoulli}(a_{ig}).
    \label{eq:stoch-stage-one-estimator}
\end{equation}
The superscript $(1)$ marks the Stage-1 estimation sample. For the analysis,
an independent $L_1$-sample panel may be associated with every workflow in
the finite class $\G$ and revealed only when that workflow is queried. This
lazy-sampling interpretation does not require the algorithm to enumerate
$\G$. For every $\delta_1\in(0,1)$, Hoeffding's inequality and a union bound
give
\begin{equation}
    \Prob\left\{
        \max_{\substack{i=1,\ldots,n\\g\in\G}}
        \left|\widehat a_{ig}^{(1)}-a_{ig}\right|
        \le
        \varepsilon_a(L_1,\delta_1)
    \right\}
    \ge
    1-\delta_1,
    \qquad
    \varepsilon_a(L_1,\delta_1)
    :=
    \sqrt{
        \frac{\log(2n|\G|/\delta_1)}{2L_1}
    }.
    \label{eq:stoch-stage-one-uniform}
\end{equation}

Let $\widehat{\mathbb P}_1$ be the product-Bernoulli execution law obtained by
replacing $a_{ig}$ with $\widehat a_{ig}^{(1)}$, and let $\widehat\E_1$
denote expectation under this plug-in law. Define
\begin{equation}
    \widehat\Pi_{\psi,\gamma}^{(1)}(\bm m)
    =
    \widehat\E_1\left[
        \frac1n\sum_{i=1}^n
        \psi_{k(\bm m)}\bigl(R_i(\bm m)\bigr)
    \right]
    -\gamma C(\bm m),
    \qquad
    \widehat\Pi_{\psi,\gamma}^{(1)}(\bm 0)=0.
    \label{eq:stoch-plugin-value-main}
\end{equation}
On the concentration event in \eqref{eq:stoch-stage-one-uniform}, a
common-uniform coupling of the true and plug-in Bernoulli executions gives
\begin{equation}
    \sup_{\substack{\bm m\in\mathbb Z_+^{\G}\\
                    k(\bm m)\le K_{\max}}}
    \left|
        \widehat\Pi_{\psi,\gamma}^{(1)}(\bm m)
        -
        \Pi_{\psi,\gamma}^{\mathrm{stoch}}(\bm m)
    \right|
    \le
    \varepsilon_{\mathrm{est}}(L_1,\delta_1)
    :=
    \min\left\{
        1,
        K_{\max}\varepsilon_a(L_1,\delta_1)
    \right\}.
    \label{eq:stoch-estimation-error-main}
\end{equation}

Algorithm~\ref{alg:stoch-cost-aware-generation} in
Appendix~\ref{app:stochastic-workflow-proofs} applies the grid, ellipsoid,
batched-pricing, primal-recovery, and rounding arguments to the plug-in iid
law. Its workflow-indexed constraints are separated using the plug-in
stochastic pricing score in
\eqref{eq:stoch-plugin-reduced-cost-main}. Although the iid execution law is
determined by the marginal probabilities, this stochastic reduced cost is not
generally the deterministic mean-column score
$\sum_i\mu_i\widehat a_{ig}^{(1)}-\gamma c_g$: the marginal task price depends
on the realized number of other correct candidates in the same execution
scenario. Every newly proposed workflow is therefore evaluated using its own
fresh $L_1$-sample panel before its plug-in stochastic score is used.

\vspace{1mm}
\paragraph{Fresh-sample evaluation.}
Let $\M_T$ be the workflow pool produced by the generation stage. Because
$\M_T$ was assembled using the Stage-1 estimates, those same executions are
not reused to select the final portfolio. We freeze $\M_T$ and collect a
second, independent sample. Specifically, for every replication
$\ell=1,\ldots,L_2$, task $i$, workflow $g\in\M_T$, and potential execution
copy $q=1,\ldots,K_{\max}$, let $ Z_{igq}^{(2,\ell)}
    \sim
    \mathrm{Bernoulli}(a_{ig})$
be mutually independent Stage-2 draws, independent of the complete generation
history. For a candidate execution-count vector $\bm m$, define
\begin{equation}
    R_i^{(2,\ell)}(\bm m)
    =
    \sum_{g\in\M_T}\sum_{q=1}^{m_g}
    Z_{igq}^{(2,\ell)}.
    \label{eq:stoch-stage-two-count}
\end{equation}
Thus one Stage-2 replication supplies a common fresh execution table from
which every candidate portfolio over $\M_T$ can be evaluated. Set
\begin{equation}
    \widehat\Pi_{L_2}^{\mathrm{final}}(\bm m)
    =
    \frac1{L_2}\sum_{\ell=1}^{L_2}
    \left[
        \frac1n\sum_{i=1}^n
        \psi_{k(\bm m)}\bigl(R_i^{(2,\ell)}(\bm m)\bigr)
    \right]
    -\gamma C(\bm m),
    \qquad
    \widehat\Pi_{L_2}^{\mathrm{final}}(\bm 0)=0.
    \label{eq:stoch-final-sample-main}
\end{equation}
For a cardinality grid $\mathcal K$, define $\mathcal C_{\mathcal K}(\M_T)
    =
    \{\bm 0\}
    \cup
    \left\{
        \bm m\in\mathbb Z_+^{\M_T}:
        k(\bm m)\in\mathcal K
    \right\}.$
Let $N_T$ denote the size of this class. Since repeated execution is allowed,
\begin{equation}
    N_T
    =
    1+\sum_{k\in\mathcal K}
    \binom{|\M_T|+k-1}{k},
    \qquad
    \varepsilon_{\mathrm{samp}}(L_2,\delta_2)
    :=
    \sqrt{
        \frac{\log(2N_T/\delta_2)}{2nL_2}
    }.
    \label{eq:stoch-sampling-error-main}
\end{equation}
Conditional on the frozen generated pool, Hoeffding's inequality and a union
bound imply that, with probability at least $1-\delta_2$, every portfolio in
$\mathcal C_{\mathcal K}(\M_T)$ is evaluated within
$\varepsilon_{\mathrm{samp}}(L_2,\delta_2)$ of its true stochastic value.

Let
\begin{equation*}
    \widetilde{\bm m}
    \in
    \operatorname*{arg\,max}_{
        \bm m\in\mathcal C_{\mathcal K}(\M_T)
    }
    \widehat\Pi_{L_2}^{\mathrm{final}}(\bm m)
\end{equation*}
be the Stage-2 empirical maximizer. Retain a feasible baseline
$\bm m^0\in\mathcal C_{\mathcal K}(\M_T)$ whose true stochastic value is
known to satisfy
$\Pi_{\psi,\gamma}^{\mathrm{stoch}}(\bm m^0)\ge\underline V>0$, and return
\begin{equation}
    \widehat{\bm m}
    =
    \begin{cases}
        \widetilde{\bm m},
        &
        \widehat\Pi_{L_2}^{\mathrm{final}}(\widetilde{\bm m})
        -\varepsilon_{\mathrm{samp}}(L_2,\delta_2)
        \ge\underline V,
        \\[1mm]
        \bm m^0,
        &\text{otherwise}.
    \end{cases}
    \label{eq:stoch-conservative-final-rule}
\end{equation}
The rule switches away from the baseline only when the fresh Stage-2 value
clears $\underline V$ by more than the uniform sampling deviation.

\vspace{1mm}
The complete two-stage procedure is stated as
Algorithm~\ref{alg:stoch-cost-aware-generation} in the appendix. The following
theorem records its end-to-end guarantee.

\begin{theorem}[High-probability stochastic implicit-class guarantee]
\label{thm:stoch-end-to-end-main}
Suppose \Cref{ass:stochastic-execution-law} holds, $\G$ is a finite implicitly
represented workflow-type class with $c_g\le c_{\max}$. Suppose
$\psi_k(r)=h(r/k)$ for a common increasing concave function
$h:[0,1]\to[0,1]$ with $h(0)=0$ and $h(1)=1$, and let $\mathcal K$ have
coverage ratio $\varrho\ge1$. Retain a baseline $\bm m^0$ satisfying
$\Pi_{\psi,\gamma}^{\mathrm{stoch}}(\bm m^0)\ge\underline V>0$.

Run Algorithm~\ref{alg:stoch-cost-aware-generation} with sample sizes
$L_1,L_2$, confidence levels $\delta_1,\delta_2\in(0,1)$, ellipsoid tolerance
$\varepsilon_{\mathrm{ell}}>0$, and pricing-batch size $m$. Suppose every
primitive pricing call satisfies
\eqref{eq:stoch-plugin-oracle-guarantee-app}, conditional on the query history,
with probability at least $p_{\mathrm{orc}}>0$, and let
$N_{\mathrm{price}}^{\mathrm{stoch}}(\varepsilon_{\mathrm{ell}})$ bound the
total number of pricing batches. Define $\varepsilon_{\mathrm{rnd}}
    :=
    \max_{k\in\mathcal K}\frac{c_{\psi,k}}e.$
Then, with probability at least $1
    -\delta_1
    -\delta_2
    -N_{\mathrm{price}}^{\mathrm{stoch}}
        (\varepsilon_{\mathrm{ell}})
        e^{-p_{\mathrm{orc}}m},$
the returned execution-count vector satisfies
\begin{equation}
    \Pi_{\psi,\gamma}^{\mathrm{stoch}}(\widehat{\bm m})
    \ge
    \max\left\{
        \underline V,
        \frac1\varrho
        \OPT_{\psi,\gamma}^{\mathrm{stoch}}(\G)
        -\varepsilon_{\mathrm{ell}}
        -\varepsilon_{\mathrm{rnd}}
        -2\varepsilon_{\mathrm{est}}(L_1,\delta_1)
        -2\varepsilon_{\mathrm{samp}}(L_2,\delta_2)
    \right\}.
    \label{eq:stoch-full-class-main}
\end{equation}
Consequently,
\begin{equation}
    \Pi_{\psi,\gamma}^{\mathrm{stoch}}(\widehat{\bm m})
    \ge
    \frac1\varrho
    \frac{\underline V}{
        \underline V
        +\varepsilon_{\mathrm{ell}}
        +\varepsilon_{\mathrm{rnd}}
        +2\varepsilon_{\mathrm{est}}(L_1,\delta_1)
        +2\varepsilon_{\mathrm{samp}}(L_2,\delta_2)
    }
    \OPT_{\psi,\gamma}^{\mathrm{stoch}}(\G).
    \label{eq:stoch-multiplicative-final}
\end{equation}
\end{theorem}

\section{Numerical Experiments}
\label{sec:numerical-experiments}

The numerical study addresses two questions. First, how reliably can a
post-output selector identify a correct answer from a mixed candidate set, and
how does this ability vary across applications and selector models? We study
selector recovery on three domains (ABCD, Schema-Guided Dialogue (SGD), and
HotpotQA) using three selector model families. Second, does selector-aware
workflow portfolio design improve deployment performance? We answer these
questions using each data set by evaluating the complete pipeline: initial workflow
evaluation, selector calibration, cost-aware portfolio optimization,
dual-guided workflow generation, selector recalibration after the workflow
bank expands, and fresh held-out deployment.

The three domains differ in the type of decision the system must make and the
information available for making it. ABCD \citep{chen2021abcd} is a
customer-service routing task with a common operational taxonomy: after
observing the first few customer turns, the system must identify the
appropriate service subflow. SGD \citep{rastogi2020sgd} contains task-oriented
conversations between users and virtual assistants across many services and
domains. Each service is accompanied by a schema describing the intents it
supports and the information fields relevant to those intents. The system must
therefore interpret a conversation relative to a service-specific set of
possible intents and slots, rather than a single taxonomy shared across all
conversations. HotpotQA \citep{yang2018hotpotqa} is an open-domain
question-answering task in which answering a question typically requires
combining information from multiple passages. It therefore provides a
substantively different setting in which candidate workflows must perform
multistep evidence-based reasoning. Together, the three domains span
fixed-taxonomy service routing, schema-dependent dialogue understanding, and
multistep question answering. 

\vspace{1mm}
\paragraph{Common experimental protocol.}
Candidate-generating workflows are stochastic. For every task--workflow pair,
we execute the workflow independently $L=5$ times using
\texttt{do\_sample=True}, temperature one, and distinct random seeds. We retain
each execution's output, correctness indicator, seed, and parse status and
estimate
\[
    \widehat a_{ig}
    =
    \frac{1}{5}\sum_{\ell=1}^{5}Z_{ig\ell}.
\]
Selectors are evaluated deterministically
(\texttt{do\_sample=False}). Candidate order is randomized to reduce position
effects. When one task contributes multiple selector panels or orderings,
confidence intervals treat the task, rather than each individual selector
decision, as the independent sampling unit.

In all three end-to-end experiments, we set $K_{\max}=6$ and search $K\in\{1,\ldots,6\}$. The primary formulation permits repeated execution of a workflow type, and we also report a no-repeat benchmark imposing $m_g\leq1$. We optimize the LP relaxation using a 16-scenario sample-average approximation (SAA), where each scenario is drawn from the plug-in independent Bernoulli execution law defined by the estimated probabilities $\widehat a_{ig}$. We then apply 512 randomized-rounding trials. For every returned integral portfolio, we compute the distribution of the number of correct outputs exactly: because this count is a sum of independent Bernoulli variables with potentially different probabilities, it follows a Poisson--binomial distribution. We average the selector recovery curve over this exact distribution rather than reporting the finite 16-scenario average.

A correct final decision is valued at \$1, recurring compute cost is expressed using the small-model reference-call normalization described below, and we set
$\gamma=1$. To conserve space, we report ABCD in detail, summarize the SGD and HotpotQA results in \Cref{sec:numerical-other-domains}, and provide their
complete experiments in \Cref{app:sgd-end-to-end,app:hotpotqa-end-to-end} in the Appendix.



\subsection{ABCD customer-service routing}
\label{sec:numerical-abcd}

\paragraph{Task and data.}
ABCD is a customer-service dialogue dataset organized around operational flows
and policy-constrained service actions \citep{chen2021abcd}. We study an early
routing task: after observing at most the customer's first three messages and
the available dialogue context, the system must identify the appropriate service subflow from 96 possible labels, including routes
associated with order cancellation, refund status, subscription changes,
troubleshooting, and escalation.
We preserve the official data splits
and the naturally occurring distribution of labels within each split. The
experiment uses 800 development tasks to evaluate workflows and optimize the
portfolio, all 1,004 available tasks from the development split to calibrate
the selector, and 800 tasks from the official test split for final held-out
evaluation. No single label accounts for more than 4.2\% of the observations
in any split. Thus, performance cannot be driven by repeatedly predicting a
small number of dominant classes; the system must distinguish among a large
set of possible routing decisions.

\vspace{1mm}
\paragraph{Initial workflow bank and stochastic execution.}
We begin with a deliberately simple bank of 18 single-call workflows, formed by
crossing three generator models (Qwen2.5-3B-Instruct,
Mistral-7B-Instruct-v0.3, and Granite-3.3-8B-Instruct) with six prompting
strategies. The strategies differ in the reasoning process the model is asked to follow
before making the routing decision. The \emph{direct} strategy asks for the
routing label immediately. \emph{Evidence first} asks the model to identify
the parts of the conversation that are most informative for the decision
before choosing a label. \emph{Decomposition} asks it to consider the
customer's intent, the current service state, and the routing logic separately
before combining them into a final decision. \emph{Verify and revise} asks the
model to make an initial judgment, check it for possible errors, and revise it
if needed within the same response. \emph{Alternatives} asks the model to
compare several plausible routing labels before selecting one, while
\emph{limited context} deliberately withholds part of the available dialogue
information to create a less informed workflow.
Importantly, these differences do not change
the workflow topology: every initial workflow consists of a single model call,
with no separate verifier, fallback call, or multistep agent interaction. This
provides a controlled starting bank against which we later evaluate whether
richer multistep agentic workflows add value.

We set the reference cost of one 7B-model call to $0.001$ and scale execution
cost linearly with model size and the number of model calls.\footnote{As of
August 22, 2026, Together AI prices Qwen2.5-7B-Instruct-Turbo, a 7B-parameter
model, at $0.30$ per million input tokens and $0.30$ per million output tokens.
At these rates, a call to this 7B model using a total of 3,333.33 input and
output tokens costs $3{,}333.33\times\frac{0.30}{10^6}=0.001.$
Thus, the $0.001$ reference cost corresponds directly to the market price of a
3,333.33-token call to a hosted 7B model.}
The generated workflows instead consist of executable compositions of
GPT-4o-mini calls connected through Python control flow. Depending on the
generated program, calls may be sequenced or run in parallel and combined
through voting, critique, verification, clarification, or conditional
fallback. 
OpenAI describes GPT-4o-mini as a small model and, at its release, characterized it as belonging to roughly the same small-model tier as models such as Llama~3~8B. For consistency in experimental cost accounting, we assign each GPT-4o-mini invocation one reference small-model call and the same normalized cost of \$0.001.

Across the 800 development tasks, the Mistral decomposition workflow has the
highest estimated one-execution accuracy, at 46.625\%. However, when each of
the 18 initial workflows is executed once, the probability that at least one
of them produces the correct answer rises to 69.896\%. This difference shows
that the workflows are complementary: many tasks missed by the best individual
workflow are solved by another workflow in the pool. 

\begin{table}[t]
\centering
\caption{ABCD design and selector calibration. The primary selector
uses execution slots rather than answer deduplication and aggregates cyclic
candidate-order rotations by vote. Confidence intervals for $\lambda$ are
clustered at the task level.}
\label{tab:abcd-design-calibration}
\small
\begin{tabular}{p{0.39\textwidth}p{0.52\textwidth}}
\toprule
Quantity & setting or estimate \\
\midrule
Development / calibration / held-out tasks
    & $800 / 1{,}004 / 800$ \\
Allowed routing labels & 96 ABCD subflows \\
Initial workflow bank
    & 18 workflows: 3 generator models $\times$ 6 reasoning strategies \\
Stochastic execution
    & $L=5$ independent draws per task--workflow cell; temperature one \\
Run-size search
    & $K\in\{1,2,3,4,5,6\}$; repeats allowed in the primary analysis \\
Workflow generation
    & one-scenario common-uniform ellipsoid pricing; 20 distinct-candidate
      evaluation budget \\
Final portfolio optimization
    & 16 SAA scenarios and 512 randomized-rounding trials per $K$ \\
Selector
    & Qwen2.5-7B-Instruct; greedy decoding; up to six cyclic rotations \\
Initial-bank calibration sample
    & 1,500 balanced panels from 586 distinct tasks \\
Initial-bank selector strength
    & $\widehat\lambda=2.9848$, 95\% CI $[2.4930,3.5844]$ \\
Initial-bank implied pairwise recovery
    & $\widehat\lambda/(1+\widehat\lambda)=74.90\%$ \\
Average selector advantage over uniform choice
    & 21.33 percentage points on the initial-bank calibration design \\
Expanded-bank selector strength
    & $\widehat\lambda=2.5229$, 95\% CI $[2.0918,3.0670]$ \\
Expanded-bank implied pairwise recovery
    & $\widehat\lambda/(1+\widehat\lambda)=71.61\%$ \\
Selector parse rate
    & 100\% on both initial- and expanded-bank calibration panels \\
\bottomrule
\end{tabular}
\end{table}

\vspace{1mm}
\paragraph{Precision and stability of the $L=5$ execution design.}
With only five executions for each task-workflow pair, the individual
$\widehat a_{ig}$ estimates can be noisy.\footnote{For reference, estimating a
single Bernoulli success probability with a worst-case 95\% Wilson interval
requires $L=93$ for a $\pm0.10$ margin and $L=381$ for a $\pm0.05$ margin.}
Our goal, however, is not to estimate every task-workflow success probability
precisely. Rather, we need sufficiently stable estimates of each workflow's
average performance across the 800 development tasks to support portfolio
construction. We therefore examine whether the workflow-level results are sensitive to using
only a few executions. For each workflow, we recompute its average accuracy
using only the first $L'=1,2,3$ executions and compare the resulting workflow
ranking with the ranking based on all $L=5$ executions. We measure agreement
using Spearman's rank correlation \citep{spearman1961proof}, where values close
to one indicate similar rankings. The correlations are 0.945, 0.971, and
0.958 for $L'=1,2,3$, respectively. Relative to the $L=5$ estimates, the largest absolute change in any workflow's average accuracy across these
comparisons is 1.25 percentage points, approximately 2.7\% of the best
workflow's 46.625\% average accuracy, 
and four of the five highest-ranked workflows remain
in the top five. These results suggest that, although individual
task-workflow estimates remain noisy, $L=5$ provides reasonably stable
workflow-level information for portfolio construction.

\vspace{1mm}
\paragraph{Bank-specific selector calibration.}
We next estimate how well the selector can distinguish correct from incorrect
outputs produced by the initial ABCD workflow bank. For each
$K\in\{2,\ldots,6\}$ and each possible number of correct candidates
$r\in\{1,\ldots,K-1\}$, we construct 100 candidate panels containing exactly
$r$ correct execution outputs and $K-r$ incorrect outputs. The selector
observes the task, dialogue context, and candidate labels, but does not observe
workflow identity, generator identity, execution cost, or which candidates are
correct. The resulting calibration sample contains 1,500 panels drawn from 586
distinct tasks.

We estimate selector strength by maximum likelihood under the Plackett--Luce
recovery model. Let $S_j\in\{0,1\}$ indicate whether the selector chooses a
correct candidate on calibration panel $j$, where $r_j$ of the $K_j$
candidates are correct. Under the Plackett--Luce specification,
\[
    \operatorname{logit}
    \Pr(S_j=1\mid r_j,K_j)
    =
    \alpha
    +
    \log\!\left(\frac{r_j}{K_j-r_j}\right),
    \qquad
    \lambda=\exp(\alpha).
\]
Thus, $\lambda$ measures how strongly the selector favors correct candidates
relative to incorrect ones, and we estimate it by fitting the logistic model
above. The resulting estimate is $\widehat\lambda=2.9848$, with a
task-clustered 95\% confidence interval of $[2.4930,3.5844]$. Under the fitted
model, when exactly one of two candidates is correct, the selector chooses the
correct candidate with probability 74.90\%. Averaged across all panel
compositions in the calibration design, its fitted probability of selecting a
correct candidate is 21.33 percentage points higher than under uniform random
selection.

\Cref{fig:abcd-selector-recovery} compares the observed selector recovery
rates with two benchmarks: uniform random selection and the fitted
Plackett--Luce curve. The fitted curve captures the main pattern that selector
performance improves as the fraction of correct candidates increases, although
the empirical recovery rates do not lie exactly on the one-parameter model.
We thus use Plackett--Luce as a simple operational approximation for
portfolio optimization rather than as an exact behavioral description of the
selector. To check whether this approximation leads to useful deployment
decisions, we subsequently evaluate the optimized portfolios using the actual
selector on fresh held-out tasks. The final deployment results therefore do not
depend solely on the fitted parametric recovery model.

\begin{figure}[t]
    \centering
    \includegraphics[width=0.64\textwidth]{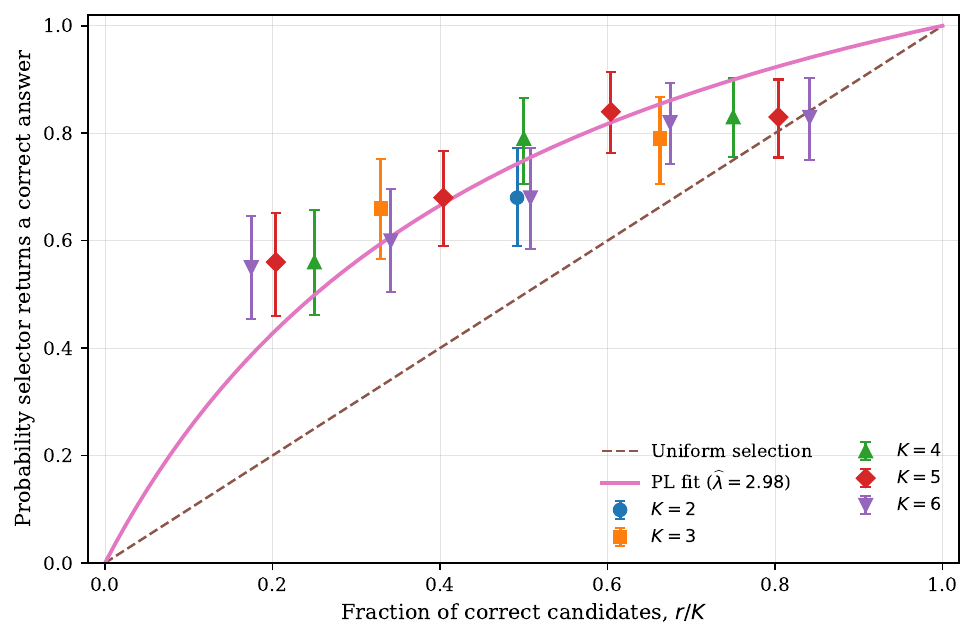}
    \caption{ABCD selector recovery. Points show task-clustered estimates for
    controlled $(K,r)$ cells, vertical bars show 95\% task-clustered bootstrap
    intervals, the dashed line is uniform selection, and the solid curve is the
    fitted Plackett--Luce model. Small horizontal offsets separate cells with
    the same correct fraction.}
    \label{fig:abcd-selector-recovery}
\end{figure}

\vspace{1mm}
\paragraph{Initial-pool optimization.}
We next test whether the complementarity in the initial bank can be converted into deployed value after accounting for selector confusion and compute cost.
We set $K_{\max}=6$, matching the largest candidate set used in selector calibration. For every $K\in\{1,\ldots,6\}$, we solve the repeat-allowed stochastic LP relaxation using 16 independent execution-outcome scenarios drawn from the plug-in Bernoulli law. We then apply randomized rounding 512 times and evaluate each distinct integral portfolio using the exact Poisson--binomial calculation described above. The best portfolio from the initial
workflow bank has $K=6$ and contains one Granite verify-and-revise execution,
one Mistral decomposition execution, two Mistral verify-and-revise executions,
one Qwen evidence-first execution, and one Qwen verify-and-revise execution.
On the development sample, its calibrated selector accuracy is 52.523\%, and
its workflow-cost-adjusted value is 0.520074 per task. Thus, the optimized
portfolio assigns two execution slots to the Mistral verify-and-revise
workflow. When repeated execution is prohibited, the best no-repeat portfolio
has a slightly lower development value of 0.518755.

\vspace{1mm}
\paragraph{Dual-guided workflow generation.}
We use ADAS Meta Agent Search \citep{hu2024adas} to generate new workflows,
with OpenAI's GPT-4o-mini serving both as the meta-agent that proposes workflow
structures and as the execution model within the generated workflows. The
search is guided by task prices from the stochastic dual. These prices identify
development tasks on which an additional correct workflow output would be most
valuable, allowing ADAS to focus its search on weaknesses of the current
workflow bank. To implement this idea, for each $K$ we run the ellipsoid separation procedure
associated with \Cref{alg:stoch-cost-aware-generation}, using a one-scenario
sample-average approximation (SAA) of the stochastic projected dual. Because
the stochastic dual assigns a separate slot price $\theta_q$ to each labeled
execution copy $q=1,\ldots,K$, workflow pricing is checked separately for each
copy. At a given ellipsoid iterate, we first check the explicit dual
constraints and the workflow constraints corresponding to workflows already
evaluated. ADAS is invoked only if these constraints do not separate the
current iterate. Whenever ADAS proposes a new workflow, we freeze its design
and evaluate it independently using the same $L=5$ execution protocol before
using its estimated performance in subsequent pricing calculations.

We impose a prespecified budget of 20 distinct new workflow evaluations during
generation. After the generation stage, the final finite-pool portfolio
optimization uses 16 SAA scenarios and 512 randomized-rounding trials. Across
the fixed-$K$ searches, the procedure makes 26 ADAS pricing queries and
incorporates four generated workflows, expanding the evaluated bank from 18
to 22 workflow types. Because each fixed-$K$ search is subject to its
prespecified workflow-evaluation budget, the resulting expanded bank should be
viewed as the output of a budgeted dual-guided search rather than an exhaustive
search over the full workflow class. 

\vspace{1mm}
\paragraph{Expanded-bank recalibration and reoptimization.}
After adding the generated workflows, we recalibrate the same Qwen selector
using a new controlled panel constructed from outputs in the expanded workflow
bank. This step allows the estimated selector model to reflect the new kinds of
candidate outputs it will encounter after workflow generation. The estimated
selector strength is $\widehat\lambda=2.5229$, with a 95\% confidence interval
of $[2.0918,3.0670]$. This interval overlaps the initial-bank confidence
interval, although the point estimate is lower. The comparison illustrates why
we recalibrate after expanding the workflow bank: $\lambda$ summarizes selector
performance for the candidate distribution it faces and should not be treated
as a fixed intrinsic property of the selector.

We then reoptimize the portfolio over all 22 workflow types. The best
repeat-allowed portfolio has $K=4$, where $K$ counts complete workflow
executions, or equivalently the four candidate outputs ultimately presented to
the selector. Two of these four slots are assigned to generated workflow G4
(\texttt{FinalABCDRouterWithLabel}) in
\Cref{tab:abcd-generated-workflows}. A single execution of G4 makes three
independently seeded GPT-4o-mini routing calls on the same ABCD task and returns
the label receiving the most votes. Assigning G4 to two slots therefore means
running this entire three-call workflow twice independently, producing two
candidate outputs. The remaining two slots contain one initial-bank Mistral
verify-and-revise execution and one Qwen evidence-first execution. Thus, the
$K=4$ portfolio presents four candidate outputs to the final selector but uses
eight underlying model calls per task: six from the two G4 executions and one
from each of the two initial-bank workflows.

On the development sample, this portfolio has calibrated selector accuracy of
54.190\% and a workflow-cost-adjusted value of 0.534410. The no-repeat
benchmark also selects $K=4$, but replaces the second copy of G4 with G3
(\texttt{Final\_Single\_Routing\_Agent}), a generated router with critique and
fallback. Its workflow-cost-adjusted value is 0.533396.

\vspace{1mm}
\paragraph{Composition of the generated workflows.}
\Cref{tab:abcd-generated-workflows} summarizes the four workflows generated by
ADAS and added to the evaluated workflow bank. Every model call within these
workflows uses GPT-4o-mini. Thus, the generated workflows differ not in the
underlying model, but in how they organize multiple calls and combine their
outputs, using mechanisms such as parallel sampling, voting, critique,
clarification, and conditional fallback. The repeat-allowed optimal portfolio
assigns two execution slots to \texttt{FinalABCDRouterWithLabel}. Each deployed
task therefore receives two independent runs of this three-call voting
workflow, together with one Mistral and one Qwen workflow from the initial
bank. In the no-repeat benchmark, the second copy is replaced by one execution
of \texttt{Final\_Single\_Routing\_Agent}.

\begin{table}[t]
\centering
\caption{ADAS-generated workflows incorporated into the expanded ABCD bank.
All generated model calls use GPT-4o-mini through the ADAS
\texttt{LLMAgentBase} interface. Each GPT-4o-mini invocation is assigned one
\$0.001 reference-call unit; mean calls and normalized cost are measured during
independent workflow evaluation.}
\label{tab:abcd-generated-workflows}
\small
\begin{tabular}{@{}p{0.15\textwidth}p{0.56\textwidth}rr@{}}
\toprule
Workflow & Executable structure & Mean calls & Cost \\
\midrule

\textbf{G1}
&
Five independent routing calls followed by a plurality vote.
If the highest vote is tied, three additional routing calls are made and
a second plurality vote resolves the tie.
& 8.0 & 0.008 \\[2mm]

\textbf{G2}
&
Five independent routing calls followed by a plurality vote.
If the leading vote is tied, one additional clarification routing call
determines the returned label.
& 6.0 & 0.006 \\[2mm]

\textbf{G3}
&
One initial routing call followed by a critique/validation call.
If the proposed label is invalid, ambiguous, or unsure, an additional
fallback routing call is made.
& 2.0 & 0.002 \\[2mm]

\textbf{G4}
&
Three independent routing calls on the same task, followed by a vote;
the workflow returns the label with the largest vote count.
& 3.0 & 0.003 \\

\bottomrule
\end{tabular}

\vspace{1mm}
{\footnotesize
\emph{Notes.} G1--G4 correspond respectively to the frozen ADAS programs
\texttt{FinalEnhancedRoutingAgent}, \texttt{FinalRoutingAgent},
\texttt{Final\_Single\_Routing\_Agent}, and
\texttt{FinalABCDRouterWithLabel}. Cost is recurring normalized execution cost
per workflow run.}
\end{table}

\vspace{1mm}
\paragraph{Fresh held-out deployment.}
We freeze all reported portfolios before examining the 800 held-out tasks. For
each workflow type appearing in a reported portfolio, we collect five new
stochastic executions per task. We first evaluate each portfolio under the
plug-in Plackett--Luce model, using the estimated workflow success
probabilities and selector strength. We then evaluate actual deployment
performance by using Qwen2.5-7B-Instruct as the final selector. The selector
does not observe workflow identity or correctness. To reduce sensitivity to
candidate position, we present each candidate set under up to six cyclic
rotations of the candidate order and use the candidate receiving the most
selection votes across rotations.

\vspace{1mm}
\subsection{Results on ABCD}
\Cref{tab:abcd-heldout-deployment,fig:abcd-heldout-accuracy}
report deployment performance on 800 fresh held-out ABCD tasks.
Actual selector accuracy increases from 43.625\% for the best
initial singleton to 46.750\% for the optimized initial-bank
portfolio and to 50.250\% for the portfolio obtained after
ellipsoid-guided workflow generation. Thus,
optimizing the initial workflow bank contributes 3.125 percentage points, and
expanding the bank through workflow generation contributes an additional
3.500 percentage points. Relative to the initial singleton, the complete
procedure gains 6.625 percentage points, or 15.2\%. The same ordering appears
under the calibrated Plackett--Luce model, where predicted selector accuracy
rises from 43.550\% to 47.269\% and then to 50.648\%. Notably, the final
portfolio uses only four execution slots, compared with six for the optimized
initial-bank portfolio. The generated workflows therefore allow the system to
use fewer workflow executions while achieving higher held-out accuracy. The
no-repeat benchmark reaches 50.125\%, only 0.125 percentage points below the
primary repeat-allowed solution.

The improvement does not come simply from increasing the probability that at
least one candidate is correct. That probability, which we refer to as oracle
coverage, decreases modestly from 59.517\% for the optimized initial-bank
portfolio to 58.264\% for the expanded-bank portfolio. At the same time, the
accuracy obtained by selecting uniformly at random from the realized candidate
outputs rises from 40.533\% to 45.681\%. Thus, the generated workflows produce
candidate sets in which correct outputs are more prevalent, even though the
chance of having at least one correct output is slightly lower. After
subtracting both workflow execution cost and the realized cost of the selector,
actual total-system net value increases from 0.435207 for the singleton to
0.455829 for the optimized initial-bank portfolio and 0.490671 for the final
ADAS-expanded portfolio.

\begin{table}[t]
\centering
\caption{Fresh held-out ABCD deployment on 800 tasks. PL accuracy is the
plug-in Plackett--Luce prediction; actual accuracy uses the blind deterministic
Qwen selector. Brackets report 95\% Wilson intervals. Workflow and selector
costs are per task, and actual total-system net subtracts both.}
\label{tab:abcd-heldout-deployment}
\scriptsize
\resizebox{\textwidth}{!}{%
\begin{tabular}{lrrrrrrrr}
\toprule
Method & $K$ & PL accuracy & Actual accuracy (95\% CI) & Random accuracy
& Oracle coverage & Workflow cost & Selector cost & Actual total-system net \\
\midrule
Best initial singleton
& 1 & 0.4355 & $0.4363\ [0.4023,0.4708]$ & 0.4355 & 0.4355
& 0.001043 & 0.000000 & 0.435207 \\
Optimized initial bank
& 6 & 0.4727 & $0.4675\ [0.4332,0.5021]$ & 0.4053 & 0.5952
& 0.005157 & 0.006514 & 0.455829 \\
ADAS-expanded bank
& 4 & 0.5065 & $0.5025\ [0.4679,0.5371]$ & 0.4568 & 0.5826
& 0.007486 & 0.004343 & 0.490671 \\
ADAS-expanded, no-repeat
& 4 & 0.5069 & $0.5013\ [0.4667,0.5358]$ & 0.4572 & 0.5829
& 0.006486 & 0.004343 & 0.490421 \\
\bottomrule
\end{tabular}%
}
\end{table}

\begin{figure}[h]
    \centering
    \includegraphics[width=0.86\textwidth]{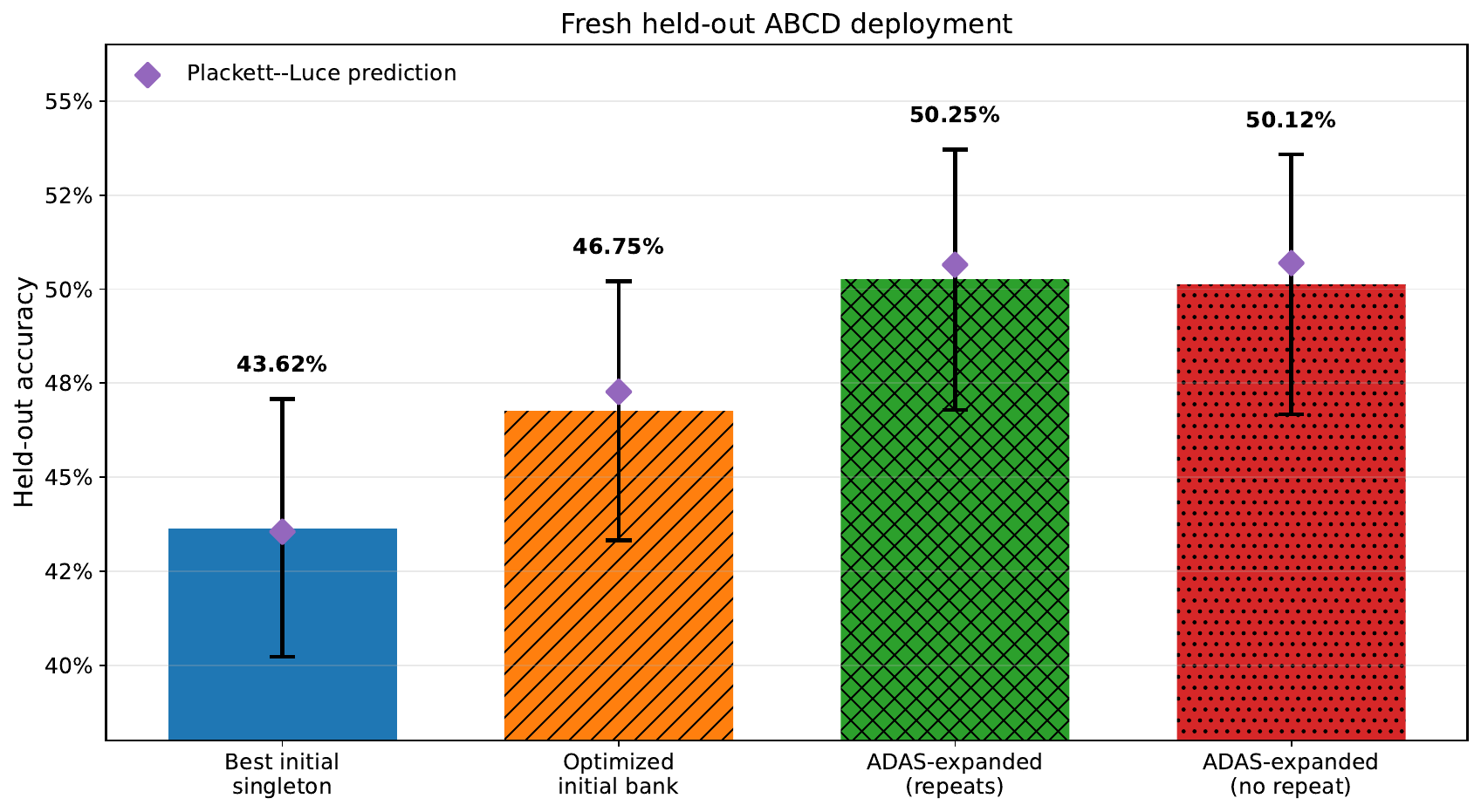}
    \caption{Fresh held-out ABCD accuracy. Bars report actual deterministic
    selector accuracy, error bars show 95\% Wilson intervals over 800 tasks,
    and diamonds report the corresponding plug-in Plackett--Luce predictions.}
    \label{fig:abcd-heldout-accuracy}
\end{figure}

\subsection{Results on SGD and HotpotQA}
\label{sec:numerical-other-domains}

The same end-to-end pipeline produces different sources of value across the
other two domains. On SGD, the initial workflow bank is already strong and
complementary: optimizing it raises actual held-out selector accuracy from
85.250\% for the best singleton to 92.750\%. None of the 20 evaluated ADAS
candidates is incorporated at the queried dual prices, so the final portfolio
coincides with the optimized initial-bank portfolio. Thus, SGD demonstrates
the value of portfolio optimization without requiring workflow generation.

On HotpotQA, optimizing the initial bank raises actual held-out accuracy only
from 30.250\% to 31.125\%, but dual-guided generation incorporates a
three-call evidence-extraction, verification, and multihop-solving workflow.
Two independent executions of this workflow attain 55.250\% actual accuracy;
one execution attains 54.875\% and has slightly higher total-system net value
after selector cost is included. Together with ABCD, these results show that
portfolio optimization, workflow generation, and multiplicity are distinct
sources of value whose importance varies across applications. Complete
experimental designs, calibration results, portfolio compositions, and
held-out evaluations are reported in
\Cref{app:sgd-end-to-end,app:hotpotqa-end-to-end}.

\section{Conclusion}
\label{sec:conclusion}

Agentic AI systems often approach the same task through multiple workflows and
then rely on a selector to choose the final answer. This paper shows that
workflow variety creates both opportunity and risk. While additional executions may
solve cases that the current portfolio misses, they also consume compute
and may introduce plausible wrong answers that make final selection more
difficult. We develop a framework for managing this tradeoff. Our selector strength bound
places a sharp limit on the value of workflow variety and identifies conditions
under which a single workflow or no deployment is optimal. We then develop
optimization methods for choosing the run size and allocating execution slots
across workflow types. For large implicit workflow classes, dual prices guide a
workflow generation oracle toward residual tasks on which an additional correct
output would be most valuable, while the ellipsoid method provides performance
guarantees without requiring the workflow class to be enumerated.

The numerical results illustrate both the value and the limits of workflow portfolios. Relative to the best initial singleton, the final portfolio raises
actual held-out selector accuracy by 6.625 percentage points on ABCD, 7.500 points on SGD, and 25.000 points on HotpotQA. Workflow generation adds value on
ABCD and HotpotQA but not on SGD, where optimizing the already strong initial bank is sufficient. Multiplicity likewise has heterogeneous value: it binds in
ABCD and HotpotQA but not in SGD, and its incremental accuracy benefit need not justify its additional compute and selection cost. 
The central message is that more inference time variety is not automatically better. Effective deployment requires firms to coordinate workflow
generation, run size, execution allocation, and compute expenditure with the
recovery capabilities of the selector.


\bibliographystyle{informs2014}
\bibliography{refs}

\newpage
\setcounter{page}{1}
\section*{\centering Appendices for ``Managing Agentic AI Workflow Portfolios under Imperfect Selection and Compute Cost.''}\vspace{0.3cm}\label{sec:Appendix}
\begin{APPENDICES}
\normalsize

\section{A Graph-Based Formalization of Workflow Classes}
\label{app:workflow-grammar}

This appendix gives one formal representation of the implicit workflow class
$\G$. The main analysis requires only that each workflow $g$ can be executed on
a task, produces a candidate answer and trace, and has an evaluated correctness
column and recurring cost. A graph grammar is useful for making the feasible
class explicit, but the portfolio optimization results do not depend on this
particular representation.

Let $\mathcal U=\{u_1,\ldots,u_Q\}$ denote a library of primitive executable
modules. A module may be a prompt-based agent, retriever, solver, verifier,
critic, tool call, aggregation rule, learned classifier, neural-network block,
or answer formatter. A workflow $g$ is a finite directed computation graph
\[
    g=(V_g,E_g,\ell_g,s_g,t_g),
\]
where $V_g$ is a finite node set, $E_g$ represents information flow,
$\ell_g:V_g\to\mathcal U$ assigns a module to each node, and $s_g,t_g$ are input
and output nodes. We impose $|V_g|\le L_{\max}$ and any application-specific
restrictions on retrieval, tools, model calls, memory, context length, or
termination. The feasible workflow class $\G$ consists of all graphs satisfying
these restrictions.

\section{Other selector technologies}

The endogenous framework allows a different recovery curve $\psi_k$ for each candidate-set size $k$. The selector-strength bound in \Cref{sec:selector-curve} applies whenever these curves admit a common fraction-based envelope with finite odds-lift. The fixed-size IP remains valid for any nondecreasing recovery curve, while the randomized-rounding certificate additionally requires diminishing increments. A selector that reaches recovery probability one at an interior correct fraction has infinite odds-lift, so the uniform selector-strength cap becomes vacuous even though the exact size-conditioned IP remains applicable.

\paragraph{Conservative clean recovery.}
Suppose each correct candidate is rejected with probability $\beta$, each incorrect candidate is falsely certified with probability $\alpha$, and verifier outputs are conditionally independent. The event that at least one correct candidate is certified and no incorrect candidate is certified has probability
\begin{equation}
    \psi_{\alpha,\beta,k}^{\mathrm{clean}}(r)
    =
    \1\{r\ge1\}(1-\beta^r)(1-\alpha)^{k-r}.
\end{equation}

\paragraph{Uniform selection among certified candidates.}
If the selector chooses uniformly among certified candidates, with
$U\sim\mathrm{Binomial}(r,1-\beta)$ and
$W\sim\mathrm{Binomial}(k-r,\alpha)$,
\begin{equation}
    \psi_{\alpha,\beta,k}^{\mathrm{unif}}(r)
    =
    \sum_{u=1}^{r}\sum_{w=0}^{k-r}
    \frac{u}{u+w}
    \binom{r}{u}(1-\beta)^u\beta^{r-u}
    \binom{k-r}{w}\alpha^w(1-\alpha)^{k-r-w}.
\end{equation}

\paragraph{Voting.}
Under a conservative model in which accepted wrong workflows coordinate on the same wrong answer, let
$C_r\sim\mathrm{Binomial}(r,1-\beta)$ and
$W_{k-r}\sim\mathrm{Binomial}(k-r,\alpha)$. A sufficient recovery event is $C_r>W_{k-r}$, giving
\begin{equation}
    \psi_{\alpha,\beta,k}^{\mathrm{vote}}(r)
    =
    \Prob\{C_r>W_{k-r}\}.
\end{equation}
Voting is typically threshold-like and need not be concave. The size-conditioned IP remains exact, but the concave-coverage rounding certificate does not apply unless the estimated curve satisfies diminishing increments.

\section{A Numerical Illustration of Workflow Multiplicity}
\label{app:workflow-multiplicity-example}

This example illustrates why it can be optimal to assign multiple execution
slots to the same workflow type under an imperfect selector. Suppose there are
two deterministic workflow types, $A$ and $B$, and three task types with the
correctness patterns and population shares shown in
\Cref{tab:multiplicity-task-types}.

\begin{table}[h]
\centering
\caption{Task types and workflow correctness.}
\label{tab:multiplicity-task-types}
\begin{tabular}{cccc}
\toprule
Task type & Workflow $A$ & Workflow $B$ & Population share \\
\midrule
Type 1 & Correct   & Incorrect & $0.40$ \\
Type 2 & Incorrect & Correct   & $0.20$ \\
Type 3 & Incorrect & Incorrect & $0.40$ \\
\bottomrule
\end{tabular}
\end{table}
For compactness, write execution-count vectors in the coordinate order
$(A,B)$, and define
\[
    \bm m^A=(1,0),
    \qquad
    \bm m^{AB}=(1,1),
    \qquad
    \bm m^{AAB}=(2,1).
\]

Workflow $A$ has standalone accuracy $0.40$, whereas workflow $B$ has
standalone accuracy $0.20$. Workflow $B$ is nevertheless complementary because
it solves Type~2 tasks, which workflow $A$ misses.

Suppose the selector follows the Plackett--Luce recovery curve with
$\lambda=2$:
\begin{equation}
    \psi^{\mathrm{PL}}_{k,2}(r)
    =
    \frac{2r}{2r+k-r}.
    \label{eq:multiplicity-pl-recovery}
\end{equation}
For the portfolio $\bm m^{AB}$, exactly one of the two candidates is correct on
Types~1 and~2, so
\begin{equation}
    J_2(\bm m^{AB})
    =
    0.40\left(\frac{2}{3}\right)
    +
    0.20\left(\frac{2}{3}\right)
    =
    0.40.
    \label{eq:multiplicity-ab}
\end{equation}

Now consider the portfolio $\bm m^{AAB}$. On Type~1 tasks, two of the three
candidates are correct, giving recovery probability $4/5$. On Type~2 tasks,
only workflow $B$ is correct, giving recovery probability $1/2$. Therefore,
\begin{equation}
    J_2(\bm m^{AAB})
    =
    0.40\left(\frac{4}{5}\right)
    +
    0.20\left(\frac{1}{2}\right)
    =
    0.42.
    \label{eq:multiplicity-aab}
\end{equation}
Thus,
\begin{equation}
    J_2(\bm m^{AAB})
    =
    0.42
    >
    J_2(\bm m^{AB})
    =
    J_2(\bm m^A)
    =
    0.40.
    \label{eq:multiplicity-comparison}
\end{equation}

The additional execution of $A$ creates no new task coverage. Instead, it
increases the representation of the stronger workflow on the more common
Type~1 tasks, raising recovery from $2/3$ to $4/5$, while retaining workflow
$B$ to cover Type~2 tasks. The gain on Type~1 tasks exceeds the corresponding
loss in selector recovery on Type~2 tasks.

The result also survives a small execution cost. If every execution has
accuracy-equivalent cost $\kappa=0.005$, then
\begin{align}
    \Pi_2(\bm m^A)     &=0.400-0.005=0.395, \nonumber\\
    \Pi_2(\bm m^{AB})   &=0.400-0.010=0.390, \nonumber\\
    \Pi_2(\bm m^{AAB}) &=0.420-0.015=0.405.
    \label{eq:multiplicity-cost}
\end{align}
Hence, $\bm m^{AAB}$ remains optimal among these portfolios after accounting for
recurring execution cost. By contrast, repeating $A$ without including another
workflow type would leave its selector-aware accuracy unchanged at $0.40$ and
would only add cost. Multiplicity is valuable here because it adjusts the
relative representation of complementary workflow types, not because pure
repetition expands coverage.

\section{Proofs}
\label{app:proofs}

\subsection{Proof of Lemma~\ref{lem:odds-lift-envelope}}
For $p\in(0,1)$, the condition $\Lambda(H)\le\Lambda$ gives
\begin{equation}
    \frac{H(p)}{1-H(p)}
    \le
    \Lambda\frac{p}{1-p}.
\end{equation}
Finite odds-lift implies $H(p)<1$ at every interior $p$, so all denominators are positive. Multiplying through and collecting terms yields
\begin{equation}
    H(p)\{1-p+\Lambda p\}
    \le
    \Lambda p.
\end{equation}
Because $1-p+\Lambda p=1+(\Lambda-1)p>0$,
\begin{equation}
    H(p)
    \le
    \frac{\Lambda p}{1+(\Lambda-1)p}
    =h_\Lambda(p).
\end{equation}
At $p=0$, both sides are zero under $H(0)=0$; at $p=1$, the right-hand side is one and the inequality follows from $H(1)\le1$. \Halmos

\subsection{Proof of Lemma~\ref{lem:odds-envelope-concavity}}
For $p\in[0,1]$,
\begin{equation}
    h_\Lambda'(p)
    =
    \frac{\Lambda}{\{1+(\Lambda-1)p\}^2}\ge0,
    \qquad
    h_\Lambda''(p)
    =
    -\frac{2\Lambda(\Lambda-1)}{\{1+(\Lambda-1)p\}^3}\le0.
\end{equation}
Thus $h_\Lambda$ is increasing and concave. \Halmos

\subsection{Proof of Lemma~\ref{prop:pl-concavity}}
For continuous $r$,
\begin{equation}
    \frac{d}{dr}\frac{\lambda r}{k+(\lambda-1)r}
    =
    \frac{\lambda k}{\{k+(\lambda-1)r\}^2}\ge0,
\end{equation}
and
\begin{equation}
    \frac{d^2}{dr^2}\frac{\lambda r}{k+(\lambda-1)r}
    =
    -\frac{2\lambda k(\lambda-1)}{\{k+(\lambda-1)r\}^3}\le0.
\end{equation}
The discrete increment follows by subtracting adjacent values and simplifying, which yields \eqref{eq:pl-discrete-increments}; its denominator increases in $\ell$. Finally,
\begin{equation}
    h_\lambda'(p)
    =
    \frac{\lambda}{\{1+(\lambda-1)p\}^2}>0,
    \qquad
    h_\lambda''(p)
    =
    -\frac{2\lambda(\lambda-1)}{\{1+(\lambda-1)p\}^3}\le0.
\end{equation}
\Halmos

\subsection{Proof of Proposition~\ref{prop:endogenous-structure}}
Consider one task and zero compute costs. Let $c_1,c_2$ be workflows that are correct on the task and let $w_1,w_2$ be wrong. Nonmonotonicity follows from
\begin{equation}
    J_\lambda(\{c_1\})=1,
    \qquad
    J_\lambda(\{c_1,w_1\})=h_\lambda(1/2)<1.
\end{equation}
For submodularity, take $A=\{c_1\}$, $B=\{c_1,w_1\}$, and $e=c_2$. Then
\begin{equation}
    \Delta(e\mid A)=0,
    \qquad
    \Delta(e\mid B)
    =
    h_\lambda(2/3)-h_\lambda(1/2)>0,
\end{equation}
which violates diminishing returns. For supermodularity, take $A=\{w_1\}$, $B=\{w_1,w_2\}$, and $e=c_1$. Then
\begin{equation}
    \Delta(e\mid A)=h_\lambda(1/2)
    >
    h_\lambda(1/3)=\Delta(e\mid B),
\end{equation}
which violates increasing returns. Adding or subtracting a modular function does not change the submodularity or supermodularity inequalities. \Halmos

\subsection{Proof of Theorem~\ref{thm:geometric-cardinality-approx}}
Fix an execution-count vector $\bm m\in\mathbb Z_+^{\M}$ satisfying
$k(\bm m)=k$, label its $k$ execution slots, and sample $b$ labeled slots
uniformly without replacement. Let $\bm M^{(b)}$ be the resulting random
execution-count vector and let $\alpha=b/k$. For task $i$, write
$r_i=r_i(\bm m)$ and let
\[
    R_i=r_i\bigl(\bm M^{(b)}\bigr).
\]
Then
\[
    R_i\sim\mathrm{Hypergeom}(k,r_i,b),
    \qquad
    \E[R_i]=\alpha r_i.
\]
If $r_i=0$, then both the size-$k$ task value and the expected size-$b$ task value are zero. Suppose $r_i>0$ and define $p_i=r_i/k$ and $X_i=R_i/b$. Since $h$ is increasing, concave, and satisfies $h(0)=0$,
\begin{equation}
    h(x)\ge h(p_i)\min\{x/p_i,1\},
    \qquad x\in[0,1].
    \label{eq:concave-subportfolio-support}
\end{equation}
Indeed, for $x\le p_i$ the inequality follows from concavity between $0$ and $p_i$, and for $x\ge p_i$ it follows from monotonicity. Now
\begin{equation}
    \frac{X_i}{p_i}
    =
    \frac{R_i}{\alpha r_i}.
\end{equation}
Let $Z_i=R_i/r_i\in[0,1]$. Then $\E[Z_i]=\alpha$ and, for every $z\in[0,1]$, $\min\{z,\alpha\}\ge \alpha z$. Hence
\begin{align}
    \E\left[\min\left\{\frac{X_i}{p_i},1\right\}\right]
    &=
    \frac1\alpha \E[\min\{Z_i,\alpha\}]
    \ge
    \frac1\alpha\alpha\E[Z_i]
    =
    \alpha.
\end{align}
Combining this with \eqref{eq:concave-subportfolio-support} gives
\begin{equation}
    \E[h(R_i/b)]
    \ge
    \alpha h(r_i/k).
\end{equation}
Averaging across tasks yields
\begin{equation}
    \E\!\left[J_\psi\bigl(\bm M^{(b)}\bigr)\right]
\ge
\alpha J_\psi(\bm m),
\end{equation}
Because costs are additive,
\begin{equation}
    \E\!\left[C\bigl(\bm M^{(b)}\bigr)\right]
=
\alpha C(\bm m),
\end{equation}
Therefore
\begin{equation}
    \E\!\left[
    \Pi_{\psi,\gamma}\bigl(\bm M^{(b)}\bigr)
\right]
\ge
\alpha\Pi_{\psi,\gamma}(\bm m).
\end{equation}
Some realization of $\bm M^{(b)}$ attains at least this expected value. If $\bm m$ is an optimal size-$k$ execution-count vector with positive value, then
\begin{equation}
    \OPT_b(\M)\ge \alpha \OPT_k(\M),
\end{equation}
which proves \eqref{eq:cross-cardinality-stability}; if $\OPT_k(\M)\le0$, the statement with positive parts is immediate. For the grid result, choose for each $k$ an anchor $b\in\mathcal K$ with $b\le k\le \varrho b$. Then
\begin{equation}
    \OPT_b^+(\M)
    \ge
    \frac{b}{k}\OPT_k^+(\M)
    \ge
    \frac1\varrho \OPT_k^+(\M).
\end{equation}
Maximizing over $k$ proves \eqref{eq:geometric-grid-guarantee}. The dyadic and ratio-$1/(1-\eta_{\mathrm{grid}})$ claims follow from the definition of the grid coverage ratio. \Halmos

\subsection{Proof of Theorem~\ref{thm:diversification-bound}}
The lower bound follows from the outside option and singleton feasibility.
With one candidate, the selector must return that candidate, so the net value
of the singleton execution-count vector $\bm e_g$ is
$\bar a_g-\gamma c_g$.

Fix a nonzero execution-count vector $\bm m\in\mathbb Z_+^{\M}$ satisfying
$k(\bm m)=k$, and define
\[
    p_i(\bm m)
    =
    \frac{r_i(\bm m)}{k}.
\]
By \Cref{ass:recovery-envelope,lem:odds-lift-envelope},
\begin{align}
    J_\psi(\bm m)
    &=
    \frac1n\sum_{i=1}^n\psi_k\bigl(r_i(\bm m)\bigr)\\
    &\le
    \frac1n\sum_{i=1}^n H\bigl(p_i(\bm m)\bigr)\\
    &\le
    \frac1n\sum_{i=1}^n h_\Lambda\bigl(p_i(\bm m)\bigr).
\end{align}
By \Cref{lem:odds-envelope-concavity} and Jensen's inequality,
\begin{align}
    \frac1n\sum_{i=1}^n h_\Lambda\bigl(p_i(\bm m)\bigr)
    &\le
    h_\Lambda\left(
        \frac1n\sum_{i=1}^n p_i(\bm m)
    \right)\\
    &=
    h_\Lambda\left(
        \frac1k\sum_{g\in\M}\bar a_gm_g
    \right)\\
    &\le
    h_\Lambda(A_k).
\end{align}
Also $C(\bm m)\ge C_k$. Hence
\[
    \Pi_{\psi,\gamma}(\bm m)
    \le
    h_\Lambda(A_k)-\gamma C_k.
\]
Maximizing over $k$ and including the outside option proves \eqref{eq:diversification-upper-bound}. When $\gamma=0$, $A_k\le a^\star$ for all $k$, giving \eqref{eq:no-cost-diversification-bound}.

For tightness over the selector class, take recovery $H=h_\Lambda$. Suppose $a^\star=r/k$ and $k\le K_{\max}$. Construct $k$ tasks and $k$ workflows using a cyclic incidence matrix in which each workflow solves exactly $r$ tasks and each task is solved by exactly $r$ workflows. Let $\bm m^\star$ assign one execution slot to each of these $k$ workflows.
Then $p_i(\bm m^\star)=a^\star$ for every task, and therefore
$J_\psi(\bm m^\star)=h_\Lambda(a^\star)$. Rational values can be represented
exactly after replication, and arbitrary values can be approximated. \Halmos

\subsection{Proof of Corollary~\ref{cor:diversification-implications}}
The first inequality follows from \eqref{eq:no-cost-diversification-bound}. Let
\begin{equation}
    G_\Lambda(a)
    =
    h_\Lambda(a)-a
    =
    \frac{(\Lambda-1)a(1-a)}{1+(\Lambda-1)a}.
\end{equation}
For $\Lambda>1$, differentiation gives the unique maximizer $a=1/(1+\sqrt\Lambda)$ and maximum $(\sqrt\Lambda-1)/(\sqrt\Lambda+1)$. At $\Lambda=1$, both sides are zero.

If all workflows cost $c$, any execution-count vector $\bm m$ satisfying
$k(\bm m)=k\ge2$ obeys
\[
    \Pi_{\psi,\gamma}(\bm m)
    \le
    h_\Lambda(a^\star)-k\gamma c.
\]
Under \eqref{eq:uniform-cost-singleton-condition},
\begin{equation}
    h_\Lambda(a^\star)-k\gamma c
    \le
    a^\star-\gamma c,
\end{equation}
so no multi-workflow portfolio beats the best singleton. The outside option dominates only when the singleton net value is negative. \Halmos

\subsection{Proof of Proposition~\ref{prop:ip-exact}}

Fix an execution-count vector
$\bm m\in\mathbb Z_+^{\M}$ satisfying $k(\bm m)=k$, and set
\[
    z_g=m_g,
    \qquad g\in\M.
\]
Then
\[
    \sum_{g\in\M}a_{ig}z_g
    =
    r_i(\bm m),
    \qquad i=1,\ldots,n.
\]
For this fixed $z$, the prefix constraints imply that the active tier
variables for task $i$ form a prefix, while the capacity constraint implies
that the length of this prefix is at most $r_i(\bm m)$. Since
$d_{\ell,k}\ge0$, there exists an optimal choice
\[
    y_{i\ell}
    =
    \1\{\ell\le r_i(\bm m)\},
    \qquad \ell=1,\ldots,k.
\]
Its task-$i$ contribution is
\[
    \sum_{\ell=1}^k d_{\ell,k}y_{i\ell}
    =
    \sum_{\ell=1}^{r_i(\bm m)}d_{\ell,k}
    =
    \psi_k\bigl(r_i(\bm m)\bigr),
\]
where the last equality uses $\psi_k(0)=0$. Moreover, the cost term is
\[
    -\gamma\sum_{g\in\M}c_gz_g
    =
    -\gamma C(\bm m).
\]
Thus every feasible size-$k$ execution-count vector induces an integer-program
solution with the same objective value, and hence
\[
    I_k(\M)\ge \OPT_k(\M).
\]

Conversely, let $(y,z)$ be any integral feasible solution of
\eqref{eq:cost-aware-ip}, and define the execution-count vector
$\bm m\in\mathbb Z_+^{\M}$ by $m_g=z_g$. Then
$k(\bm m)=k$. For each task $i$, let
\[
    t_i=\sum_{\ell=1}^k y_{i\ell}.
\]
Because $y$ is binary and satisfies the prefix constraints, its active entries
form a prefix of length $t_i$. Feasibility gives
\[
    t_i
    \le
    \sum_{g\in\M}a_{ig}z_g
    =
    r_i(\bm m).
\]
Therefore, using the nonnegativity of the increments,
\[
    \sum_{\ell=1}^k d_{\ell,k}y_{i\ell}
    =
    \sum_{\ell=1}^{t_i}d_{\ell,k}
    =
    \psi_k(t_i)
    \le
    \psi_k\bigl(r_i(\bm m)\bigr).
\]
The integer-program objective is consequently at most
$\Pi_{\psi,\gamma}(\bm m)$, and hence at most $\OPT_k(\M)$. Thus
\[
    I_k(\M)\le \OPT_k(\M).
\]
Combining the two inequalities gives
$I_k(\M)=\OPT_k(\M)$. Finally, combining this equality with the size
decomposition \eqref{eq:size-decomposition} proves
\eqref{eq:exact-global-enumeration}. \Halmos

\subsection{Proof of Theorem~\ref{thm:cost-rounding}}
Write $d_\ell=d_{\ell,k}$. For a feasible repeated-execution LP vector $z$,
define
\[
    x_i=\sum_g a_{ig}z_g,
    \qquad
    q_g=z_g/k.
\]
Draw $G_1,\ldots,G_k$ independently from $q$, and let
$R_i=\sum_{j=1}^k a_{iG_j}$. Then
\[
    R_i\sim\mathrm{Binomial}(k,x_i/k),
    \qquad
    \E[R_i]=x_i.
\]
For every integer $t\ge1$ and $r\ge0$,
\[
    \min\{r,t\}
    \ge
    t\left\{1-(1-1/t)^r\right\}.
\]
Therefore
\begin{align*}
    \E[\min\{R_i,t\}]
    &\ge
    t\left[1-\E\{(1-1/t)^{R_i}\}\right]\\
    &=
    t\left[1-\left(1-\frac{x_i}{kt}\right)^k\right]\\
    &\ge
    t\left(1-e^{-x_i/t}\right)\\
    &\ge
    \left(1-\frac1e\right)\min\{x_i,t\}.
\end{align*}
The last inequality follows from
$1-e^{-u}\ge(1-1/e)\min\{u,1\}$ for $u\ge0$. Using the decomposition
\[
    \widetilde\psi_k(r)
    =
    d_{k,k}r+
    \sum_{t=1}^{k-1}(d_{t,k}-d_{t+1,k})\min\{r,t\},
\]
and averaging across tasks gives
\[
    \E\!\left[
    J_\psi\bigl(\bm M_k^{\RR}(z)\bigr)
\right]
\ge
\left(1-\frac1e\right)Q_k(z)
+
\frac{d_{k,k}}eR_k(z).
\]
Moreover, if $N_g$ is the sampled multiplicity of type $g$, then
\[
    \E\!\left[
    C\bigl(\bm M_k^{\RR}(z)\bigr)
\right]
=
\sum_g c_g
\E\!\left[M_{k,g}^{\RR}(z)\right]
=
\sum_g c_gz_g
=
C_k(z).
\]
Subtracting cost proves \eqref{eq:net-rounding-lower-bound} and
\eqref{eq:net-rounding-fractional-gap}. If $z=z^{k\star}$ is LP-optimal, then
$\OPT_k(\M)\le L_k(\M)=F_k(z^{k\star})$, which gives the first inequality in
\eqref{eq:net-rounding-additive-gap}. Finally, concavity implies
$Q_k(z)\le d_{1,k}R_k(z)$, so
\[
    d_{k,k}R_k(z)
    =
    (1-c_{\psi,k})d_{1,k}R_k(z)
    \ge
    (1-c_{\psi,k})Q_k(z).
\]
This proves the curvature bound. \Halmos

\subsection{Proof of Corollary~\ref{cor:pl-multiplicative-rounding}}

By \Cref{thm:cost-rounding},
\[
    \E\!\left[
        \Pi_{\lambda,\gamma}(\bm M_k^{\RR})
    \right]
    \ge
    L_k(\M)
    -
    \frac{c_{k,\lambda}^{\mathrm{PL}}}{e}Q_k^\star .
\]
Since the LP relaxation has $0\le y_{i\ell}\le1$ and
$\sum_{\ell=1}^k d_{\ell,k}=1$, we have $0\le Q_k^\star\le1$. Hence
\[
    \E\!\left[
        \Pi_{\lambda,\gamma}(\bm M_k^{\RR})
    \right]
    \ge
    L_k(\M)
    -
    \frac{c_{k,\lambda}^{\mathrm{PL}}}{e},
\]
This establishes the fixed-size additive net-value bound used below.

For the accuracy-only bound, \Cref{thm:cost-rounding} gives
\[
    \E\!\left[
        J_\lambda(\bm M_k^{\RR})
    \right]
    \ge
    \left(1-\frac1e\right)Q_k^\star
    +
    \frac{d_{k,k}}e R_k^\star.
\]
By the definition of curvature,
\[
    d_{k,k}
    =
    \left(1-c_{k,\lambda}^{\mathrm{PL}}\right)d_{1,k}.
\]
Concavity implies $Q_k^\star\le d_{1,k}R_k^\star$, so
\[
    d_{k,k}R_k^\star
    \ge
    \left(1-c_{k,\lambda}^{\mathrm{PL}}\right)Q_k^\star.
\]
Substituting this inequality yields
\[
    \E\!\left[
        J_\lambda(\bm M_k^{\RR})
    \right]
    \ge
    \left(1-\frac1e\right)Q_k^\star
    +
    \frac{
        1-c_{k,\lambda}^{\mathrm{PL}}
    }e
    Q_k^\star
    =
    \left(
        1-\frac{c_{k,\lambda}^{\mathrm{PL}}}{e}
    \right)
    Q_k^\star,
\]
This establishes an auxiliary accuracy-only bound.

It remains to prove the grid-level multiplicative bound. Let
\[
    U=
    U_{\mathrm{LP}}(\mathcal K;\M),
    \qquad
    a=
    \frac1e\max_{k\in\mathcal K}c_{k,\lambda}^{\mathrm{PL}}.
\]
If $U=0$, the result is immediate. Otherwise, choose
$k^\star\in\argmax_{k\in\mathcal K}L_k(\M)$, so that
$L_{k^\star}(\M)=U$. By the fixed-size additive net-value bound derived above,
\[
    \E\!\left[
        \Pi_{\lambda,\gamma}(S_{k^\star}^{\RR})
    \right]
    \ge
    U-a.
\]
Because $\widehat{\bm M}$ is chosen after comparing the rounded portfolios with the
baseline $\bm m^0$ and the outside option, Jensen's inequality for the convex
maximum function gives
\[
    \E\!\left[
        \Pi_{\lambda,\gamma}(\widehat{\bm M})
    \right]
    \ge
    \max\{\underline V,U-a\}.
\]
For every $U\ge0$, $a\ge0$, and $\underline V>0$,
\[
    \max\{\underline V,U-a\}
    \ge
    \frac{\underline V}{\underline V+a}U.
\]
Indeed, if $U\le \underline V+a$, then
\[
    \max\{\underline V,U-a\}
    \ge
    \underline V
    \ge
    \frac{\underline V}{\underline V+a}U.
\]
If $U>\underline V+a$, then
\[
    \max\{\underline V,U-a\}
    \ge
    U-a
    \ge
    \frac{\underline V}{\underline V+a}U.
\]
Combining the last two displays proves
\eqref{eq:grid-multiplicative-rounding}. Finally, exact integer-program solves
over $\M$ dominate the rounded feasible portfolios, so the same lower bound
also applies when the final restricted problems are solved exactly. \Halmos

\subsection{Projected dual, weak separation, and proof details}
\label{app:ellipsoid-pricing-details}

For $x\in[0,k]$, let
\[
    \varphi_k(x)=\frac1n\widetilde\psi_k(x),
\]
and define
\begin{equation}
    \chi_k(u)
    =
    \max_{r=0,\ldots,k}
    \left\{
        \frac{\psi_k(r)}n-ur
    \right\},
    \qquad
    0\le u\le\frac{d_{1,k}}n.
    \label{eq:chi-k-definition}
\end{equation}
Because $\varphi_k$ is concave and piecewise linear with slopes in
$[0,d_{1,k}/n]$,
\begin{equation}
    \varphi_k(x)
    =
    \min_{0\le u\le d_{1,k}/n}
    \{ux+\chi_k(u)\}.
    \label{eq:concave-conjugate-representation}
\end{equation}

\begin{lemma}[Projected dual for repeated executions]
\label{lem:projected-dual-prices}
For every nonempty finite workflow-type class $\G$ and fixed size $k$,
\begin{equation}
    L_k(\G)
    =
    \min_{0\le\mu_i\le d_{1,k}/n}
    \left\{
        \sum_{i=1}^n\chi_k(\mu_i)
        +
        k\max_{g\in\G}
        \left(
            \sum_{i=1}^n\mu_i a_{ig}-\gamma c_g
        \right)
    \right\}.
    \label{eq:projected-repeated-dual}
\end{equation}
Equivalently, introducing $\theta$ and epigraph variables $\xi_i$ yields
\eqref{eq:ellipsoid-dual}.
\end{lemma}

\begin{proof}
After optimizing the tier variables, the repeated-execution LP is
\[
    \max_{z\ge0:\,\sum_gz_g=k}
    \left\{
        \sum_{i=1}^n
        \varphi_k\left(\sum_ga_{ig}z_g\right)
        -\gamma\sum_gc_gz_g
    \right\}.
\]
Substitute \eqref{eq:concave-conjugate-representation} and apply finite-
dimensional minimax to interchange the minimization over $\mu$ and the
maximization over the scaled simplex. For fixed $\mu$, linear optimization
over that simplex gives
\[
    \max_{z\ge0:\,\sum_gz_g=k}
    \sum_gz_g
    \left(\sum_i\mu_i a_{ig}-\gamma c_g\right)
    =
    k\max_{g\in\G}
    \left(\sum_i\mu_i a_{ig}-\gamma c_g\right).
\]
This proves \eqref{eq:projected-repeated-dual}. The epigraph representation of
$\chi_k$ is
\[
    \xi_i+r\mu_i\ge\frac{\psi_k(r)}n,
    \qquad r=0,\ldots,k,
\]
which gives \eqref{eq:ellipsoid-dual}. \Halmos
\end{proof}

\begin{lemma}[Batched pricing gives high-confidence weak separation]
\label{lem:batched-pricing-separation}
Fix a candidate dual point $(\mu,\xi,\theta)$ and a tolerance $\delta>0$.
Suppose each primitive pricing call made at task prices $\mu$ returns a
$\delta$-accurate workflow, meaning a workflow $\widetilde g$ satisfying
\[
    s_{\widetilde g}(\mu)
    \ge
    \max_{g\in\G}s_g(\mu)-\delta,
    \qquad
    s_g(\mu)=\sum_i\mu_i a_{ig}-\gamma c_g,
\]
with conditional probability at least $p_{\mathrm{orc}}>0$. Run $m$ primitive
calls at the same prices and keep the highest-scoring returned workflow. Then,
conditional on the history before the batch, the batch is $\delta$-accurate
with probability at least
\[
    1-e^{-p_{\mathrm{orc}}m}.
\]
On this event, if the best returned workflow has score above $\theta$, its
workflow inequality is a valid violated constraint. If its score is at most
$\theta$, then all workflow inequalities are satisfied after replacing
$\theta$ by $\theta+\delta$, increasing the dual objective by at most
$k\delta$.
\end{lemma}

\begin{proof}
Let $Y_j$ be the indicator that primitive call $j$ in the batch is
$\delta$-accurate. The calls may be adaptive to previous failed attempts inside
the batch, but by assumption
\[
    \Prob\{Y_j=1\mid\text{past within the batch}\}
    \ge
    p_{\mathrm{orc}}.
\]
Therefore
\[
    \Prob\{Y_1=\cdots=Y_m=0\mid\text{history}\}
    \le
    (1-p_{\mathrm{orc}})^m
    \le
    e^{-p_{\mathrm{orc}}m}.
\]
Thus, with probability at least $1-e^{-p_{\mathrm{orc}}m}$, at least one
primitive call is $\delta$-accurate. Since the batch keeps the highest-scoring
returned workflow, the batch output is also $\delta$-accurate.

If the batch output satisfies $s_{\widetilde g}(\mu)>\theta$, then the
constraint
\[
    s_{\widetilde g}(\mu)\le\theta
\]
is violated and supplies a valid separating hyperplane. If instead
$s_{\widetilde g}(\mu)\le\theta$, then $\delta$-accuracy gives
\[
    \max_{g\in\G}s_g(\mu)
    \le
    s_{\widetilde g}(\mu)+\delta
    \le
    \theta+\delta.
\]
Thus $s_g(\mu)\le\theta+\delta$ for every $g\in\G$. Since the coefficient of
$\theta$ in \eqref{eq:ellipsoid-dual} is $k$, this relaxation increases the
dual objective by at most $k\delta$. \Halmos
\end{proof}

\subsection{Proof of Theorem~\ref{thm:pricing-certificate}}

For each fixed $k$, \eqref{eq:ellipsoid-dual} is a rational LP in $2n+1$
variables, with polynomially many explicit recovery and box constraints and an
implicit family of workflow constraints. By
\Cref{lem:batched-pricing-separation}, a successful batch supplies weak
separation for the implicit workflow constraints. The weak
optimization--separation equivalence for the ellipsoid method therefore returns
an $\varepsilon_{\mathrm{ell}}/2$-optimal dual solution in a number of
separation queries polynomial in
\[
    n,\quad k,\quad B,\quad \log(1/\varepsilon_{\mathrm{ell}}).
\]
With $\tau_k=\varepsilon_{\mathrm{ell}}/(2k)$, the possible slot-price
relaxation in \Cref{lem:batched-pricing-separation} contributes at most
$k\tau_k=\varepsilon_{\mathrm{ell}}/2$ to the dual objective. Hence, on the
event that all batches used by the ellipsoid routine are successful, standard
primal recovery from the workflow inequalities generated by the routine gives,
for every $k\in\mathcal K$, a feasible fractional execution vector $z^k$
supported only on oracle-returned workflows and satisfying
\begin{equation}
    F_k(z^k)
    \ge
    L_k(\G)-\varepsilon_{\mathrm{ell}}.
    \label{eq:ellipsoid-primal-accuracy}
\end{equation}

It remains to bound the probability that all batches are successful. Let
$N_{\mathrm{ell}}(\varepsilon_{\mathrm{ell}})$ be the deterministic upper bound
on the total number of ellipsoid separation queries across all grid sizes.
Index these separation queries by $q=1,\ldots,N_{\mathrm{ell}}$. Let
$\mathcal B_q$ be the event that batch $q$ contains at least one
$\tau_k$-accurate primitive pricing call, where $k$ is the grid size being
solved at that query. By \Cref{lem:batched-pricing-separation},
\[
    \Prob(\mathcal B_q^c\mid\text{history before batch }q)
    \le
    e^{-p_{\mathrm{orc}}m}.
\]
Therefore,
\begin{align}
    \Prob\left(
        \bigcup_{q=1}^{N_{\mathrm{ell}}}\mathcal B_q^c
    \right)
    &\le
    \sum_{q=1}^{N_{\mathrm{ell}}}
    \E\!\left[
        \Prob(\mathcal B_q^c\mid\text{history before batch }q)
    \right]  \\
    &\le
    N_{\mathrm{ell}}(\varepsilon_{\mathrm{ell}})
    e^{-p_{\mathrm{orc}}m}.
    \label{eq:all-batches-success-tail}
\end{align}
Equivalently, since
$T=mN_{\mathrm{ell}}(\varepsilon_{\mathrm{ell}})$,
\[
    N_{\mathrm{ell}}(\varepsilon_{\mathrm{ell}})
    e^{-p_{\mathrm{orc}}m}
    =
    \exp\left\{
        -\frac{p_{\mathrm{orc}}T}
        {N_{\mathrm{ell}}(\varepsilon_{\mathrm{ell}})}
        +
        \log N_{\mathrm{ell}}(\varepsilon_{\mathrm{ell}})
    \right\}.
\]
This proves the probability statement in
\eqref{eq:ellipsoid-oracle-success-probability}.

On the event that all batches are successful, apply
\Cref{thm:cost-rounding} to each recovered fractional solution $z^k$. Since
$Q_k(z^k)\le1$, for every $k\in\mathcal K$,
\[
    \E_{\mathrm{rnd}}\!\left[
        \Pi_{\lambda,\gamma}(\bm M_k^{\RR})
    \right]
    \ge
    F_k(z^k)
    -
    \frac{c_{k,\lambda}^{\mathrm{PL}}}{e}
    \ge
    L_k(\G)
    -
    \varepsilon_{\mathrm{ell}}
    -
    \frac{c_{k,\lambda}^{\mathrm{PL}}}{e}.
\]
Taking the best grid point gives
\[
    \max_{k\in\mathcal K}
    \E_{\mathrm{rnd}}\!\left[
        \Pi_{\lambda,\gamma}(\bm M_k^{\RR})
    \right]
    \ge
    U_{\mathrm{LP}}(\mathcal K;\G)
    -
    \varepsilon_{\mathrm{ell}}
    -
    \varepsilon_{\mathrm{rnd}}(\mathcal K,\lambda).
\]
By \Cref{thm:geometric-cardinality-approx},
\[
    U_{\mathrm{LP}}(\mathcal K;\G)
    \ge
    \frac1\varrho
    \OPT_{\lambda,\gamma}(\G).
\]
Comparing the rounded grid candidates with the retained baseline $\bm m^0$ proves
\eqref{eq:ellipsoid-additive-final}.

Finally, for every $A\ge0$, $e\ge0$, and $\underline V>0$,
\[
    \max\{\underline V,A-e\}
    \ge
    \frac{\underline V}{\underline V+e}A.
\]
Apply this inequality with
\[
    A=\frac1\varrho\OPT_{\lambda,\gamma}(\G),
    \qquad
    e=
    \varepsilon_{\mathrm{ell}}
    +
    \varepsilon_{\mathrm{rnd}}(\mathcal K,\lambda),
\]
to obtain \eqref{eq:ellipsoid-multiplicative-final}. \Halmos

\section{Managerial Implications and Extensions}
\label{sec:extensions}

\paragraph{Accuracy value and the price of compute.}
The parameter $\gamma$ is not an arbitrary regularizer. If a correct decision is worth $v$ dollars relative to an incorrect one,
maximizing $vJ_\psi(\bm m)-C(\bm m)$ is equivalent to maximizing
$J_\psi(\bm m)-\gamma C(\bm m)$ with $\gamma=1/v$. Varying $\gamma$ traces an accuracy--compute frontier. The screening bound in \Cref{thm:diversification-bound} can eliminate run sizes that cannot be efficient before solving any IP.

\paragraph{Selector investment versus workflow investment.}
The selector-strength cap \eqref{eq:selector-only-gain-cap} maps an odds-lift bound $\Lambda$ into the largest possible gross return from workflow complementarity. When $\Lambda$ is close to one, spending on additional workflows has little possible upside; improving the selector may dominate expanding the portfolio. When $\Lambda$ is large, workflow variety can become more valuable, subject to cost. Under Plackett--Luce, $\Lambda=\lambda$, so selector-model investment can be evaluated through the fitted discrimination parameter.

\paragraph{Shared modules, latency, and nonadditive costs.}
The additive cost
$C(\bm m)=\sum_{g\in\G}c_gm_g$
is appropriate for token or dollar expenditure when all assigned workflow
executions are run. Workflow graphs may share retrieval results, cached model calls, or verification modules, and parallel execution makes latency depend on a critical path rather than a sum. These features can be represented by fixed-charge module variables or a set-dependent cost $C(S)$. The structural accuracy bound remains valid, while the finite-pool optimization becomes a richer mixed-integer problem.

\paragraph{Stochastic workflow outputs.}
If correctness is stochastic, the exact selector-aware contribution of task
$i$ under execution-count vector $\bm m$ is
\[
    \E\left[
        \psi_{k(\bm m)}\bigl(R_i(\bm m)\bigr)
    \right].
\]
where the expectation is over the joint distribution of workflow outcomes. In general this is not equal to applying $\psi$ to the sum of marginal success probabilities. A practical approach is scenario-based sample-average approximation: repeated workflow runs create deterministic correctness scenarios, and the portfolio is optimized against their average selector value \citep{kleywegt2002sample}.

\paragraph{Task-dependent run sets.}
The paper chooses one persistent portfolio for a task population. A pre-output router could select a task-specific subset before execution, yielding a joint routing and post-output selection problem. The current model provides the value and cost of each candidate run set; an outer policy can then allocate run sets by task features or congestion state.

\paragraph{Voting and nonconcave selectors.}
Majority voting and threshold verification can have increasing marginal value near a decision threshold. The general selector-strength bound still applies if a finite common odds-lift envelope exists, even when the raw recovery curve is nonconcave. If recovery reaches one at an interior correct fraction, however, the odds-lift index is infinite and that uniform cap is vacuous. For each fixed size, the exact IP remains valid with arbitrary nondecreasing increments, but the concave-coverage rounding certificate may fail. Endogenous size still matters because adding a wrong vote can move the system away from the threshold and incurs compute. Selector-specific approximation or robust envelope methods are natural extensions.

\paragraph{Creation cost versus execution cost.}
Execution cost $c_g$ recurs on every deployed task. Workflow creation and evaluation costs are one-time investments. The generation algorithm can impose a separate budget on oracle calls, candidate evaluation, or human engineering. A longer planning horizon increases the relative importance of recurring execution cost, while a short-lived application may place more weight on creation cost.

\paragraph{Robust selector calibration.}
The fitted $\lambda$ may vary across task types or drift when portfolios are optimized. For the structural cap, a robust analysis can construct an upper confidence envelope $H^+$ for recovery and use $\Lambda(H^+)$, or use an upper confidence bound for $\lambda$ under Plackett--Luce. For portfolio optimization, one can optimize worst-case net value over a family of recovery curves or impose a distribution over task-specific selector strengths. The cardinality-grid and IP architecture remains unchanged; only the recovery coefficients $d_{\ell,k}$ vary. If the recovery curves do not share a common concave fraction-based representation, the exact fixed-size IPs remain valid but the geometric cardinality approximation should be treated as a Plackett--Luce or common-fraction specialization rather than a universal guarantee.

\section{Proofs for Stochastic Workflow Outcomes}
\label{app:stochastic-workflow-proofs}

This appendix gives the complete two-stage stochastic procedure and its proof.
Stage~1 uses $L_1$ independent executions to estimate each task--workflow
success probability and constructs a plug-in iid execution law. Workflow
generation, stochastic pricing, and primal recovery are performed under that
plug-in law. After the generated workflow pool is frozen, Stage~2 discards the
Stage-1 outcomes for selection purposes and uses $L_2$ fresh execution panels
to choose the final portfolio.

For the analysis, associate with every task--workflow pair $(i,g)$ a potential
Stage-1 pilot panel
\begin{equation}
    \left(
        Z_{ig}^{(1,\ell)}:
        \ell=1,\ldots,L_1
    \right),
    \qquad
    Z_{ig}^{(1,\ell)}
    \stackrel{\mathrm{iid}}{\sim}
    \mathrm{Bernoulli}(a_{ig}),
    \label{eq:stoch-app-potential-pilot-table}
\end{equation}
with all panels mutually independent. The algorithm reveals a panel only when
the corresponding workflow is first evaluated. This lazy revelation is
equivalent to drawing the entire finite table in advance and does not require
enumerating $\G$ computationally.

\begin{lemma}[Uniform concentration of the $L_1$ workflow estimates]
\label{lem:stoch-stage-one-concentration-app}
Under \Cref{ass:stochastic-execution-law}, for every
$\delta_1\in(0,1)$,
\begin{equation}
    \Prob\left\{
        \max_{\substack{i=1,\ldots,n\\g\in\G}}
        \left|\widehat a_{ig}^{(1)}-a_{ig}\right|
        \le
        \varepsilon_a(L_1,\delta_1)
    \right\}
    \ge
    1-\delta_1,
    \label{eq:stoch-app-stage-one-event}
\end{equation}
where $\varepsilon_a(L_1,\delta_1)$ is defined in
\eqref{eq:stoch-stage-one-uniform}.
\end{lemma}

\begin{proof}
For each fixed $(i,g)$, Hoeffding's inequality gives
\begin{equation*}
    \Prob\left\{
        \left|\widehat a_{ig}^{(1)}-a_{ig}\right|>t
    \right\}
    \le
    2e^{-2L_1t^2}.
\end{equation*}
A union bound over the $n|\G|$ potential pilot panels and the choice
$t=\varepsilon_a(L_1,\delta_1)$ prove the result. Since the complete pilot
table can be regarded as drawn before the generation process begins, revealing
its entries adaptively does not alter the bound. \Halmos
\end{proof}

\begin{lemma}[Perturbation of iid stochastic portfolio values]
\label{lem:stoch-value-perturbation-app}
Let $b=(b_{ig})$ and $b'=(b'_{ig})$ be two success-probability matrices such
that
$\max_{i,g}|b_{ig}-b'_{ig}|\le\eta$. Let
$\Pi_{\psi,\gamma}^{b}$ and $\Pi_{\psi,\gamma}^{b'}$ denote portfolio values
under the corresponding product-Bernoulli execution laws. Then
\begin{equation}
    \sup_{\substack{\bm m\in\mathbb Z_+^{\G}\\
                    k(\bm m)\le K_{\max}}}
    \left|
        \Pi_{\psi,\gamma}^{b}(\bm m)
        -
        \Pi_{\psi,\gamma}^{b'}(\bm m)
    \right|
    \le
    \min\{1,K_{\max}\eta\}.
    \label{eq:stoch-app-value-perturbation}
\end{equation}
Consequently, on the event in
\eqref{eq:stoch-app-stage-one-event},
\begin{equation}
    \sup_{\substack{\bm m\in\mathbb Z_+^{\G}\\
                    k(\bm m)\le K_{\max}}}
    \left|
        \widehat\Pi_{\psi,\gamma}^{(1)}(\bm m)
        -
        \Pi_{\psi,\gamma}^{\mathrm{stoch}}(\bm m)
    \right|
    \le
    \varepsilon_{\mathrm{est}}(L_1,\delta_1).
    \label{eq:stoch-app-plugin-true-uniform}
\end{equation}
\end{lemma}

\begin{proof}
Fix $\bm m$ and couple every Bernoulli execution under $b$ and $b'$ using the
same independent uniform random variable. On task $i$, one execution of
workflow $g$ differs under the two laws with probability
$|b_{ig}-b'_{ig}|$. Hence the probability that any of the $k(\bm m)$ coupled
execution outcomes differs is at most
\begin{equation*}
    \sum_{g\in\G}m_g|b_{ig}-b'_{ig}|
    \le
    k(\bm m)\eta.
\end{equation*}
If no coupled outcome differs, the two correct counts and therefore the two
recovery values coincide. Since recovery lies in $[0,1]$, the difference in
expected task value is at most
$\min\{1,k(\bm m)\eta\}$. Averaging over tasks and observing that the cost
term is identical under the two laws prove
\eqref{eq:stoch-app-value-perturbation}. The second claim follows from
\Cref{lem:stoch-stage-one-concentration-app}. \Halmos
\end{proof}

For a fixed run size $k$, label the execution copies by $q=1,\ldots,k$. The
Stage-1 estimates are integer multiples of $1/L_1$. This permits a common
finite scenario representation of the entire plug-in law that does not change
when a new workflow is revealed. Let
\begin{equation}
    \Omega_k
    =
    \{1,\ldots,L_1\}^{n\times k},
    \qquad
    p_\omega=L_1^{-nk},
    \label{eq:stoch-app-scenario-space}
\end{equation}
and write $U_{iq}(\omega)\in\{1,\ldots,L_1\}$ for coordinate $(i,q)$ of
scenario $\omega$. For every workflow $g$, define its potential plug-in
correctness outcome in scenario $\omega$ by
\begin{equation}
    \widehat A_{igq}^{(1)}(\omega)
    =
    \1\left\{
        U_{iq}(\omega)
        \le
        L_1\widehat a_{ig}^{(1)}
    \right\}.
    \label{eq:stoch-app-scenario-array}
\end{equation}
For each fixed integral assignment of workflows to execution copies, the
variables in \eqref{eq:stoch-app-scenario-array} are independent across tasks
and copies and have the required Bernoulli means
$\widehat a_{ig}^{(1)}$. The same uniforms may be used for different workflow
choices within one copy because an integral assignment selects only one
workflow for that copy.

For a finite workflow class $\M$, define
\begin{equation}
    \mathcal C_k(\M)
    =
    \left\{
        \bm m\in\mathbb Z_+^{\M}:
        \sum_{g\in\M}m_g=k
    \right\},
    \qquad
    |\mathcal C_k(\M)|
    =
    \binom{|\M|+k-1}{k},
    \label{eq:stoch-app-multiset-class}
\end{equation}
and let
\begin{equation}
    \widehat\OPT_k^{(1)}(\M)
    =
    \max_{\bm m\in\mathcal C_k(\M)}
    \widehat\Pi_{\psi,\gamma}^{(1)}(\bm m).
    \label{eq:stoch-app-fixed-size-optimum}
\end{equation}

\begin{lemma}[Plug-in stochastic LP relaxation and projected dual]
\label{lem:stoch-true-dual-app}
Fix $k$ and suppose
$d_{1,k}\ge\cdots\ge d_{k,k}\ge0$. Let
$\widetilde\psi_k$ be defined by
\eqref{eq:fixed-size-analytical-continuation}, and set
\begin{equation}
    \mathcal X_k(\G)
    =
    \left\{
        x\ge0:
        \sum_{g\in\G}x_{gq}=1,
        \ q=1,\ldots,k
    \right\}.
    \label{eq:stoch-app-slot-simplex}
\end{equation}
For $x\in\mathcal X_k(\G)$, define
\begin{align}
    \widehat s_{i\omega}^{(1)}(x)
    &:={}
    \sum_{q=1}^k\sum_{g\in\G}
    \widehat A_{igq}^{(1)}(\omega)x_{gq},
    \label{eq:stoch-app-fractional-row-count}\\
    \widehat{\mathcal F}_k^{(1)}(x)
    &:={}
    \sum_{\omega\in\Omega_k}p_\omega
    \frac1n\sum_{i=1}^n
    \widetilde\psi_k\bigl(\widehat s_{i\omega}^{(1)}(x)\bigr)
    -\gamma\sum_{q=1}^k\sum_{g\in\G}c_gx_{gq},
    \label{eq:stoch-app-fractional-objective}\\
    \widehat{\mathcal L}_k^{(1)}(\G)
    &:={}
    \max_{x\in\mathcal X_k(\G)}
    \widehat{\mathcal F}_k^{(1)}(x).
    \label{eq:stoch-app-primal-relaxation}
\end{align}
Then
\begin{equation}
    \widehat{\mathcal L}_k^{(1)}(\G)
    \ge
    \widehat\OPT_k^{(1)}(\G).
    \label{eq:stoch-app-lp-upper-bound}
\end{equation}
Define
\begin{equation}
    \chi_k(u)
    =
    \max_{r=0,\ldots,k}
    \left\{
        \frac{\psi_k(r)}n-ur
    \right\},
    \qquad
    0\le u\le\frac{d_{1,k}}n.
    \label{eq:stoch-app-chi}
\end{equation}
Strong duality gives
\begin{align}
    \widehat{\mathcal L}_k^{(1)}(\G)
    =
    \min_{\mu,\theta}\quad
    &
    \sum_{\omega\in\Omega_k}p_\omega
    \sum_{i=1}^n\chi_k(\mu_{i\omega})
    +\sum_{q=1}^k\theta_q
    \label{eq:stoch-app-projected-dual}\\
    \textnormal{s.t.}\quad
    &
    \sum_{\omega\in\Omega_k}p_\omega
    \sum_{i=1}^n
    \mu_{i\omega}\widehat A_{igq}^{(1)}(\omega)
    -\gamma c_g
    \le
    \theta_q,
    &&g\in\G,\ q=1,\ldots,k,
    \nonumber\\
    &
    0\le\mu_{i\omega}\le\frac{d_{1,k}}n,
    &&i=1,\ldots,n,\ \omega\in\Omega_k.
    \nonumber
\end{align}
Equivalently, introducing epigraph variables $\xi_{i\omega}$ gives the
rational LP
\begin{align}
    \widehat{\mathcal L}_k^{(1)}(\G)
    =
    \min_{\mu,\xi,\theta}\quad
    &
    \sum_{\omega\in\Omega_k}p_\omega
    \sum_{i=1}^n\xi_{i\omega}
    +\sum_{q=1}^k\theta_q
    \label{eq:stoch-app-epigraph-dual}\\
    \textnormal{s.t.}\quad
    &
    \xi_{i\omega}+r\mu_{i\omega}
    \ge
    \frac{\psi_k(r)}n,
    &&i=1,\ldots,n,\ \omega\in\Omega_k,\ r=0,\ldots,k,
    \nonumber\\
    &
    \sum_{\omega\in\Omega_k}p_\omega
    \sum_{i=1}^n
    \mu_{i\omega}\widehat A_{igq}^{(1)}(\omega)
    -\gamma c_g
    \le
    \theta_q,
    &&g\in\G,\ q=1,\ldots,k,
    \nonumber\\
    &
    0\le\mu_{i\omega}\le\frac{d_{1,k}}n,
    \qquad
    0\le\xi_{i\omega}\le\frac1n,
    &&i=1,\ldots,n,\ \omega\in\Omega_k,
    \nonumber\\
    &
    -\gamma c_{\max}\le\theta_q\le d_{1,k},
    &&q=1,\ldots,k.
    \nonumber
\end{align}
The plug-in stochastic pricing score is therefore
\begin{equation}
    \widehat{\mathcal S}_{kq}^{(1)}(g;\mu)
    =
    \sum_{\omega\in\Omega_k}p_\omega
    \sum_{i=1}^n
    \mu_{i\omega}\widehat A_{igq}^{(1)}(\omega)
    -\gamma c_g.
    \label{eq:stoch-plugin-reduced-cost-main}
\end{equation}
\end{lemma}

\begin{proof}
An integral $x$ assigns one workflow to every execution copy $q$ and induces
counts $m_g=\sum_qx_{gq}$. Its scenario-wise correct counts are integral, so
$\widetilde\psi_k=\psi_k$ at those counts. By the construction in
\eqref{eq:stoch-app-scenario-array}, its scenario distribution is exactly the
Stage-1 plug-in iid law. Thus its objective is
$\widehat\Pi_{\psi,\gamma}^{(1)}(\bm m)$, and relaxing integrality proves
\eqref{eq:stoch-app-lp-upper-bound}.

Let $\varphi_k(t)=\widetilde\psi_k(t)/n$. Concavity gives
\begin{equation}
    \varphi_k(t)
    =
    \min_{0\le u\le d_{1,k}/n}
    \{ut+\chi_k(u)\},
    \qquad
    0\le t\le k.
    \label{eq:stoch-app-conjugate}
\end{equation}
Substitute \eqref{eq:stoch-app-conjugate} into
\eqref{eq:stoch-app-fractional-objective} and apply finite-dimensional minimax
to interchange maximization over $x$ and minimization over $\mu$. For fixed
$\mu$, maximization separates by execution copy:
\begin{align}
    &\max_{x\in\mathcal X_k(\G)}
    \sum_{q=1}^k\sum_{g\in\G}x_{gq}
    \widehat{\mathcal S}_{kq}^{(1)}(g;\mu)
    \nonumber\\
    &\hspace{30mm}=
    \sum_{q=1}^k
    \max_{g\in\G}
    \widehat{\mathcal S}_{kq}^{(1)}(g;\mu).
    \label{eq:stoch-app-slot-support}
\end{align}
Introducing the variables $\theta_q$ yields
\eqref{eq:stoch-app-projected-dual}; replacing each $\chi_k$ by its finite
epigraph representation yields \eqref{eq:stoch-app-epigraph-dual}. \Halmos
\end{proof}

Even under iid execution, the score in
\eqref{eq:stoch-plugin-reduced-cost-main} is not generally equal to the
deterministic mean-column score
$\sum_i\bar\mu_i\widehat a_{ig}^{(1)}-\gamma c_g$. The scenario-dependent
marginal price $\mu_{i\omega}$ and the candidate's correctness indicator are
evaluated in the same execution scenario, so averaging them separately need
not preserve their product.

\begin{algorithm}[t]
\caption{IID Stochastic Dual-Guided Workflow Optimization}
\label{alg:stoch-cost-aware-generation}
\begin{algorithmic}[1]
\small
\STATE \textbf{Input}: initial workflow pool $\M_0$; cardinality grid
$\mathcal K$; recovery curves $\{\psi_k\}$; cost price $\gamma$; Stage-1
sample size $L_1$ and confidence level $\delta_1$; Stage-2 sample size $L_2$
and confidence level $\delta_2$; ellipsoid tolerance
$\varepsilon_{\mathrm{ell}}>0$; pricing-batch size $m$; primitive stochastic
pricing oracle; baseline $\bm m^0$ with certified value $\underline V>0$ and
$\operatorname{supp}(\bm m^0)\subseteq\M_0$.
\STATE For every $g\in\M_0$ and task $i$, collect $L_1$ independent
executions and compute $\widehat a_{ig}^{(1)}$ using
\eqref{eq:stoch-stage-one-estimator}. Set
$\M_{\mathrm{rec}}\leftarrow\M_0$.
\FOR{$k\in\mathcal K$}
    \STATE Set
    $\tau_k\leftarrow\varepsilon_{\mathrm{ell}}/(2k)$.
    \STATE Run the ellipsoid method on the plug-in epigraph dual
    \eqref{eq:stoch-app-epigraph-dual} with optimization tolerance
    $\varepsilon_{\mathrm{ell}}/2$.
    \WHILE{the ellipsoid routine has not terminated}
        \STATE At the current point $(\mu^t,\xi^t,\theta^t)$, separate violated
        explicit recovery and box constraints directly.
        \IF{no explicit constraint is violated}
            \FOR{$q=1,\ldots,k$}
                \STATE Make $m$ primitive pricing calls at
                $(k,q,\mu^t,\tau_k)$. Whenever a call proposes a workflow $g$
                that has not been evaluated, collect its independent
                $L_1$-sample panel, compute
                $(\widehat a_{1g}^{(1)},\ldots,\widehat a_{ng}^{(1)})$, and
                set $\M_{\mathrm{rec}}\leftarrow
                \M_{\mathrm{rec}}\cup\{g\}$ before scoring it.
                \STATE Retain the proposed workflow $g^{tq}$ with the largest
                score $\widehat{\mathcal S}_{kq}^{(1)}(g;\mu^t)$.
                \IF{$\widehat{\mathcal S}_{kq}^{(1)}(g^{tq};\mu^t)>
                \theta_q^t$}
                    \STATE Add its violated workflow constraint to the
                    ellipsoid system.
                \ELSE
                    \STATE Use the weak-separation certificate
                    $\max_{g\in\G}\widehat{\mathcal S}_{kq}^{(1)}(g;\mu^t)
                    \le\theta_q^t+\tau_k$.
                \ENDIF
            \ENDFOR
        \ENDIF
    \ENDWHILE
    \STATE By primal recovery, construct
    $x^k\in\mathcal X_k(\G)$ supported on $\M_{\mathrm{rec}}$ such that
    $\widehat{\mathcal F}_k^{(1)}(x^k)
    \ge\widehat{\mathcal L}_k^{(1)}(\G)
    -\varepsilon_{\mathrm{ell}}$.
\ENDFOR
\STATE Freeze the generated pool
$\M_T\leftarrow\M_{\mathrm{rec}}$. Do not reuse Stage-1 execution outcomes for
final portfolio selection.
\STATE Draw the $L_2$ independent Stage-2 execution panels
$\{Z_{igq}^{(2,\ell)}\}$, compute
$\widehat\Pi_{L_2}^{\mathrm{final}}$ for every
$\bm m\in\mathcal C_{\mathcal K}(\M_T)$, and choose the empirical maximizer
$\widetilde{\bm m}$.
\STATE Return $\widehat{\bm m}$ according to the conservative rule
\eqref{eq:stoch-conservative-final-rule}.
\normalsize
\end{algorithmic}
\end{algorithm}

At a query $(k,q,\mu)$, a primitive plug-in pricing call is
$\tau_k$-accurate if it returns $\widetilde g$ satisfying
\begin{equation}
    \widehat{\mathcal S}_{kq}^{(1)}(\widetilde g;\mu)
    \ge
    \max_{g\in\G}
    \widehat{\mathcal S}_{kq}^{(1)}(g;\mu)
    -\tau_k.
    \label{eq:stoch-plugin-oracle-guarantee-app}
\end{equation}
The guarantee is defined relative to the complete potential Stage-1 pilot
table in \eqref{eq:stoch-app-potential-pilot-table}; only the panels of
workflows actually proposed by the oracle need to be materialized.

\begin{lemma}[Batched plug-in stochastic pricing]
\label{lem:stoch-pricing-separation-app}
Suppose each primitive pricing call satisfies
\eqref{eq:stoch-plugin-oracle-guarantee-app}, conditional on the complete query
history, with probability at least $p_{\mathrm{orc}}>0$. If each pricing query
uses $m$ primitive calls and retains the highest-scoring returned workflow,
let $\mathcal E_{\mathrm{orc}}$ be the event that every resulting batch is
$\tau_k$-accurate. If at most
$N_{\mathrm{price}}^{\mathrm{stoch}}$ batches are used, then
\begin{equation}
    \Prob(\mathcal E_{\mathrm{orc}})
    \ge
    1-
    N_{\mathrm{price}}^{\mathrm{stoch}}
    e^{-p_{\mathrm{orc}}m}.
    \label{eq:stoch-app-oracle-event-probability}
\end{equation}
\end{lemma}

\begin{proof}
A batch fails only if all $m$ primitive calls fail, which has conditional
probability at most
$(1-p_{\mathrm{orc}})^m\le e^{-p_{\mathrm{orc}}m}$. A union bound over the
pricing batches proves the claim. \Halmos
\end{proof}

\begin{lemma}[Ellipsoid recovery under plug-in iid pricing]
\label{lem:stoch-ellipsoid-recovery-app}
For each $k\in\mathcal K$, run the plug-in stochastic dual with ellipsoid
optimization tolerance $\varepsilon_{\mathrm{ell},k}/2$ and pricing tolerance
$\tau_k=\varepsilon_{\mathrm{ell},k}/(2k)$. On
$\mathcal E_{\mathrm{orc}}$, primal recovery returns an
$x^k\in\mathcal X_k(\G)$ supported on the returned workflow pool $\M_T$ such
that
\begin{equation}
    \widehat{\mathcal F}_k^{(1)}(x^k)
    \ge
    \widehat{\mathcal L}_k^{(1)}(\G)
    -\varepsilon_{\mathrm{ell},k}.
    \label{eq:stoch-app-ellipsoid-primal-recovery}
\end{equation}
\end{lemma}

\begin{proof}
Successful plug-in pricing supplies weak separation. If the best returned
score for copy $q$ is at most $\theta_q$, then
\eqref{eq:stoch-plugin-oracle-guarantee-app} implies that all copy-$q$
workflow constraints hold after replacing $\theta_q$ by
$\theta_q+\tau_k$. Across all $k$ copies, this increases the dual objective by
at most $k\tau_k=\varepsilon_{\mathrm{ell},k}/2$. Ellipsoid optimization
contributes the remaining half. Standard optimization--separation and primal
recovery give \eqref{eq:stoch-app-ellipsoid-primal-recovery}. \Halmos
\end{proof}

\begin{lemma}[Rounding the plug-in stochastic LP]
\label{lem:stoch-lp-rounding-app}
Fix $k$ and $x\in\mathcal X_k(\G)$. Independently for every copy $q$, draw
$G_q$ with $\Prob\{G_q=g\}=x_{gq}$ and define
\begin{equation*}
    M_{k,g}^{\RR}(x)
    =
    \sum_{q=1}^k\1\{G_q=g\},
    \qquad
    \bm M_k^{\RR}(x)
    =
    \bigl(M_{k,g}^{\RR}(x):g\in\G\bigr).
\end{equation*}
Let
\begin{align}
    \widehat Q_k^{(1)}(x)
    &:={}
    \sum_{\omega\in\Omega_k}p_\omega
    \frac1n\sum_{i=1}^n
    \widetilde\psi_k\bigl(\widehat s_{i\omega}^{(1)}(x)\bigr),
    \label{eq:stoch-app-rounding-Q}\\
    \widehat R_k^{(1)}(x)
    &:={}
    \sum_{\omega\in\Omega_k}p_\omega
    \frac1n\sum_{i=1}^n
    \widehat s_{i\omega}^{(1)}(x),
    \qquad
    C_k(x)
    :={}
    \sum_{q=1}^k\sum_{g\in\G}c_gx_{gq}.
    \label{eq:stoch-app-rounding-RC}
\end{align}
Then
\begin{align}
    \E_{\mathrm{rnd}}\!\left[
        \widehat\Pi_{\psi,\gamma}^{(1)}
        \bigl(\bm M_k^{\RR}(x)\bigr)
    \right]
    &\ge
    \left(1-\frac1e\right)\widehat Q_k^{(1)}(x)
    +\frac{d_{k,k}}e\widehat R_k^{(1)}(x)
    -\gamma C_k(x),
    \label{eq:stoch-app-rounding-lower}\\
    \widehat{\mathcal F}_k^{(1)}(x)
    -
    \E_{\mathrm{rnd}}\!\left[
        \widehat\Pi_{\psi,\gamma}^{(1)}
        \bigl(\bm M_k^{\RR}(x)\bigr)
    \right]
    &\le
    \frac1e
    \left\{
        \widehat Q_k^{(1)}(x)
        -d_{k,k}\widehat R_k^{(1)}(x)
    \right\}
    \le
    \frac{c_{\psi,k}}e\widehat Q_k^{(1)}(x).
    \label{eq:stoch-app-rounding-gap}
\end{align}
Since $\widehat Q_k^{(1)}(x)\le1$, a uniform rounding loss is
\begin{equation}
    \varepsilon_{\mathrm{rnd},k}
    =
    \frac{c_{\psi,k}}e.
    \label{eq:stoch-app-uniform-rounding-error}
\end{equation}
\end{lemma}

\begin{proof}
Fix $(i,\omega)$ and set
$B_q=\widehat A_{i,G_q,q}^{(1)}(\omega)$ and
$p_q=\sum_gx_{gq}\widehat A_{igq}^{(1)}(\omega)$. Conditional on $\omega$,
the $B_q$ are independent under the rounding and
$N=\sum_qB_q$ has mean $\widehat s_{i\omega}^{(1)}(x)$. For every integer
$t\ge1$,
\begin{equation}
    \min\{r,t\}
    \ge
    t\{1-(1-1/t)^r\}.
    \label{eq:stoch-app-truncation-inequality}
\end{equation}
Hence
\begin{align}
    \E_{\mathrm{rnd}}[\min\{N,t\}\mid\omega]
    &\ge
    t\left[
        1-\prod_{q=1}^k\left(1-\frac{p_q}{t}\right)
    \right]
    \nonumber\\
    &\ge
    t\left(1-e^{-\widehat s_{i\omega}^{(1)}(x)/t}\right)
    \nonumber\\
    &\ge
    \left(1-\frac1e\right)
    \min\{\widehat s_{i\omega}^{(1)}(x),t\}.
    \label{eq:stoch-app-threshold-correlation}
\end{align}
Using
\begin{equation}
    \widetilde\psi_k(r)
    =
    d_{k,k}r
    +
    \sum_{t=1}^{k-1}
    (d_{t,k}-d_{t+1,k})\min\{r,t\},
    \label{eq:stoch-app-threshold-decomposition}
\end{equation}
and averaging over tasks and scenarios prove
\eqref{eq:stoch-app-rounding-lower}. Rounding also preserves execution cost in
expectation. Subtraction gives the first inequality in
\eqref{eq:stoch-app-rounding-gap}. Finally,
$\widetilde\psi_k(r)\le d_{1,k}r$ and
$d_{k,k}=(1-c_{\psi,k})d_{1,k}$ give the curvature bound, while
$0\le\widetilde\psi_k(r)\le1$ gives
$\widehat Q_k^{(1)}(x)\le1$. \Halmos
\end{proof}

\begin{lemma}[Stochastic cardinality grid under the plug-in iid law]
\label{lem:stoch-grid-app}
Suppose $\psi_k(r)=h(r/k)$ for a common increasing concave function
$h:[0,1]\to[0,1]$ with $h(0)=0$. For every workflow class $\M$ and
$1\le b\le k\le K_{\max}$,
\begin{equation}
    \left\{\widehat\OPT_b^{(1)}(\M)\right\}^+
    \ge
    \frac bk
    \left\{\widehat\OPT_k^{(1)}(\M)\right\}^+.
    \label{eq:stoch-app-cross-cardinality}
\end{equation}
Consequently, if $\mathcal K$ has coverage ratio $\varrho$,
\begin{equation}
    \max\left\{
        0,
        \max_{b\in\mathcal K}
        \widehat\OPT_b^{(1)}(\M)
    \right\}
    \ge
    \frac1\varrho
    \widehat\OPT_{\psi,\gamma}^{(1)}(\M),
    \label{eq:stoch-app-grid-guarantee}
\end{equation}
where
\begin{equation*}
    \widehat\OPT_{\psi,\gamma}^{(1)}(\M)
    =
    \max_{\substack{\bm m\in\mathbb Z_+^{\M}\\
                    k(\bm m)\le K_{\max}}}
    \widehat\Pi_{\psi,\gamma}^{(1)}(\bm m).
\end{equation*}
\end{lemma}

\begin{proof}
Fix $\bm m\in\mathcal C_k(\M)$, label its $k$ execution copies, retain $b$
of them uniformly without replacement, and let $\bm M^{(b)}$ be the retained
count vector. Set $\alpha=b/k$. Conditional on a realized plug-in size-$k$
execution, if task $i$ has $r_i$ correct copies and $R_i$ of them are retained,
then
$R_i\sim\mathrm{Hypergeom}(k,r_i,b)$. For $r_i>0$, set
$p_i=r_i/k$ and $X_i=R_i/b$. Concavity and monotonicity give
\begin{equation}
    h(x)
    \ge
    h(p_i)\min\{x/p_i,1\},
    \qquad
    x\in[0,1].
    \label{eq:stoch-app-grid-support}
\end{equation}
Writing $Z_i=R_i/r_i$, we have $X_i/p_i=Z_i/\alpha$,
$\E[Z_i]=\alpha$, and $\min\{z,\alpha\}\ge\alpha z$ on $[0,1]$. Hence
$\E[h(R_i/b)]\ge\alpha h(r_i/k)$. The case $r_i=0$ is immediate.

By the iid plug-in execution law, any retained set of labeled copies has the
same distribution as a direct run of the retained execution-count vector, and
$\E[C(\bm M^{(b)})]=\alpha C(\bm m)$. Therefore
\begin{equation*}
    \E\left[
        \widehat\Pi_{\psi,\gamma}^{(1)}(\bm M^{(b)})
    \right]
    \ge
    \alpha
    \widehat\Pi_{\psi,\gamma}^{(1)}(\bm m).
\end{equation*}
Some retained multiset attains at least this expected value. Applying the
result to an optimal positive-valued size-$k$ vector proves
\eqref{eq:stoch-app-cross-cardinality}; choosing a grid anchor
$b\le k\le\varrho b$ proves \eqref{eq:stoch-app-grid-guarantee}. \Halmos
\end{proof}

\begin{lemma}[Plug-in generated-pool certificate]
\label{lem:stoch-generated-pool-certificate-app}
On $\mathcal E_{\mathrm{orc}}$, define
\begin{equation*}
    \varepsilon_{\mathrm{ell}}
    =
    \max_{k\in\mathcal K}\varepsilon_{\mathrm{ell},k},
    \qquad
    \varepsilon_{\mathrm{rnd}}
    =
    \max_{k\in\mathcal K}\varepsilon_{\mathrm{rnd},k},
\end{equation*}
and let
\begin{equation*}
    \widehat V_T^{(1)}
    =
    \max_{\bm m\in\mathcal C_{\mathcal K}(\M_T)}
    \widehat\Pi_{\psi,\gamma}^{(1)}(\bm m).
\end{equation*}
Then
\begin{equation}
    \widehat V_T^{(1)}
    \ge
    \frac1\varrho
    \widehat\OPT_{\psi,\gamma}^{(1)}(\G)
    -\varepsilon_{\mathrm{ell}}
    -\varepsilon_{\mathrm{rnd}}.
    \label{eq:stoch-app-plugin-generation-certificate}
\end{equation}
\end{lemma}

\begin{proof}
For every $k\in\mathcal K$, the recovered $x^k$ is supported on $\M_T$.
By \Cref{lem:stoch-ellipsoid-recovery-app,lem:stoch-lp-rounding-app},
\begin{align}
    \widehat\OPT_k^{(1)}(\M_T)
    &\ge
    \E_{\mathrm{rnd}}\!\left[
        \widehat\Pi_{\psi,\gamma}^{(1)}
        \bigl(\bm M_k^{\RR}(x^k)\bigr)
    \right]
    \nonumber\\
    &\ge
    \widehat{\mathcal L}_k^{(1)}(\G)
    -\varepsilon_{\mathrm{ell},k}
    -\varepsilon_{\mathrm{rnd},k}
    \nonumber\\
    &\ge
    \widehat\OPT_k^{(1)}(\G)
    -\varepsilon_{\mathrm{ell},k}
    -\varepsilon_{\mathrm{rnd},k}.
    \label{eq:stoch-app-fixed-size-pool-quality}
\end{align}
By \Cref{lem:stoch-grid-app}, some $k\in\mathcal K$ satisfies
\begin{equation*}
    \widehat\OPT_k^{(1)}(\G)
    \ge
    \frac1\varrho
    \widehat\OPT_{\psi,\gamma}^{(1)}(\G).
\end{equation*}
Combining the two displays proves
\eqref{eq:stoch-app-plugin-generation-certificate}. \Halmos
\end{proof}

\begin{lemma}[Transfer of the generated-pool certificate to the true iid law]
\label{lem:stoch-plugin-transfer-app}
On the intersection of the Stage-1 concentration event
\eqref{eq:stoch-app-stage-one-event} and $\mathcal E_{\mathrm{orc}}$,
\begin{equation}
    \max_{\bm m\in\mathcal C_{\mathcal K}(\M_T)}
    \Pi_{\psi,\gamma}^{\mathrm{stoch}}(\bm m)
    \ge
    \frac1\varrho
    \OPT_{\psi,\gamma}^{\mathrm{stoch}}(\G)
    -\varepsilon_{\mathrm{ell}}
    -\varepsilon_{\mathrm{rnd}}
    -\left(1+\frac1\varrho\right)
        \varepsilon_{\mathrm{est}}(L_1,\delta_1).
    \label{eq:stoch-app-true-generation-certificate}
\end{equation}
In particular, the final term is at most
$2\varepsilon_{\mathrm{est}}(L_1,\delta_1)$.
\end{lemma}

\begin{proof}
By \Cref{lem:stoch-value-perturbation-app}, the true and plug-in values of
every feasible portfolio differ by at most
$\varepsilon_{\mathrm{est}}(L_1,\delta_1)$. Therefore
\begin{align*}
    \max_{\bm m\in\mathcal C_{\mathcal K}(\M_T)}
    \Pi_{\psi,\gamma}^{\mathrm{stoch}}(\bm m)
    &\ge
    \widehat V_T^{(1)}
    -\varepsilon_{\mathrm{est}}(L_1,\delta_1)\\
    &\ge
    \frac1\varrho
    \widehat\OPT_{\psi,\gamma}^{(1)}(\G)
    -\varepsilon_{\mathrm{ell}}
    -\varepsilon_{\mathrm{rnd}}
    -\varepsilon_{\mathrm{est}}(L_1,\delta_1)\\
    &\ge
    \frac1\varrho
    \OPT_{\psi,\gamma}^{\mathrm{stoch}}(\G)
    -\varepsilon_{\mathrm{ell}}
    -\varepsilon_{\mathrm{rnd}}
    -\left(1+\frac1\varrho\right)
        \varepsilon_{\mathrm{est}}(L_1,\delta_1),
\end{align*}
which proves the claim. \Halmos
\end{proof}

Let $\mathcal H_T$ denote the sigma-field generated by the completed workflow
generation stage, including all Stage-1 pilot outcomes and oracle randomness.

\begin{lemma}[Conditional fresh-$L_2$ deviation]
\label{lem:stoch-fresh-sample-app}
Let $\mathcal A_T$ be a finite, $\mathcal H_T$-measurable class of
execution-count vectors. Conditional on $\mathcal H_T$, with probability at
least $1-\delta_2$,
\begin{equation}
    \sup_{\bm m\in\mathcal A_T}
    \left|
        \widehat\Pi_{L_2}^{\mathrm{final}}(\bm m)
        -
        \Pi_{\psi,\gamma}^{\mathrm{stoch}}(\bm m)
    \right|
    \le
    \sqrt{
        \frac{\log(2|\mathcal A_T|/\delta_2)}{2nL_2}
    }.
    \label{eq:stoch-app-uniform-deviation}
\end{equation}
Consequently, an exact empirical maximizer over $\mathcal A_T$ has true value
at least the best true value in $\mathcal A_T$ minus twice the right-hand side.
\end{lemma}

\begin{proof}
For fixed $\bm m$, the $nL_2$ random variables
\begin{equation*}
    \psi_{k(\bm m)}
    \bigl(R_i^{(2,\ell)}(\bm m)\bigr),
    \qquad
    i=1,\ldots,n,
    \quad
    \ell=1,\ldots,L_2,
\end{equation*}
are conditionally independent under
\Cref{ass:stochastic-execution-law}, lie in $[0,1]$, and have average mean
equal to the stochastic accuracy component of $\bm m$. Hoeffding's inequality
and a union bound over $\mathcal A_T$ prove
\eqref{eq:stoch-app-uniform-deviation}. Applying the uniform bound once to the
empirical maximizer and once to a true maximizer gives the final statement.
\Halmos
\end{proof}

\paragraph{Proof of Proposition~\ref{prop:stoch-odds-lift-main}.}
Fix a nonzero $\bm m$ and write $k=k(\bm m)$. For every task $i$,
\Cref{ass:recovery-envelope,lem:odds-lift-envelope} and concavity of
$h_\Lambda$ give
\begin{equation}
    \E[\psi_k(R_i(\bm m))]
    \le
    \E\!\left[
        h_\Lambda\!\left(\frac{R_i(\bm m)}k\right)
    \right]
    \le
    h_\Lambda\!\left(
        \frac{\E[R_i(\bm m)]}{k}
    \right).
    \label{eq:stoch-app-odds-jensen-chain}
\end{equation}
Using \eqref{eq:stoch-correct-count-main}, averaging over tasks, and applying
Jensen once more yield
\begin{equation*}
    \frac1n\sum_i\E[\psi_k(R_i(\bm m))]
    \le
    h_\Lambda\!\left(
        \frac1k\sum_gm_g\bar a_g
    \right),
\end{equation*}
which proves \eqref{eq:stoch-odds-lift-main}. For a size-$k$ vector over
$\M$, the argument of $h_\Lambda$ is at most $a^\star$ and
$C(\bm m)\ge kc_{\min}$, proving
\eqref{eq:stoch-variety-bound-main}. Singleton feasibility gives the lower
bound. Maximizing $h_\Lambda(a)-a$ over $a\in[0,1]$ proves
\eqref{eq:stoch-no-cost-cap-main}. \Halmos

\paragraph{Proof of Theorem~\ref{thm:stoch-end-to-end-main}.}
Let $\mathcal E_{\mathrm{est}}$ be the Stage-1 concentration event in
\eqref{eq:stoch-app-stage-one-event}, let $\mathcal E_{\mathrm{orc}}$ be the
event that every pricing batch is $\tau_k$-accurate, and let
$\mathcal E_{\mathrm{samp}}$ be the Stage-2 uniform-deviation event in
\eqref{eq:stoch-app-uniform-deviation} for
$\mathcal C_{\mathcal K}(\M_T)$. By
\Cref{lem:stoch-stage-one-concentration-app,lem:stoch-pricing-separation-app,lem:stoch-fresh-sample-app},
\begin{equation}
    \Prob\left(
        \mathcal E_{\mathrm{est}}
        \cap
        \mathcal E_{\mathrm{orc}}
        \cap
        \mathcal E_{\mathrm{samp}}
    \right)
    \ge
    1
    -\delta_1
    -\delta_2
    -N_{\mathrm{price}}^{\mathrm{stoch}}
        (\varepsilon_{\mathrm{ell}})
        e^{-p_{\mathrm{orc}}m}.
    \label{eq:stoch-app-joint-probability}
\end{equation}

Work on this intersection and write
\begin{equation*}
    V_T^\star
    =
    \max_{\bm m\in\mathcal C_{\mathcal K}(\M_T)}
    \Pi_{\psi,\gamma}^{\mathrm{stoch}}(\bm m).
\end{equation*}
By \Cref{lem:stoch-plugin-transfer-app},
\begin{equation}
    V_T^\star
    \ge
    \frac1\varrho
    \OPT_{\psi,\gamma}^{\mathrm{stoch}}(\G)
    -\varepsilon_{\mathrm{ell}}
    -\varepsilon_{\mathrm{rnd}}
    -2\varepsilon_{\mathrm{est}}(L_1,\delta_1).
    \label{eq:stoch-app-generated-pool-bound-proof}
\end{equation}
Let
$\varepsilon_s=\varepsilon_{\mathrm{samp}}(L_2,\delta_2)$ and let
$\bm m_T^\star$ maximize true value over
$\mathcal C_{\mathcal K}(\M_T)$.

Suppose first that the conservative rule returns $\widetilde{\bm m}$. By the
rule and uniform deviation,
\begin{equation*}
    \Pi_{\psi,\gamma}^{\mathrm{stoch}}(\widetilde{\bm m})
    \ge
    \widehat\Pi_{L_2}^{\mathrm{final}}(\widetilde{\bm m})
    -\varepsilon_s
    \ge
    \underline V.
\end{equation*}
Empirical optimality and uniform deviation also give
\begin{align*}
    \Pi_{\psi,\gamma}^{\mathrm{stoch}}(\widetilde{\bm m})
    &\ge
    \widehat\Pi_{L_2}^{\mathrm{final}}(\widetilde{\bm m})
    -\varepsilon_s\\
    &\ge
    \widehat\Pi_{L_2}^{\mathrm{final}}(\bm m_T^\star)
    -\varepsilon_s\\
    &\ge
    V_T^\star-2\varepsilon_s.
\end{align*}
Hence this case yields
\begin{equation*}
    \Pi_{\psi,\gamma}^{\mathrm{stoch}}(\widehat{\bm m})
    \ge
    \max\{\underline V,V_T^\star-2\varepsilon_s\}.
\end{equation*}

Suppose instead that the rule returns the baseline $\bm m^0$. Then
\begin{equation*}
    \widehat\Pi_{L_2}^{\mathrm{final}}(\widetilde{\bm m})
    -\varepsilon_s
    <
    \underline V.
\end{equation*}
Empirical optimality and uniform deviation imply
\begin{equation*}
    \widehat\Pi_{L_2}^{\mathrm{final}}(\widetilde{\bm m})
    -\varepsilon_s
    \ge
    V_T^\star-2\varepsilon_s,
\end{equation*}
so $V_T^\star-2\varepsilon_s<\underline V$. Since the baseline has true value
at least $\underline V$, this case gives the same bound. Combining the two
cases with \eqref{eq:stoch-app-generated-pool-bound-proof} proves
\eqref{eq:stoch-full-class-main}.

Finally, for every $A,e\ge0$ and $\underline V>0$,
\begin{equation*}
    \max\{\underline V,A-e\}
    \ge
    \frac{\underline V}{\underline V+e}A.
\end{equation*}
Apply this inequality with
\begin{equation*}
    A
    =
    \frac1\varrho
    \OPT_{\psi,\gamma}^{\mathrm{stoch}}(\G),
    \qquad
    e
    =
    \varepsilon_{\mathrm{ell}}
    +\varepsilon_{\mathrm{rnd}}
    +2\varepsilon_{\mathrm{est}}(L_1,\delta_1)
    +2\varepsilon_{\mathrm{samp}}(L_2,\delta_2),
\end{equation*}
to obtain \eqref{eq:stoch-multiplicative-final}. \Halmos

\section{Additional End-to-End Experiments}
\label{app:additional-end-to-end-experiments}

The main text reports the ABCD experiment in detail. This appendix applies the
same stochastic workflow-evaluation, selector-calibration, portfolio-
optimization, and dual-guided generation pipeline to the two other domains in
the study. We organize each domain in the same order: task construction and
data, stochastic performance of the initial workflow bank, selector
calibration, finite-pool optimization, budgeted workflow generation, and fresh
held-out deployment. The SGD and HotpotQA experiments are reported below.

\subsection{Schema-Guided Dialogue}
\label{app:sgd-end-to-end}

\paragraph{Task construction and data.}
Schema-Guided Dialogue (SGD) contains task-oriented conversations spanning
multiple services, each accompanied by a natural-language schema that defines
the service, its available intents, and its slots \citep{rastogi2020sgd}. We
form a schema-conditioned active-intent classification task. At a user turn,
the workflow observes the service description, the descriptions of all intents
available for that service, required and optional slots, slot descriptions,
and the dialogue prefix through the current user message. It must return
exactly one intent from that service's allowed intent set.

To reduce nearly repeated observations within a dialogue, we retain the first
user turn at which each distinct service--active-intent pair appears. We
exclude frames whose active intent is \texttt{NONE} and services with fewer
than two available intents. We preserve the official train, development, and
test split structure and sample 800 development tasks for workflow evaluation
and portfolio construction, 1,004 development-split tasks for selector
calibration, and 800 official test tasks for held-out deployment. The sampled
development, calibration, and held-out sets contain 33, 24, and 30 observed
active-intent labels, respectively; no single intent accounts for more than
6.13\%, 9.27\%, and 7.38\% of the corresponding samples.

\paragraph{Initial workflow bank and stochastic execution.}
We use the same initial bank as in the ABCD experiment: 18 single-call
workflows obtained by crossing Qwen2.5-3B-Instruct,
Mistral-7B-Instruct-v0.3, and Granite-3.3-8B-Instruct with the direct,
evidence-first, decomposition, verify-and-revise, alternatives, and
limited-context prompting strategies. Every workflow--task pair is executed
independently $L=5$ times at temperature one. The best estimated standalone
workflow on the development sample is Granite alternatives, with
one-execution accuracy 82.175\%. If every workflow in the initial bank is run
once independently, expected oracle coverage is 97.467\%, indicating that the
bank retains meaningful complementarity despite its strong best singleton.

The execution parser succeeds on 95.207\% of development executions, and the
five draws produce 1.433 distinct intent predictions per task--workflow cell
on average, confirming that repeated executions are not simply identical
copies of a deterministic output. Workflow-level estimates are also stable at small $L$. Relative
to the full $L=5$ estimates, the workflow rankings based on the first one,
two, and three draws have Spearman correlations 0.897, 0.989, and 0.994; the largest corresponding absolute change in a workflow's average accuracy is 1.975,
0.700, and 0.492 percentage points.

\begin{table}[t]
\centering
\caption{SGD design, selector calibration, and workflow-generation summary.}
\label{tab:sgd-design-calibration}
\small
\begin{tabular}{@{}p{0.42\textwidth}p{0.50\textwidth}@{}}
\toprule
Quantity & Setting or estimate \\
\midrule
Development / calibration / held-out tasks
    & $800 / 1{,}004 / 800$ \\
Observed active-intent labels
    & $33 / 24 / 30$; service-specific allowed intent sets \\
Initial workflow bank
    & 18 workflows: 3 generator models $\times$ 6 prompting strategies \\
Stochastic workflow execution
    & $L=5$ independent draws per task--workflow cell; temperature one \\
Run-size search
    & $K\in\{1,2,3,4,5,6\}$ \\
Final finite-pool optimization
    & 16 SAA scenarios and 512 randomized-rounding trials per $K$ \\
Selector
    & Qwen2.5-7B-Instruct; greedy decoding; cyclic-rotation vote \\
Initial selector calibration
    & 1,500 balanced panels from 702 distinct tasks \\
Initial selector strength
    & $\widehat\lambda=7.2145$, 95\% CI $[5.7747,9.0901]$ \\
Implied pairwise recovery
    & $\widehat\lambda/(1+\widehat\lambda)=87.83\%$ \\
Average selector advantage over uniform choice
    & 34.33 percentage points on the calibration design \\
Budgeted workflow generation
    & 26 ADAS pricing queries; 20 distinct candidates evaluated \\
Generated workflows incorporated
    & 0; final optimization bank remains at 18 workflow types \\
Second selector calibration
    & $\widehat\lambda=8.1506$, 95\% CI $[6.5437,10.4949]$ \\
\bottomrule
\end{tabular}
\end{table}

\paragraph{Selector calibration.}
We construct 100 feasible candidate panels for every
$K\in\{2,\ldots,6\}$ and every interior correct count
$r\in\{1,\ldots,K-1\}$. The selector observes the service schema, dialogue
prefix, and candidate intent labels but not workflow identity, cost, or
correctness. The primary rotation-vote calibration contains 1,500 panels from
702 distinct tasks and has a 100\% parse rate. Maximum-likelihood estimation of
the Plackett--Luce recovery model gives
$\widehat\lambda=7.2145$ with a 95\% confidence interval of
$[5.7747,9.0901]$, obtained by resampling tasks. This estimate corresponds to an 87.83\% fitted probability
of selecting the correct candidate when exactly one of two candidates is
correct and an average 34.33 percentage-point advantage over uniform selection
on the calibration design.

\begin{figure}[t]
    \centering
    \includegraphics[width=0.78\textwidth]
    {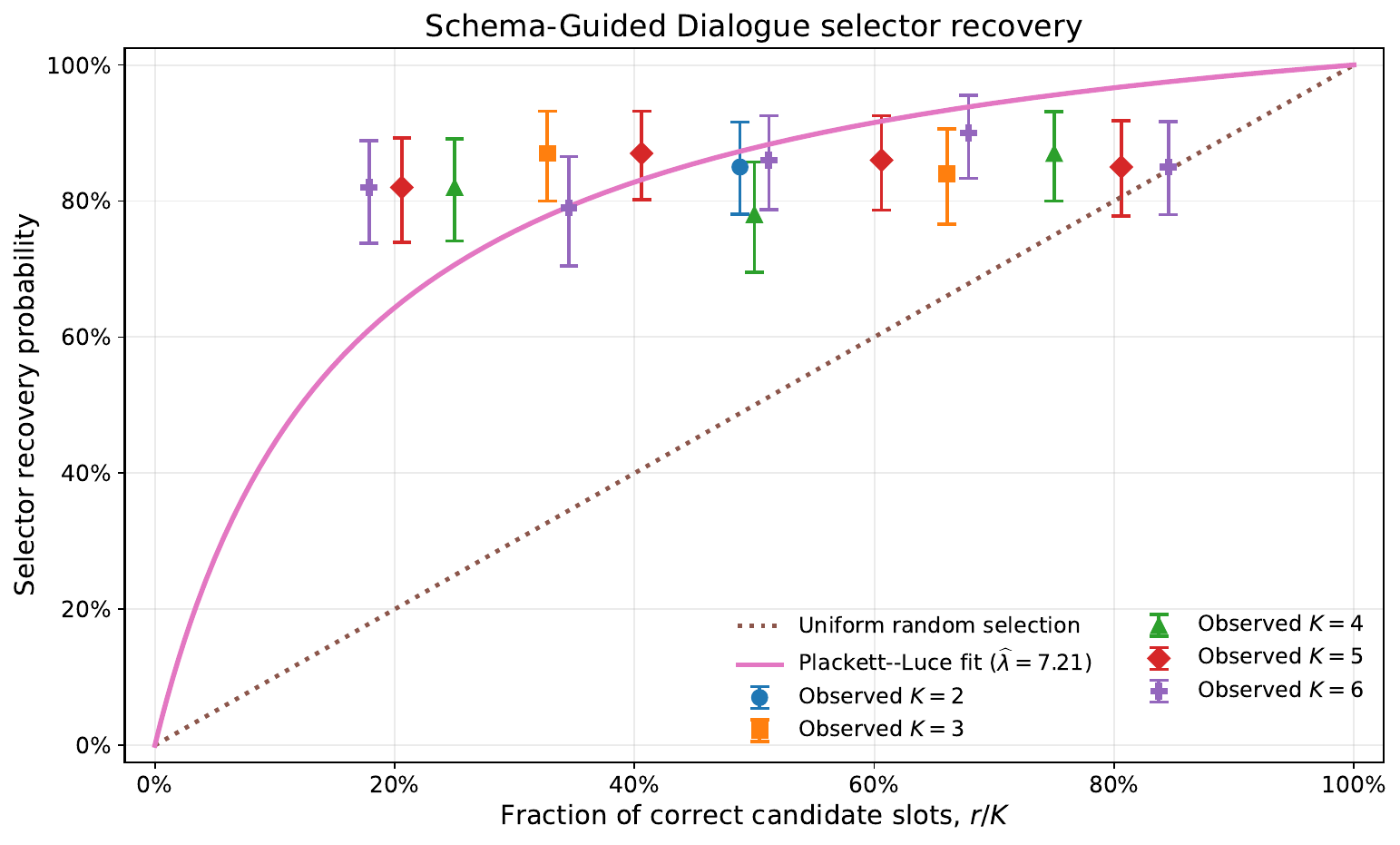}
    \caption{SGD selector recovery. Points report observed recovery on
    controlled $(K,r)$ panels, vertical bars report 95\% bootstrap intervals obtained by resampling tasks, the dotted line is uniform selection, and the solid
    curve is the fitted Plackett--Luce model. Small horizontal offsets separate
    observations having the same correct fraction.}
    \label{fig:sgd-selector-recovery}
\end{figure}

\paragraph{Initial-pool optimization.}
For each $K\in\{1,\ldots,6\}$, we apply the common 16-scenario SAA and
512-trial randomized-rounding procedure described in the main text and then
evaluate each distinct integral portfolio using the corresponding exact
plug-in calculation. Net
development value increases throughout the searched range, and the selected
portfolio has $K=6$. It contains Granite alternatives, Granite direct,
Granite limited context, Mistral evidence first, Qwen decomposition, and Qwen
limited context, one execution each. Its calibrated development accuracy is
92.612\%, and its workflow-cost-adjusted value is 0.920758. Thus, although
repetition is allowed, the selected SGD portfolio uses six distinct workflow
types. The separately optimized no-repeat solution has essentially the same
development value.

\begin{figure}[t]
    \centering
    \includegraphics[width=0.80\textwidth]
    {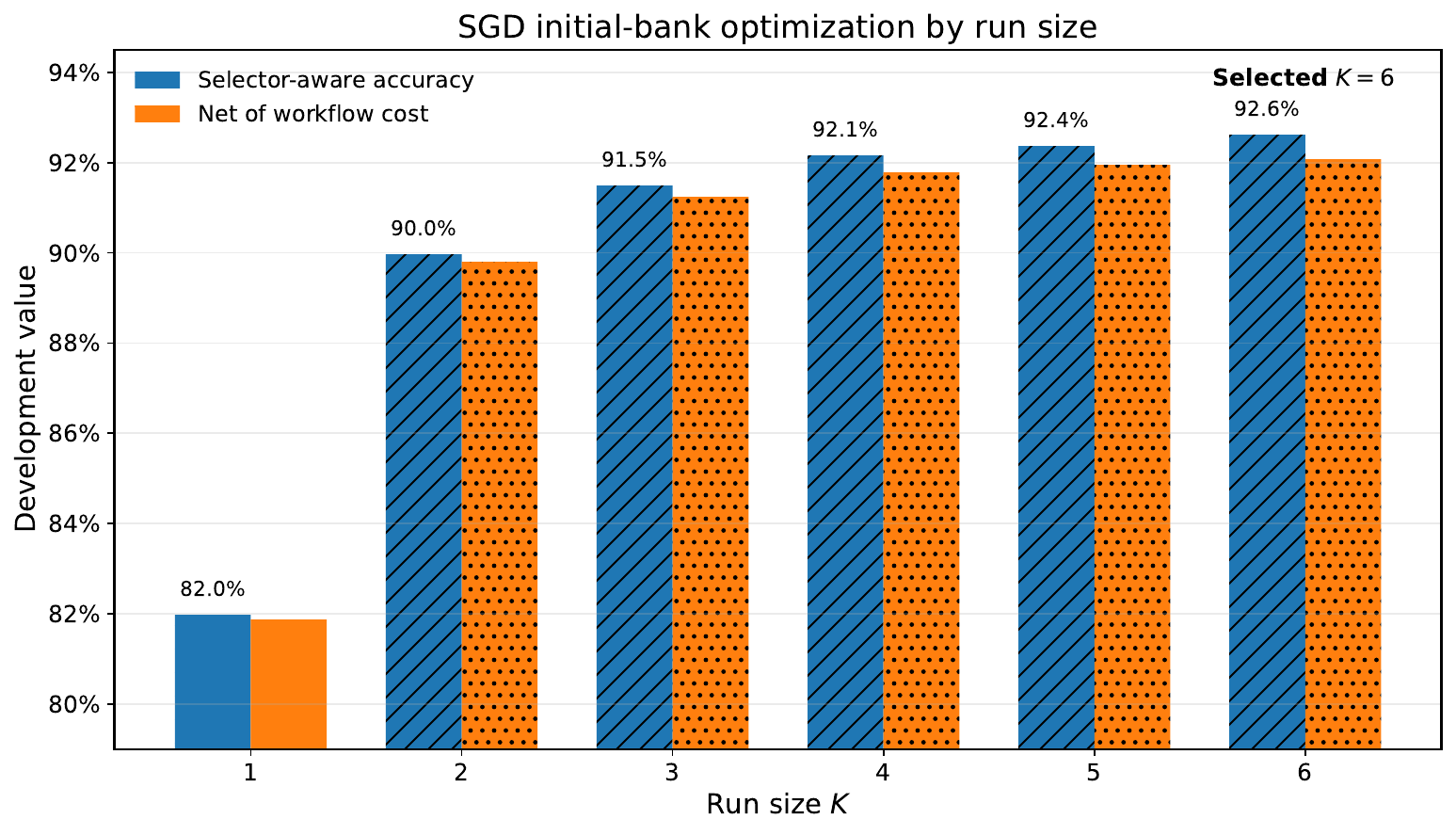}
    \caption{Development performance of the initial SGD workflow bank by run
    size. Bars report exact plug-in selector-aware accuracy and the same value
    after subtracting recurring workflow cost. The primary search selects
    $K=6$.}
    \label{fig:sgd-development-run-size}
\end{figure}

\paragraph{Budgeted dual-guided workflow generation.}
We apply the same practical stochastic ellipsoid-generation procedure as in
the ABCD experiment. For each $K$, workflow generation uses one common-uniform
SAA scenario, checks the explicit and previously materialized workflow
constraints before invoking ADAS, and evaluates every distinct new workflow
with five independent executions on each development task. The prespecified
paper budget permits 20 distinct new workflow evaluations, allocated across
the six run sizes. The SGD run makes 26 ADAS pricing queries and evaluates all
20 candidate workflows. None satisfies the stochastic pricing criterion
required for incorporation at the queried dual points, so the optimization
bank remains at its original 18 workflow types. Each fixed-$K$ search reaches
its candidate-evaluation budget; this result is therefore evidence from a
budgeted search, not a claim that no useful workflow exists anywhere in the
implicit class.

Because the workflow bank does not change, the pipeline's second calibration
is a repeated calibration of the same 18-workflow candidate distribution
rather than an expanded-bank calibration. It gives
$\widehat\lambda=8.1506$ with 95\% confidence interval
$[6.5437,10.4949]$, overlapping the initial estimate. We use this second
estimate for the final model-based predictions below; the held-out deployment
results additionally run the selector directly and do not rely solely on the
parametric recovery model.

\paragraph{Fresh held-out deployment.}
We freeze the reported portfolios, collect five new stochastic executions per
selected workflow on each of 800 official test tasks, and run the deterministic
Qwen selector using cyclic candidate-order rotations and rotation vote.
\Cref{tab:sgd-heldout-deployment,fig:sgd-heldout-accuracy} report the results.
Actual selector accuracy rises from 85.250\% for the best initial singleton to
92.750\% for the optimized initial-bank portfolio, an increase of 7.500
percentage points or 8.8\% relative to the singleton. Since the budgeted ADAS
search incorporates no new workflow, the final repeat-allowed portfolio is the
same optimized initial-bank portfolio and has the same held-out accuracy. The
no-repeat benchmark reaches 92.375\%.

The calibrated model predicts the same ordering, with selector-aware accuracy
84.700\% for the singleton, 94.708\% for the optimized repeat-allowed
portfolio, and 94.543\% for the no-repeat benchmark. For the optimized
portfolio, uniform selection among the realized candidate slots would attain
83.150\%, while oracle coverage is 97.904\%. After subtracting both workflow
cost and realized selector cost, actual total-system net value rises from
0.851357 for the singleton to 0.915629 for the optimized portfolio.

\begin{table}[t]
\centering
\caption{Fresh held-out SGD deployment on 800 tasks. PL accuracy is the
plug-in Plackett--Luce prediction; actual accuracy uses the blind deterministic
Qwen selector. Brackets report 95\% Wilson intervals. Workflow and selector
costs are per task, and actual total-system net subtracts both.}
\label{tab:sgd-heldout-deployment}
\scriptsize
\resizebox{\textwidth}{!}{%
\begin{tabular}{lrrrrrrrr}
\toprule
Method & $K$ & PL accuracy & Actual accuracy (95\% CI) & Random accuracy
& Oracle coverage & Workflow cost & Selector cost & Actual total-system net \\
\midrule
Best initial singleton
& 1 & 0.8470 & $0.8525\ [0.8262,0.8754]$ & 0.8470 & 0.8470
& 0.001143 & 0.000000 & 0.851357 \\
Optimized initial bank / final repeat portfolio
& 6 & 0.9471 & $0.9275\ [0.9074,0.9435]$ & 0.8315 & 0.9790
& 0.005357 & 0.006514 & 0.915629 \\
Final no-repeat benchmark
& 6 & 0.9454 & $0.9238\ [0.9033,0.9402]$ & 0.8303 & 0.9772
& 0.005357 & 0.006514 & 0.911879 \\
\bottomrule
\end{tabular}%
}
\end{table}

\begin{figure}[t]
    \centering
    \includegraphics[width=0.82\textwidth]
    {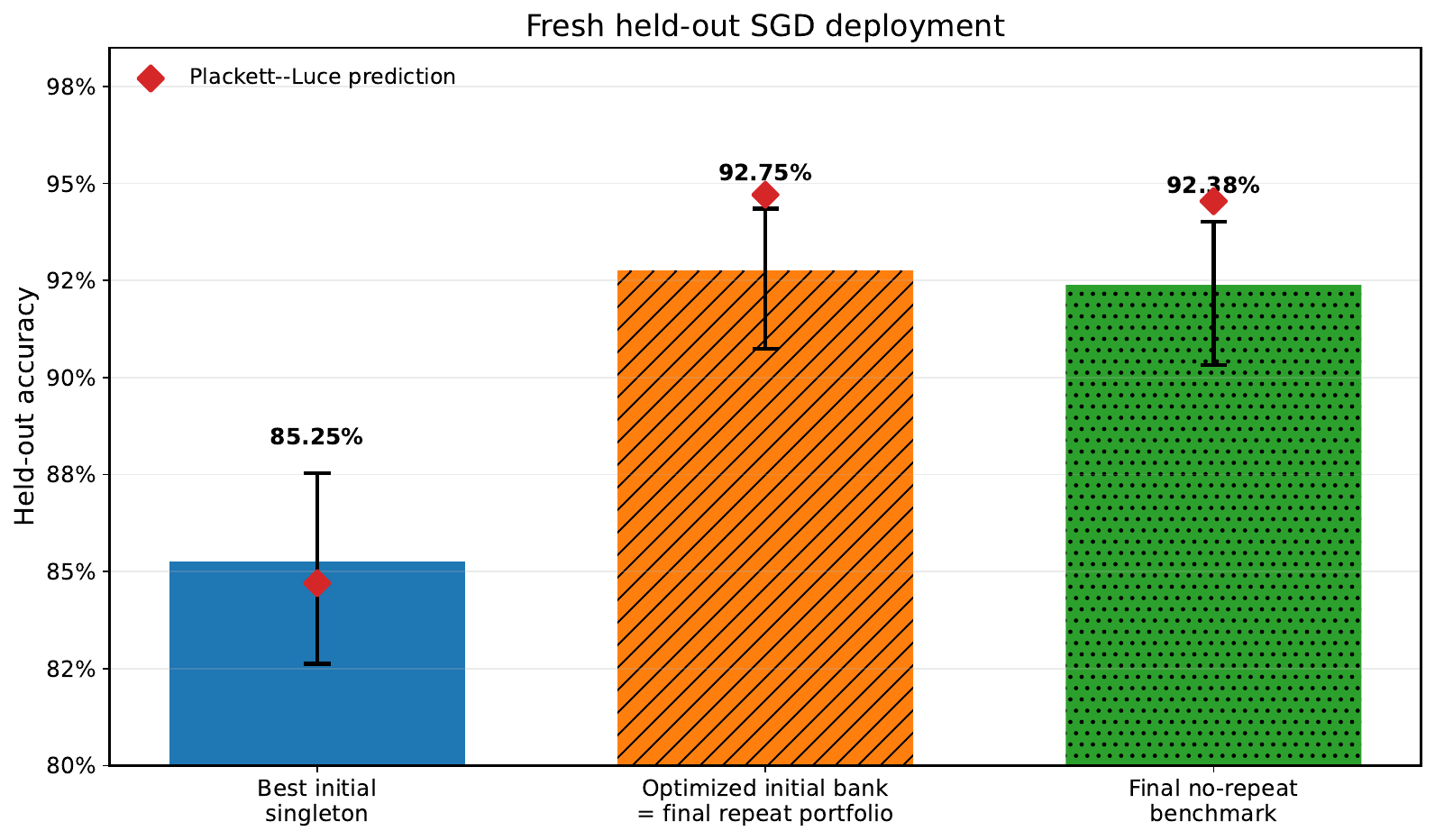}
    \caption{Fresh held-out SGD accuracy. Bars report actual deterministic
    selector accuracy, error bars show 95\% Wilson intervals over 800 tasks,
    and diamonds report the corresponding plug-in Plackett--Luce predictions.}
    \label{fig:sgd-heldout-accuracy}
\end{figure}

\paragraph{Multiplicity.}
Within the selector-calibrated range $K\leq6$, the repeat-allowed solution uses
only distinct workflow types, and the repeat and no-repeat values are
effectively identical. Repetition first appears at $K=7$ in the extended
model-based sensitivity analysis and produces only small gains. Thus, SGD
supports the value of portfolio diversity and post-output selection but,
unlike ABCD, does not provide meaningful evidence that multiplicity is
valuable at the endogenous deployment solution.

\subsection{HotpotQA}
\label{app:hotpotqa-end-to-end}

\paragraph{Task construction, data, and answer scoring.}
HotpotQA is an open-domain question-answering benchmark designed to require
reasoning across multiple pieces of evidence \citep{yang2018hotpotqa}. We use
the distractor configuration. Each workflow receives the question together
with all supplied Wikipedia passages and must return a concise answer span or
short phrase. Supporting-fact annotations are retained only as audit metadata
and are not shown to candidate workflows or to the selector.

The public distractor test labels are not distributed. We therefore sample 800
workflow-development questions from the official training split and form
disjoint selector-calibration and held-out samples of 1,004 and 800 questions
from the labeled distractor validation split. These samples contain 733, 913,
and 741 distinct normalized answers, respectively. The largest answer share is
2.63\% in the development sample, 3.98\% in the calibration sample, and 2.63\%
in the held-out sample. Answers are evaluated using the normalized exact-match
rule implemented in the notebooks: text is lowercased, articles and
nonalphanumeric punctuation are removed, and whitespace is collapsed before
comparison with the reference answer.

\paragraph{Initial workflow bank and stochastic execution.}
We use the same 18-workflow initial bank as in the other domains, obtained by
crossing Qwen2.5-3B-Instruct, Mistral-7B-Instruct-v0.3, and
Granite-3.3-8B-Instruct with the direct, evidence-first, decomposition,
verify-and-revise, alternatives, and limited-context prompting strategies.
Every workflow--question pair is executed independently $L=5$ times at
temperature one.

The best estimated standalone workflow on the development sample is Granite
evidence first, with one-execution accuracy 44.950\%. If every workflow in the
initial bank is executed once independently, expected oracle coverage is
77.578\%, showing substantial complementarity among the candidate generators.
The answer parser succeeds on 88.468\% of executions, and the five stochastic
draws produce 2.908 distinct normalized answers per workflow--question cell on
average, confirming substantial within-cell stochastic variation. Workflow-level estimates are stable despite noisy task-level
probabilities. Relative to the full $L=5$ ranking, rankings based on the first
one, two, and three executions have Spearman correlations 0.939, 0.993, and
0.990; all five of the highest-ranked workflows remain in the top five. The
largest absolute change in a workflow's average accuracy is 3.000, 0.875, and
0.633 percentage points, respectively.

\begin{table}[t]
\centering
\caption{HotpotQA design, selector calibration, and workflow-generation
summary.}
\label{tab:hotpotqa-design-calibration}
\small
\begin{tabular}{@{}p{0.43\textwidth}p{0.49\textwidth}@{}}
\toprule
Quantity & Setting or estimate \\
\midrule
Development / calibration / held-out questions
    & $800 / 1{,}004 / 800$ \\
Data split
    & Development from official train; disjoint calibration and held-out
      subsets from labeled distractor validation \\
Observed normalized answers
    & $733 / 913 / 741$ \\
Initial workflow bank
    & 18 workflows: 3 generator models $\times$ 6 prompting strategies \\
Stochastic workflow execution
    & $L=5$ independent draws per workflow--question cell; temperature one \\
Run-size search
    & $K\in\{1,2,3,4,5,6\}$ \\
Final finite-pool optimization
    & 16 SAA scenarios and 512 randomized-rounding trials per $K$ \\
Selector
    & Qwen2.5-7B-Instruct; greedy decoding; cyclic-rotation vote \\
Initial selector calibration
    & 1,500 balanced panels from 633 distinct questions \\
Initial selector strength
    & $\widehat\lambda=1.7062$, 95\% CI $[1.4557,2.0068]$ \\
Implied pairwise recovery
    & $\widehat\lambda/(1+\widehat\lambda)=63.05\%$ \\
Average selector advantage over uniform choice
    & 10.87 percentage points on the calibration design \\
Budgeted workflow generation
    & 26 ADAS pricing queries; 20 distinct candidates evaluated \\
Generated workflows incorporated
    & 1; final optimization bank contains 19 workflow types \\
Expanded-bank selector strength
    & $\widehat\lambda=1.7529$, 95\% CI $[1.5028,2.0499]$ \\
\bottomrule
\end{tabular}
\end{table}

\paragraph{Selector calibration.}
For each $K\in\{2,\ldots,6\}$ and each interior correct count
$r\in\{1,\ldots,K-1\}$, we construct 100 feasible panels containing exactly
$r$ correct candidate slots and $K-r$ incorrect slots. The selector observes
the question, all supplied passages, and the candidate answers, but not
workflow identity, cost, or correctness. The primary rotation-vote calibration
contains 1,500 panels from 633 distinct questions and has a 100\% base-panel
parse rate.

Maximum-likelihood estimation of the Plackett--Luce recovery model gives
$\widehat\lambda=1.7062$ with a 95\% confidence interval of
$[1.4557,2.0068]$, obtained by resampling tasks. The fitted selector
therefore chooses the correct answer with probability 63.05\% when
exactly one of two candidates is correct. Averaged over the calibration design,
selector recovery exceeds uniform selection by 10.87 percentage
points. Selector recovery is positive but substantially weaker than in the two
service-operation domains, consistent with the greater difficulty of
recognizing a correct free-form multi-hop answer among plausible alternatives.

\begin{figure}[t]
    \centering
    \includegraphics[width=0.80\textwidth]
    {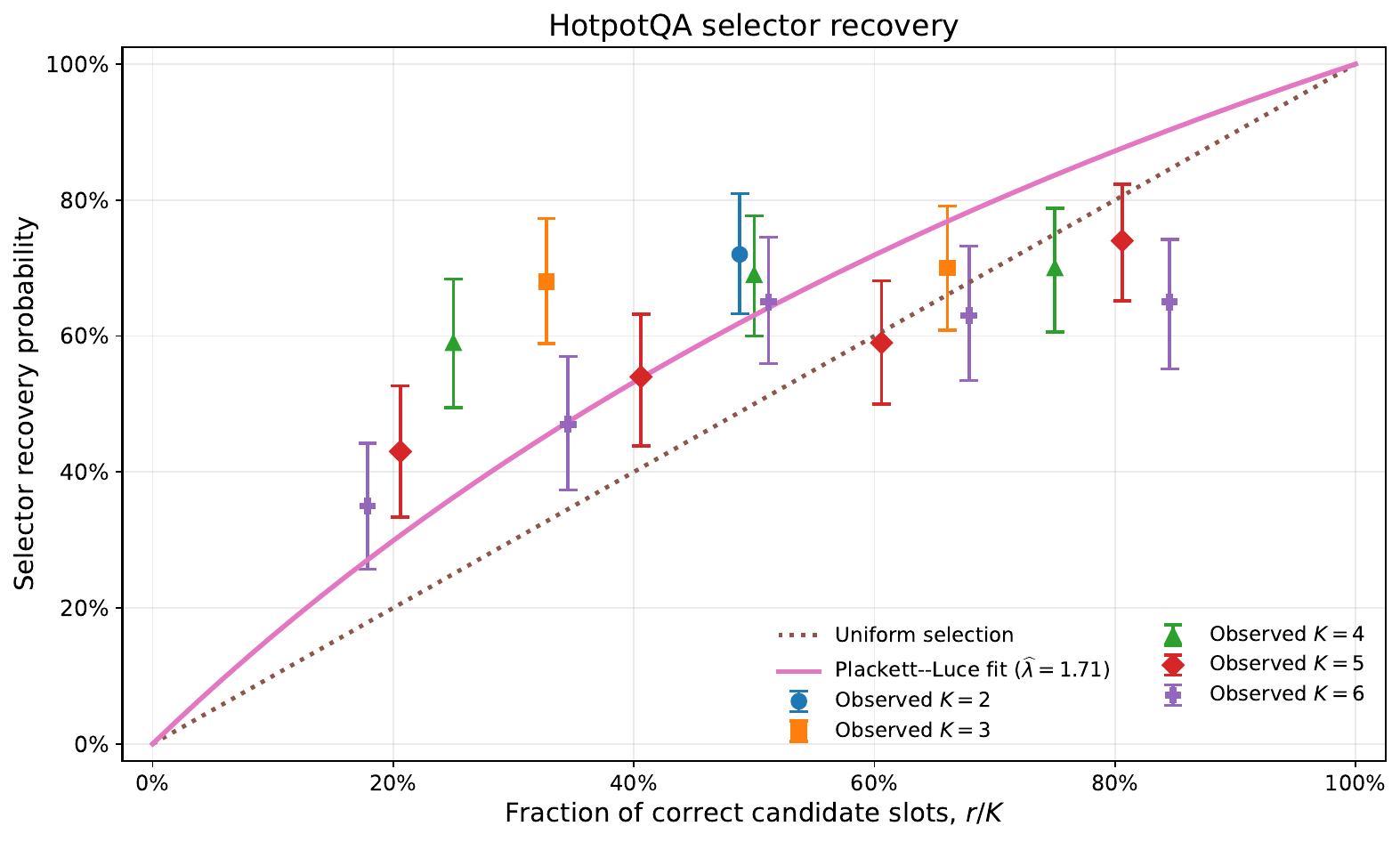}
    \caption{HotpotQA selector recovery. Points report observed recovery on
    controlled $(K,r)$ panels, vertical bars report 95\% bootstrap intervals obtained by resampling tasks, the dotted line is uniform selection, and the solid
    curve is the fitted Plackett--Luce model. Small horizontal offsets separate
    observations having the same correct fraction.}
    \label{fig:hotpotqa-selector-recovery}
\end{figure}

\paragraph{Initial-pool optimization.}
For each $K\in\{1,\ldots,6\}$, we apply the common 16-scenario SAA and
512-trial randomized-rounding procedure described in the main text and then
evaluate each distinct integral portfolio using the corresponding exact
plug-in calculation. The best
initial-bank solution has $K=6$ and assigns five execution slots to Granite
evidence first and one slot to Granite decomposition. Its calibrated
development accuracy is 48.270\%, compared with 44.950\% for the best
singleton, and its workflow-cost-adjusted value is 0.475840.

Multiplicity is already valuable in the initial bank. When repeated workflow
types are prohibited, the best no-repeat solution is attained at $K=3$ by
Granite alternatives, Granite decomposition, and Granite evidence first and
has workflow-cost-adjusted value 0.466358. Thus, allowing repetition increases
the initial-bank endogenous optimum by 0.009482 in accuracy-equivalent
net-value units.

\paragraph{Budgeted dual-guided workflow generation.}
We apply the same practical stochastic ellipsoid-generation procedure used in
the ABCD and SGD experiments. For each $K$, generation uses one common-uniform
SAA scenario, checks explicit and already materialized workflow constraints
before invoking ADAS, and evaluates every distinct new workflow using five
independent executions on each development question. The prespecified paper
budget permits 20 distinct workflow evaluations across the six run sizes.

The search makes 26 ADAS pricing queries, evaluates all 20 distinct candidates,
and incorporates one workflow during the $K=2$ search, expanding the final
optimization bank from 18 to 19 workflow types. We label this workflow G1; its
frozen ADAS name is \texttt{FinalizedMultiHop\allowbreak QASolver}. Each fixed-$K$ search reaches its
prespecified candidate-evaluation budget, so the result should be interpreted
as a budgeted ellipsoid search rather than as a claim that the complete
implicit workflow class has been exhausted.

\begin{table}[t]
\centering
\caption{Executable structure of the ADAS-generated HotpotQA workflow
incorporated into the final bank. All three model calls use GPT-4o-mini through
the ADAS \texttt{LLMAgentBase} interface.}
\label{tab:hotpotqa-generated-workflow}
\small
\begin{tabular}{@{}p{0.10\textwidth}p{0.62\textwidth}rr@{}}
\toprule
Workflow & Executable structure & Mean calls & Cost \\
\midrule
\textbf{G1}
& A first GPT-4o-mini call extracts evidence from the supplied passages. A
  second call verifies the relevance and accuracy of that evidence. A third
  call receives the question, passages, and verified evidence and derives the
  final answer. If either evidence stage returns no usable output, the workflow
  returns a fixed insufficient-evidence response.
& 3.0 & 0.003 \\
\bottomrule
\end{tabular}

\vspace{1mm}
{\footnotesize
\emph{Notes.} G1 is the frozen ADAS program
\texttt{adas\_ell\_\allowbreak 5860177799ff} /\newline \texttt{FinalizedMultiHop\allowbreak QASolver}. Cost is recurring
normalized execution cost per complete workflow run.}
\end{table}

After adding G1, we recalibrate the deterministic Qwen selector on the expanded
candidate distribution. The estimate increases modestly to
$\widehat\lambda=1.7529$ with 95\% confidence interval
$[1.5028,2.0499]$.

\paragraph{Expanded-bank reoptimization and multiplicity.}
G1 changes the portfolio substantially. As a standalone workflow, it has
estimated development accuracy 66.100\%, compared with 44.950\% for the best
initial workflow. Under the primary objective, which subtracts recurring
workflow cost but not selector cost, the final repeat-allowed optimum has
$K=2$ and assigns both execution slots to G1. Its calibrated development
accuracy is 66.863\%, and its workflow-cost-adjusted value is 0.662631. The
endogenous no-repeat optimum is one execution of G1, with calibrated accuracy
66.100\% and workflow-cost-adjusted value 0.658000.

At fixed $K=2$, multiplicity is valuable under the primary workflow-cost
objective.
When repetition is prohibited, the best $K=2$ portfolio combines G1 with
Granite evidence first and has value 0.596823. Allowing two independent
executions of G1 raises that fixed-size value by 0.065808. After optimizing run
size separately under each formulation, the repeat-allowed optimum remains
0.004631 above the no-repeat optimum.

\begin{figure}[t]
    \centering
    \includegraphics[width=0.86\textwidth]
    {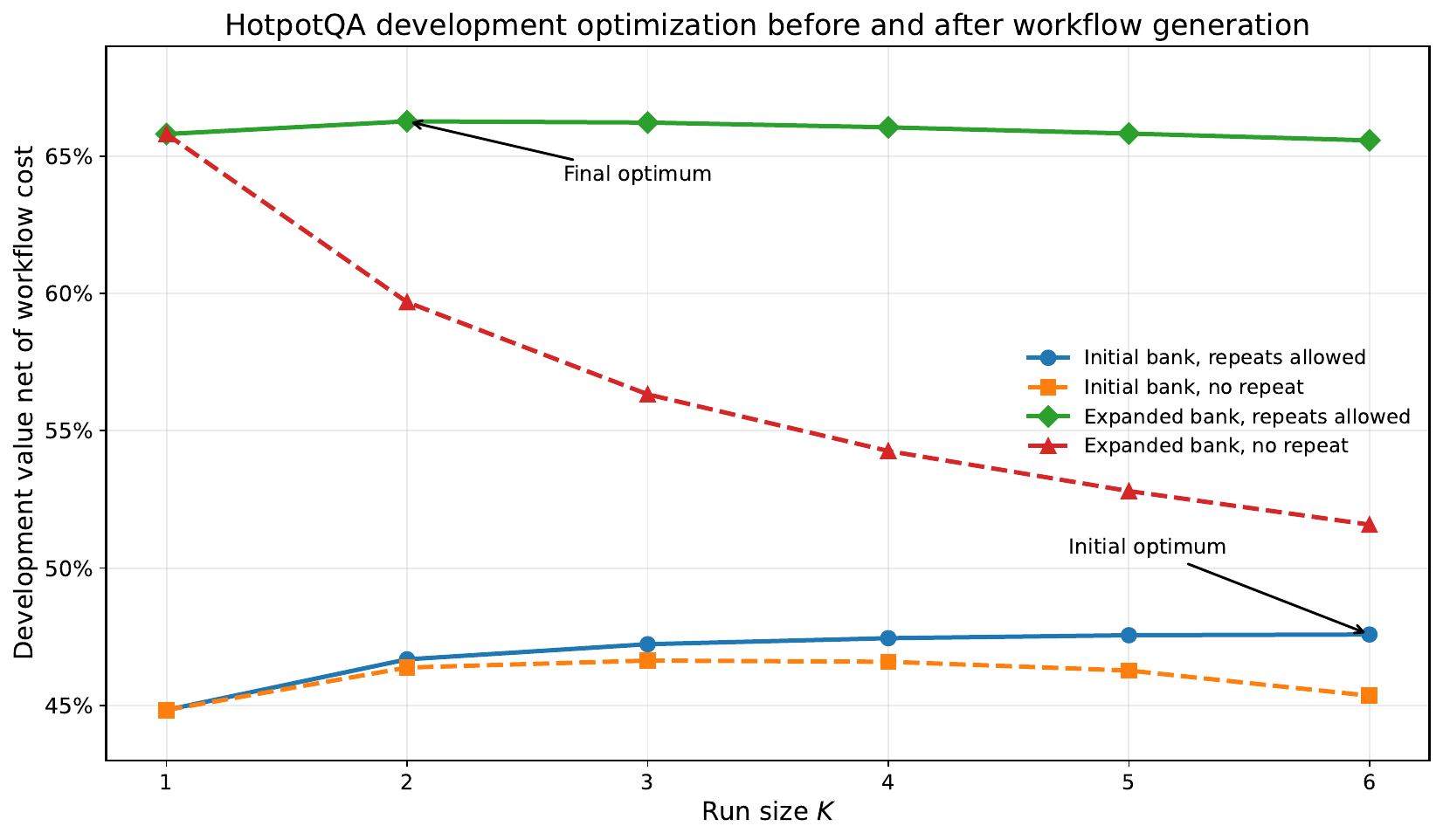}
    \caption{HotpotQA development value before and after workflow generation.
    Lines compare repeat-allowed and no-repeat optimization over the initial and
    expanded workflow banks. Workflow generation shifts the selected run size
    from $K=6$ to $K=2$ and makes repeated execution of G1 optimal under the
    primary workflow-cost objective.}
    \label{fig:hotpotqa-development-value}
\end{figure}

\begin{table}[t]
\centering
\caption{Key HotpotQA development portfolios. Accuracy is the exact plug-in
Plackett--Luce value of the reported integral portfolio. Net value subtracts
recurring workflow cost but not selector cost, matching the primary
optimization objective.}
\label{tab:hotpotqa-development-portfolios}
\small
\begin{tabular}{@{}p{0.25\textwidth}cp{0.39\textwidth}rr@{}}
\toprule
Method & $K$ & Portfolio & Accuracy & Net value \\
\midrule
Best initial singleton
    & 1 & Granite evidence first & 0.449500 & 0.448357 \\
Optimized initial bank
    & 6 & Granite decomposition $+$ 5$\times$ Granite evidence first
    & 0.482697 & 0.475840 \\
Generated workflow, no repeat
    & 1 & G1 & 0.661000 & 0.658000 \\
Final repeat-allowed portfolio
    & 2 & 2$\times$G1 & 0.668631 & 0.662631 \\
\bottomrule
\end{tabular}
\end{table}

\paragraph{Fresh held-out deployment.}
We freeze the reported portfolios, collect five new stochastic executions per
selected workflow on each of 800 held-out questions, and then deploy the
Qwen2.5-7B-Instruct selector deterministically using cyclic candidate-order
rotations and rotation vote. \Cref{tab:hotpotqa-heldout-deployment,fig:hotpotqa-heldout-accuracy} report both the plug-in Plackett--Luce
prediction and the realized selector performance.

Actual selector accuracy is 30.250\% for the best initial singleton and
31.125\% for the optimized initial-bank portfolio. After workflow generation,
accuracy rises to 55.250\% for the two-copy G1 portfolio and 54.875\% for the
single-G1 no-repeat benchmark. Relative to the best initial singleton, the full
repeat-allowed procedure gains 25.000 percentage points, or 82.6\%. Workflow
generation contributes 24.125 percentage points relative to the optimized
initial portfolio. The fitted recovery model is especially accurate for the
final repeated portfolio, predicting 55.017\% compared with the realized
55.250\%.

The held-out comparison also clarifies the value and cost of multiplicity. A
second independent G1 execution raises actual accuracy by 0.375 percentage
points relative to one G1 execution. Under the primary objective, which charges
workflow cost but not selector cost, the two-copy portfolio has slightly higher
net value: 0.546500 versus 0.545750. Once realized selector cost is also
subtracted, however, the one-copy benchmark has slightly higher total-system
net value, 0.545750 versus 0.544329. Thus, HotpotQA provides evidence that
multiplicity can improve accuracy while also illustrating that the incremental
gain may not justify the additional post-output selection cost.

\begin{table}[t]
\centering
\caption{Fresh held-out HotpotQA deployment on 800 questions. PL accuracy is
the plug-in Plackett--Luce prediction; actual accuracy uses the blind
deterministic Qwen selector. Brackets report 95\% Wilson intervals. Workflow
and selector costs are per question, and actual total-system net subtracts
both.}
\label{tab:hotpotqa-heldout-deployment}
\scriptsize
\resizebox{\textwidth}{!}{%
\begin{tabular}{lrrrrrrrr}
\toprule
Method & $K$ & PL accuracy & Actual accuracy (95\% CI) & Random accuracy
& Oracle coverage & Workflow cost & Selector cost & Actual total-system net \\
\midrule
Best initial singleton
& 1 & 0.2980 & $0.3025\ [0.2717,0.3352]$ & 0.2980 & 0.2980
& 0.001143 & 0.000000 & 0.301357 \\
Optimized initial bank
& 6 & 0.3331 & $0.3113\ [0.2801,0.3442]$ & 0.2997 & 0.4642
& 0.006857 & 0.006514 & 0.297879 \\
Final ADAS portfolio, repeats
& 2 & 0.5502 & $0.5525\ [0.5179,0.5866]$ & 0.5430 & 0.5692
& 0.006000 & 0.002171 & 0.544329 \\
Final ADAS portfolio, no repeat
& 1 & 0.5430 & $0.5488\ [0.5141,0.5829]$ & 0.5430 & 0.5430
& 0.003000 & 0.000000 & 0.545750 \\
\bottomrule
\end{tabular}%
}
\end{table}

\begin{figure}[t]
    \centering
    \includegraphics[width=0.84\textwidth]
    {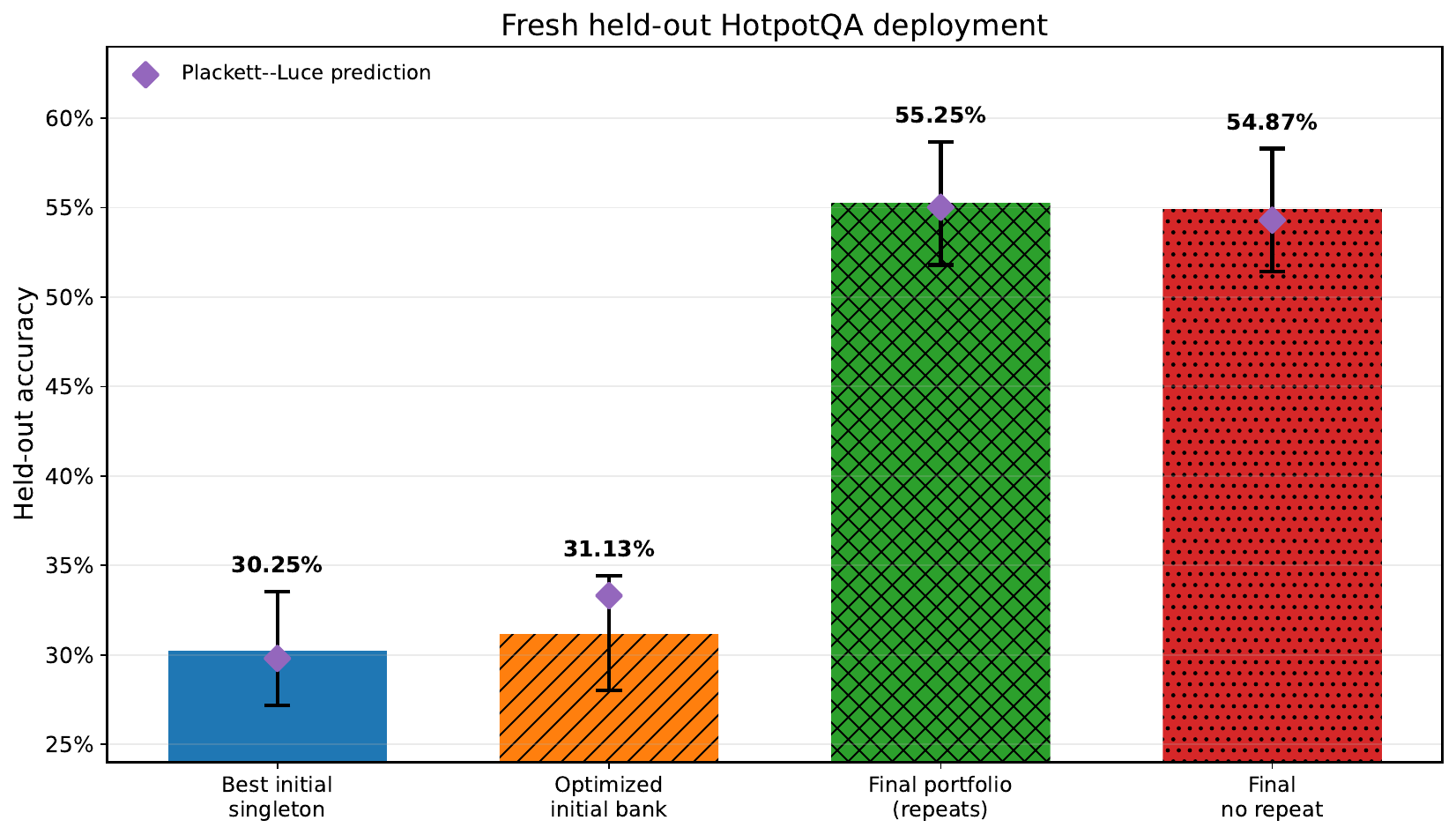}
    \caption{Fresh held-out HotpotQA accuracy. Bars report actual
    deterministic-selector accuracy, error bars show 95\% Wilson intervals over
    800 questions, and diamonds report the corresponding plug-in
    Plackett--Luce predictions.}
    \label{fig:hotpotqa-heldout-accuracy}
\end{figure}

\end{APPENDICES}

\end{document}